%% file: main_iclr.tex
\documentclass{article} 
\usepackage{iclr2023_conference,times}
\usepackage{xcolor}  
\usepackage{preamble/color-edits}
\addauthor{ms}{blue}
\usepackage{etoolbox,comment}   
 \newtoggle{arxiv}
 \togglefalse{arxiv}
\usepackage{xspace}
\usepackage{booktabs,color,xcolor,wrapfig} 
\usepackage{subcaption}
\usepackage{enumitem}

\input{preamble/header}

\input{preamble/general_macros}

\input{preamble/specific_macros}

\input{preamble/characters}

\newcommand{\kz}[1]{\textcolor{red}{[Kaiqing: #1]}}

\newcommand{\aviv}[1]{\textcolor{magenta}{#1}} 
\newcommand{\oos}{\textsc{oos}\xspace}
\newcommand{\ooc}{\textsc{ooc}\xspace}

\newcommand{\ood}{\textsc{ood}\xspace}

\title{Learning to Extrapolate: A Transductive Approach}

\author{Aviv Netanyahu$^{1,2}$\thanks{denotes equal contribution. Correspondence to Aviv Netanyahu $<$avivn@mit.edu$>$, Abhishek Gupta $<$abhgupta@cs.washington.edu$>$.}$\,\,$, Abhishek Gupta$^{1,2,3*}$, Max Simchowitz$^2$, Kaiqing Zhang$^{2,4}$ \& Pulkit Agrawal$^{1,2}$\\ 
Improbable AI Lab$^1$ \,\,\,\,\,\,\,\,\, MIT$^2$ \,\,\,\,\,\,\,\,\,
University of Washington$^3$ \,\,\,\,\,\,\,\,\,
University of Maryland, College Park$^4$
}

\iclrfinalcopy 
\begin{document} 

\maketitle 

\begin{abstract}
\input{body/abstract}
\end{abstract}

\input{body/introduction}

\input{body/setting2}

\input{body/theory_new}
\input{body/experiments_iclr}

\input{body/related}

\input{body/discussion}
\newpage
\input{body/ethics}
\input{body/reproducibility}
\input{body/acknowledgments}
\newpage  
\bibliographystyle{iclr2023_conference}
\bibliography{main_iclr}
\newpage 
\tableofcontents
\newpage

\appendix
 
\input{appendix/full_related_app}
\newpage
\input{appendix/theoretic_app}
\newpage
\input{appendix/Additional_Results}

\newpage
\input{appendix/implementation_details}

\end{document}

%% file: preamble/header.tex
\usepackage{boxedminipage}
\usepackage{multirow,nicefrac}
\usepackage{makecell,upgreek}
\usepackage{footnote}
\usepackage{longtable}

 \usepackage{stackengine}

\usepackage{booktabs}   
\usepackage{float}

\addauthor{kz}{brown}
\addauthor{ms}{purple}

\usepackage{amsthm}
\iftoggle{arxiv}
{
\usepackage{subfig}

  \usepackage{fullpage}
  \usepackage[noend]{algpseudocode}
  \usepackage{algorithm}
}
{
  \usepackage{algorithm}
  \usepackage{algorithmic}
}
\usepackage[breaklinks=true]{hyperref}

\usepackage{amssymb,mathrsfs}
\usepackage{mathtools}
\DeclareMathSymbol{\shortminus}{\mathbin}{AMSa}{"39}

	\usepackage{natbib}

\usepackage{prettyref}

\usepackage[capitalise,nameinlink]{cleveref}
\Crefname{equation}{Eq.}{Eqs.}
\Crefname{assumption}{Assumption}{Assumptions}
\Crefname{condition}{Condition}{Conditions}
\Crefname{claim}{Claim}{Claims}

\makeatletter
\newcommand\incircbin
{%
  \mathpalette\@incircbin
}
\newcommand\@incircbin[2]
{%
  \mathbin%
  {%
    \ooalign{\hidewidth$#1#2$\hidewidth\crcr$#1\bigcirc$}%
  }%
}

\usepackage{breakcites,bm}

\hypersetup{colorlinks,citecolor=blue,linkcolor = blue}
\usepackage{latexsym}
\usepackage{relsize}
\usepackage{thm-restate}
\usepackage{appendix}

\newcommand{\N}{\mathbb{N}}

\newcommand{\R}{\mathbb{R}}

\numberwithin{equation}{section}
\newcommand\numberthis{\addtocounter{equation}{1}\tag{\theequation}}

\newcommand{\rmd}{\mathrm{d}}

\def\bv{\mathbf{v}}

\DeclareFontFamily{U}{mathx}{\hyphenchar\font45}
\DeclareFontShape{U}{mathx}{m}{n}{
      <5> <6> <7> <8> <9> <10>
      <10.95> <12> <14.4> <17.28> <20.74> <24.88>
      mathx10
      }{}
\DeclareSymbolFont{mathx}{U}{mathx}{m}{n}
\DeclareFontSubstitution{U}{mathx}{m}{n}
\DeclareMathAccent{\widecheck}{0}{mathx}{"71}
\DeclareMathAccent{\wideparen}{0}{mathx}{"75}

\newcommand{\ignore}[1]{}

\newcommand{\op}{\mathrm{op}}

\newcommand{\fro}{\mathrm{F}}

	\theoremstyle{plain}
	\newtheorem{theorem}{Theorem}

	\newtheorem{lemma}{Lemma}[section]

	\newtheorem{corollary}{Corollary}[section]
	\newtheorem{proposition}[lemma]{Proposition}

	\theoremstyle{definition}

	\newtheorem{definition}{Definition}[section]

	\newtheorem{remark}{Remark}[section]

  \newtheorem{assumption}{Assumption}[section]
  \newtheorem{condition}{Condition}[section]

\makeatletter
\newcommand{\neutralize}[1]{\expandafter\let\csname c@#1\endcsname\count@}
\makeatother

\newtheorem*{theorem*}{Theorem}
\newtheorem*{lemma*}{Lemma}
\newtheorem*{corollary*}{Corollary}
\newtheorem*{proposition*}{Proposition}
\newtheorem*{claim*}{Claim}
\newtheorem*{fact*}{Fact}
\newtheorem*{observation*}{Observation}

\newtheorem*{definition*}{Definition}
\newtheorem*{remark*}{Remark}
\newtheorem*{example*}{Example}

\newtheoremstyle{named}{}{}{\itshape}{}{\bfseries}{}{.5em}{\Cref{#3} {\normalfont (informal)} }
{}
\theoremstyle{named}

\theoremstyle{plain}

\DeclareMathAlphabet{\mathbfsf}{\encodingdefault}{\sfdefault}{bx}{n}

\let\Pr\relax
\DeclareMathOperator{\Pr}{\mathbb{P}}

\newcommand{\poly}{\mathrm{poly}}

\let\oldtfrac\tfrac
\renewcommand{\tfrac}[2]{\smash{\oldtfrac{#1}{#2}}}

\let\nablaold\nabla
\renewcommand{\nabla}{\nablaold\mkern-2.5mu}

\newcommand{\Exp}{\mathbb{E}}


%% file: preamble/general_macros.tex
\newcommand{\I}{\mathbb{I}}

%% file: preamble/specific_macros.tex
\newcommand{\Dtest}{\cD_{\mathrm{test}}}
\newcommand{\Dtrain}{\cD_{\mathrm{train}}}

\newcommand{\fstk}{\fst_k}
\newcommand{\gstk}{\gst_k}
\newcommand{\hst}{h^{\star}}

\newcommand{\fhat}{\hat{f}}
\newcommand{\ghat}{\hat{g}}

\newcommand{\gst}{g_{\star}}
\newcommand{\fst}{f_{\star}}

\newcommand{\rank}{\mathrm{rank}}

\newcommand{\Risk}{\cR}

\newcommand{\hstk}{\hst_k}

%% file: preamble/characters.tex
\def\ddefloop#1{\ifx\ddefloop#1\else\ddef{#1}\expandafter\ddefloop\fi}
\def\ddef#1{\expandafter\def\csname bb#1\endcsname{\ensuremath{\mathbb{#1}}}}
\ddefloop ABCDEFGHIJKLMNOPQRSTUVWXYZ\ddefloop

\def\ddefloop#1{\ifx\ddefloop#1\else\ddef{#1}\expandafter\ddefloop\fi}
\def\ddef#1{\expandafter\def\csname fr#1\endcsname{\ensuremath{\mathfrak{#1}}}}
\ddefloop ABCDEFGHIJKLMNOPQRSTUVWXYZ\ddefloop

\def\ddefloop#1{\ifx\ddefloop#1\else\ddef{#1}\expandafter\ddefloop\fi}
\def\ddef#1{\expandafter\def\csname eul#1\endcsname{\ensuremath{\EuScript{#1}}}}
\ddefloop ABCDEFGHIJKLMNOPQRSTUVWXYZ\ddefloop

\def\ddefloop#1{\ifx\ddefloop#1\else\ddef{#1}\expandafter\ddefloop\fi}
\def\ddef#1{\expandafter\def\csname scr#1\endcsname{\ensuremath{\mathscr{#1}}}}
\ddefloop ABCDEFGHIJKLMNOPQRSTUVWXYZ\ddefloop

\def\ddefloop#1{\ifx\ddefloop#1\else\ddef{#1}\expandafter\ddefloop\fi}
\def\ddef#1{\expandafter\def\csname b#1\endcsname{\ensuremath{\mathbf{#1}}}}
\ddefloop ABCDEFGHIJKLMNOPQRSTUVWXYZ\ddefloop

\def\ddefloop#1{\ifx\ddefloop#1\else\ddef{#1}\expandafter\ddefloop\fi}
\def\ddef#1{\expandafter\def\csname bhat#1\endcsname{\ensuremath{\hat{\mathbf{#1}}}}}
\ddefloop ABCDEFGHIJKLMNOPQRSTUVWXYZ\ddefloop

\def\ddefloop#1{\ifx\ddefloop#1\else\ddef{#1}\expandafter\ddefloop\fi}
\def\ddef#1{\expandafter\def\csname btil#1\endcsname{\ensuremath{\tilde{\mathbf{#1}}}}}
\ddefloop ABCDEFGHIJKLMNOPQRSTUVWXYZ\ddefloop

\def\ddefloop#1{\ifx\ddefloop#1\else\ddef{#1}\expandafter\ddefloop\fi}
\def\ddef#1{\expandafter\def\csname bbar#1\endcsname{\ensuremath{\bar{\mathbf{#1}}}}}
\ddefloop ABCDEFGHIJKLMNOPQRSTUVWXYZ\ddefloop

\def\ddefloop#1{\ifx\ddefloop#1\else\ddef{#1}\expandafter\ddefloop\fi}
\def\ddef#1{\expandafter\def\csname bst#1\endcsname{\ensuremath{\mathbf{#1}^\star}}}
\ddefloop ABCDEFGHIJKLMNOPQRSTUVWXYZ\ddefloop

\def\ddef#1{\expandafter\def\csname c#1\endcsname{\ensuremath{\mathcal{#1}}}}
\ddefloop ABCDEFGHIJKLMNOPQRSTUVWXYZ\ddefloop

\def\ddef#1{\expandafter\def\csname cbar#1\endcsname{\ensuremath{\bar{\mathcal{#1}}}}}
\ddefloop ABCDEFGHIJKLMNOPQRSTUVWXYZ\ddefloop

\def\ddef#1{\expandafter\def\csname cvec#1\endcsname{\ensuremath{\vec{\mathcal{#1}}}}}
\ddefloop ABCDEFGHIJKLMNOPQRSTUVWXYZ\ddefloop

\def\ddef#1{\expandafter\def\csname h#1\endcsname{\ensuremath{\widehat{#1}}}}
\ddefloop ABCDEFGHIJKLMNOPQRSTUVWXYZ\ddefloop
\def\ddef#1{\expandafter\def\csname hc#1\endcsname{\ensuremath{\widehat{\mathcal{#1}}}}}
\ddefloop ABCDEFGHIJKLMNOPQRSTUVWXYZ\ddefloop
\def\ddef#1{\expandafter\def\csname t#1\endcsname{\ensuremath{\widetilde{#1}}}}
\ddefloop ABCDEFGHIJKLMNOPQRSTUVWXYZ\ddefloop
\def\ddef#1{\expandafter\def\csname tc#1\endcsname{\ensuremath{\widetilde{\mathcal{#1}}}}}
\ddefloop ABCDEFGHIJKLMNOPQRSTUVWXYZ\ddefloop

%% file: body/abstract.tex
Machine learning systems, especially with overparameterized deep neural networks, can generalize to novel test instances drawn from the same distribution as the training data. However, they fare poorly when evaluated on \emph{out-of-support} test points. In this work, we tackle the problem of developing machine learning systems that retain the power of overparameterized function approximators while enabling extrapolation to out-of-support test points when possible. This is accomplished by noting that under certain conditions, a ``transductive'' reparameterization can convert an out-of-support extrapolation problem into a problem of within-support combinatorial generalization. We propose a simple strategy based on bilinear embeddings to enable this type of combinatorial generalization, thereby addressing the out-of-support extrapolation problem under certain conditions. We instantiate a simple, practical algorithm applicable to various supervised learning and imitation learning tasks.

%% file: body/introduction.tex
\section{Introduction}

Generalization is a central problem in machine learning. Typically, one expects generalization when the test data is sampled from {\it the same distribution} as the training set, i.e {\it out-of-sample} generalization. 
However, in many scenarios, test data is sampled from a different distribution from the training set, i.e., {\it out-of-distribution} (\ood). In some \ood scenarios, the test distribution is assumed to be known during training -- a common assumption made by meta-learning methods~\citep{finn2017model}. Several works have tackled a more general scenario of ``reweighted'' distribution shift~\citep{koh21wilds, quinonero2008dataset} 
where the test distribution  shares support with the training distribution, but has a different and unknown probability density. This setting 
can be tackled via distributional robustness approaches ~\citep{sinha18dro, rahimian19dro}.

Our paper aims to find structural conditions under which generalization to test data with
support {\it outside} of the training distribution is possible. Formally, assume the problem of learning function $h$: $\hat{y} = h(x)$ using data $\{(x_i, y_i)\}_{i=1}^N \sim \mathcal{D_{\mathrm{train}}}$, where $x_i \in \mathcal{X}_{\mathrm{train}}$, the training domain. 
We are interested in making accurate predictions $h(x)$ for $x \notin \mathcal{X}_{\mathrm{train}}$ (see examples in Fig~\ref{fig:intro}).
Consider an example task of predicting actions to reach a desired goal (Fig~\ref{subfig:intro-reach}). During training, goals are provided from the blue cuboid ($x \in \mathcal{X}_{\mathrm{train}})$, but test time goals are from the orange cuboid ($x \notin \mathcal{X}_{\mathrm{train}}$). If $h$ is modeled using a deep neural network, its predictions on test goals in the blue area are likely to be accurate, but for the goals in the orange area the performance can be arbitrarily poor unless further domain knowledge is incorporated. 
This challenge manifests itself in a variety of real-world problems, ranging from supervised learning problems like object classification~\citep{barbu2019objectnet}, sequential decision making with reinforcement learning~\citep{kirk21surveygen}, transferring reinforcement learning policies from simulation to real-world~\citep{zhao2020sim}, imitation learning~\citep{Haan2019causal}, etc. Reliably deploying learning algorithms in unconstrained environments requires accounting for such ``out-of-support'' distribution shift, which we refer to as \emph{extrapolation}. 

It is widely accepted that if one can identify some \textit{structure} in the training data that constrains the behavior of optimal predictors on novel data, then extrapolation may become possible. Several methods can extrapolate if the nature of distribution shift is known apriori: convolution neural networks are appropriate if a training pattern appears at an out-of-distribution translation in a test example~\citep{kayhan2020translation}. Similarly, accurate predictions can be made for object point clouds in out-of-support orientations by \textit{building in} SE(3) equivariance~\citep{deng21vectorneuron}. Another way to extrapolate is if the model class is known apriori: fitting a linear function to a linear problem will extrapolate. Similarly, methods like NeRF~\citep{zhang2022ray} use physics of image formation to learn a 3D model of a scene which can synthesize images from novel viewpoints. 

We propose an alternative structural condition under which extrapolation is feasible. Typical machine learning approaches are \emph{inductive}: decision-making rules are inferred from training data and employed for test-time predictions. An alternative to induction is \textit{transduction}~\citep{gammerman1998transduction} where a test example is compared with the training examples to make a prediction. Our main insight is that in the transductive view of machine learning, extrapolation can be reparameterized as a combinatorial generalization problem, which, under certain low-rank and coverage conditions \citep{shah2020sample, agarwal2021causal, andreas16modular, andreas19compositionality}, admits a solution.

First we show how we can 
(i) reparameterize out-of-support inputs $h(x_{\mathrm{test}}) \rightarrow h(\Delta x, x')$, where $x' \in \mathcal{X}_{\mathrm{train}}$, 
and $\Delta x$ is a representation of the difference between $x_{\mathrm{test}}$ and $x'$.
%
We then (ii) provide conditions under which $h(\Delta x, x')$ makes accurate predictions for unseen combinations of $(\Delta x, x')$  based on a theoretically justified bilinear modeling approach: $h(\Delta x, x') \rightarrow f(\Delta x)^{\top}g(x')$, where $f$ and $g$ map their inputs into vector spaces of the same dimension. Finally,  
(iii) empirical results demonstrate the generality of extrapolation of our algorithm on a wide variety of tasks: (a) regression for analytical functions and high-dimensional point cloud data; (b) sequential decision-making tasks such as goal-reaching for a simulated robot. 


\begin{figure}[t!]
    \centering
    \begin{subfigure}{.3\textwidth}
      \centering
        \includegraphics[width=0.6\linewidth]{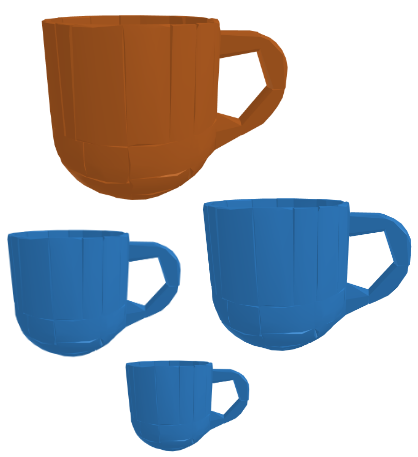}
      \caption{\footnotesize{Grasp point prediction}}
      \label{subfig:intro-grasp}
    \end{subfigure} 
    \begin{subfigure}{.3\textwidth}
      \centering
        \includegraphics[width=.9\linewidth]{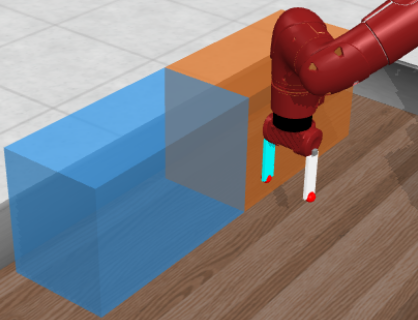}
      \caption{\footnotesize{Action prediction}}
      \label{subfig:intro-reach}
    \end{subfigure}
    \begin{subfigure}{.3\textwidth}
      \centering
        \includegraphics[width=.9\linewidth]{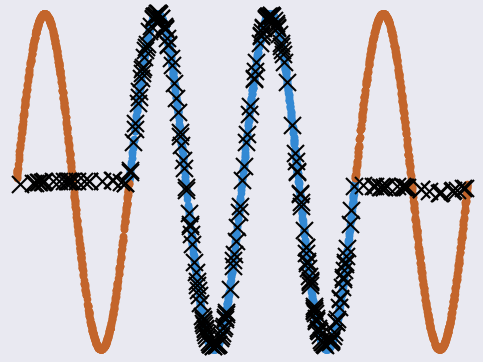}
      \caption{\footnotesize{Function value prediction}}
      \label{subfig:intro-sine}
    \end{subfigure}    
    \caption{
    \footnotesize{In the real-world the test distribution (orange) often has a \textit{different support} than the training distribution (blue). Some illustrative tasks:
    \textbf{(a)} grasp point prediction for object instances with out-of-support scale, \textbf{(b)} action prediction for reaching out-of-support goals, \textbf{(c)} function value prediction for an out-of-support input range. The black crosses show predictions for a conventionally trained deep neural network that makes accurate predictions for in-support inputs, but fails on out-of-support inputs. We propose an algorithm that makes accurate out-of-support predictions under a set of assumptions. 
    }}
    \label{fig:intro}
\end{figure}

%% file: body/setting2.tex
\renewcommand{\hst}{h_{\star}}
\newcommand{\hbarst}{\bar{h}_{\star}}
\newcommand{\Prob}{\mathrm{Pr}}
\newcommand{\btheta}{\bm{\theta}}
\newcommand{\risk}{\cR}
\newcommand{\simplex}{\cP}
\newcommand{\bSigma}{\mathbf{\Sigma}}
\newcommand{\epsgen}{\epsilon_{\mathrm{gen}}}
\newcommand{\xtest}{x_{\mathrm{test}}}
\renewcommand{\hstk}[1][k]{h_{\star,#1}}
\renewcommand{\fstk}[1][k]{f_{\star,#1}}
\renewcommand{\gstk}[1][k]{g_{\star,#1}}
\newcommand{\hthetk}[1][k]{h_{\btheta,#1}}
\newcommand{\hthet}{h_{\btheta}}

\newcommand{\fthetk}[1][k]{f_{\btheta,#1}}
\newcommand{\gthetk}[1][k]{g_{\btheta,#1}}

\section{Setup}\label{sec:prelim}
\paragraph{Notation.}
Given a space of inputs $\cX$ and targets $\cY$, we aim to learn a predictor $\hthet: \cX \to \simplex(\cY)$\footnote{Throughout, we let $\simplex(\cY)$ denote the set of distributions supported on $\cY$.}  parameterized by $\theta$, which best fits a 
ground  truth function $\hst: \cX \to \cY$. 
Given some non-negative loss function $\ell: \cY \times \cY \to \R_{\ge 0}$ on the outputs (e.g., squared loss), and a distribution $\cD$ over $\cX$, the \emph{risk} is defined as 
	\begin{align}\label{eq:risk}
	\risk(\hthet;\cD) := \Exp_{x \sim \cD}\Exp_{y \sim \hthet(x)} \ell(y,\hst(x)).
	\end{align}
Various training ($\Dtrain$) and test ($\Dtest$) distributions yield different generalization settings:
       


\textbf{In-Distribution Generalization.} This setting  assumes $\Dtest=\Dtrain$. 
The challenge is to ensure that with $N$ samples from $\Dtrain$, the expected risk $\risk(\hthet;\Dtest) = \risk(\hthet;\Dtrain)$ is small. This is a common paradigm in both empirical supervised learning (e.g.,  \cite{simonyan2014very})
and in standard statistical learning theory (e.g., \cite{vapnik2006estimation}).

\textbf{Out-of-Distribution ({\ood}).} This is more challenging and requires  accurate predictions when $\Dtrain \ne \Dtest$. When the ratio between the density function of $\Dtest$ to that of $\Dtrain$ is bounded, rigorous \ood extrapolation guarantees exist and are detailed in Appendix~\ref{sec:gen_bounded_dens}. Such a situation arises when $\Dtest$ shares support with $\Dtrain$ but is differently distributed as depicted in Fig~\ref{subfig:ood}.

\textbf{Out-of-Support ({\oos}).} There are innumerable forms of distribution shift in which density ratios are not bounded. The most extreme case is when the support of $\Dtest$ is not contained in that of $\Dtrain$. I.e., when there exists some $\cX' \subset \cX$ such that
$\Pr_{x \sim \Dtest}[x \in \cX'] > 0$, but 
$\Pr_{x \sim \Dtrain}[x \in \cX'] = 0$
(see Fig~\ref{subfig:oos}). 
We term the problem of achieving low risk on such a $\Dtest$ as \emph{\oos{} extrapolation}.
 \newcommand{\Dtrainz}{\cD_{\mathrm{train,\cZ}}}
 \newcommand{\Dtrainv}{\cD_{\mathrm{train,\cV}}}
 \newcommand{\Dtestz}{\cD_{\mathrm{test,\cZ}}}
 \newcommand{\Dtestv}{\cD_{\mathrm{test,\cV}}}
  \newcommand{\Dtrainxone}{\cD_{\mathrm{train,\Xone}}}
 \newcommand{\Dtrainxtwo}{\cD_{\mathrm{train,\Xtwo}}}
 
 \newcommand{\Dtestxone}{\cD_{\mathrm{test,\Xone}}}
 \newcommand{\Dtestxtwo}{\cD_{\mathrm{test,\Xtwo}}}
 
\newcommand{\Xtwo}{\cX_2}
\newcommand{\Xone}{\cX_1}	\newcommand{\combll}[1][]{\ll_{#1,\mathrm{comb}}}

\textbf{Out-of-Combination ({\ooc}).} This 
is a special case of \oos{}. Let $\cX = \Xone \times \Xtwo$ be the product of two spaces. Let $\Dtrainxone$, $\Dtrainxtwo$  denote the marginal distributions of $x_1 \in \Xone$, $x_2 \in \Xtwo$ under $\Dtrain$, and $\Dtestxone,\Dtestxtwo$ 
under
$\Dtest$. In \ooc{} learning, 
$\Dtestxone$, $\Dtestxtwo$ are in the support of $\Dtrainxone$, $\Dtrainxtwo$, but the joint distributions $\Dtest$ need not be in the support of $\Dtrain$. 

\textbf{Transduction.} 
In classical transduction ~\citep{gammerman1998transduction}, given data points $\{x_i\}_{i=1}^l$ and their labels $\{y_i\}_{i=1}^l$, the objective is making predictions for points $\{x_i\}_{i=l+1}^{l+k}$. In this paper we refer to predictors of $x$ that are functions of the labeled data points $\{x_i\}_{i=1}^l$ as \textit{transductive predictors}.

\begin{figure}[!h]         
    \centering
    \hspace{-6pt}
    \begin{subfigure}{.26\textwidth}
      \centering
        \includegraphics[width=\linewidth]{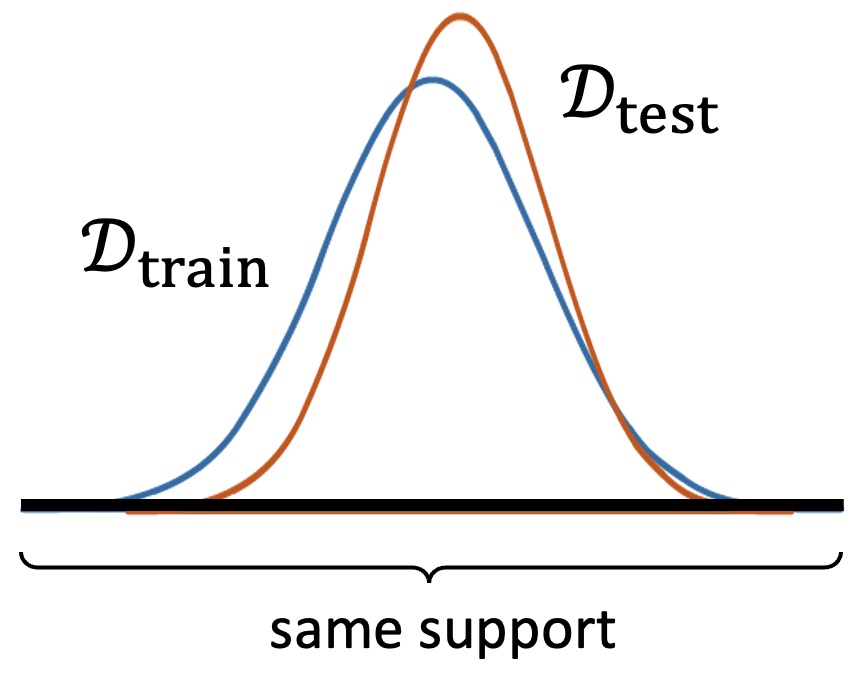}
      \caption{\footnotesize{\ood{}}}
      \label{subfig:ood}
    \end{subfigure} 
    \hspace{-3pt}
    \begin{subfigure}{.47\textwidth}
      \centering
        \includegraphics[width=\linewidth]{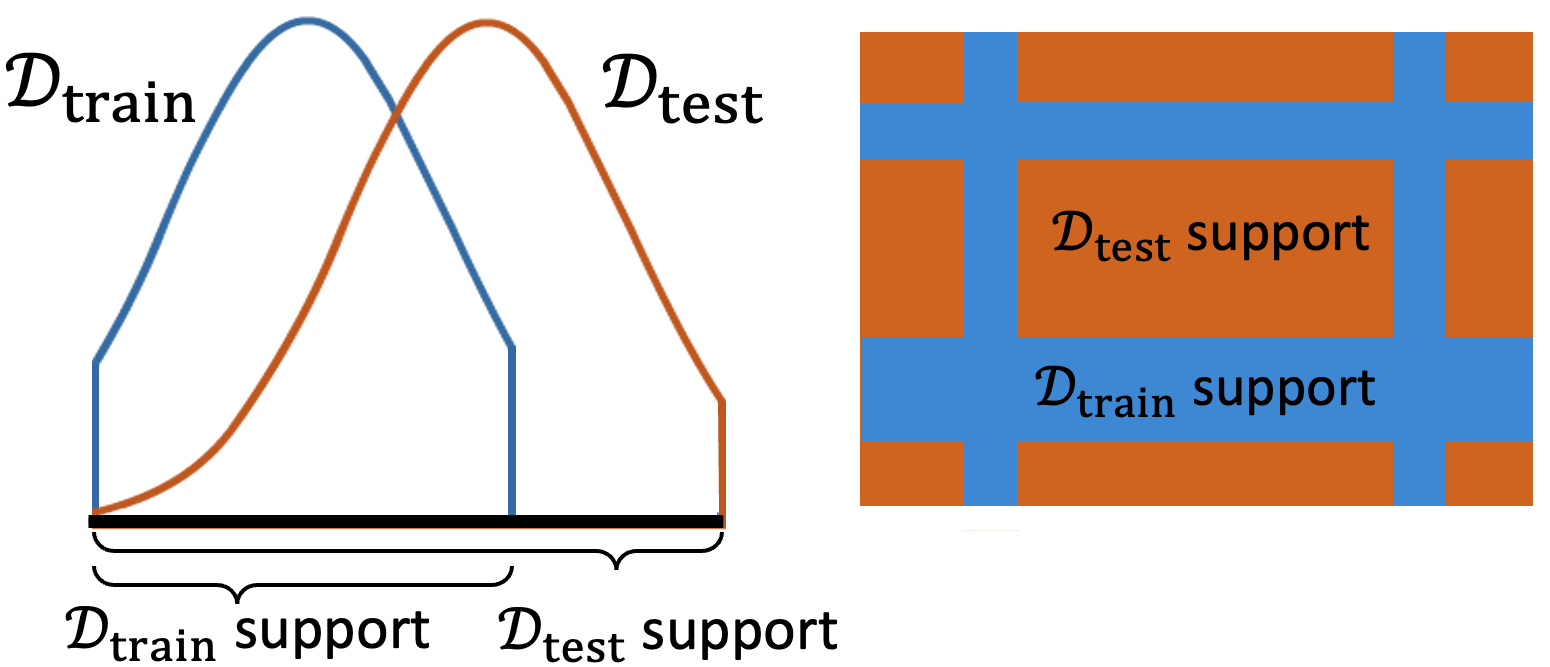}
      \caption{\footnotesize{General \oos{} }}
      \label{subfig:oos}
    \end{subfigure}  
    \hspace{-3pt}
    \begin{subfigure}{.27\textwidth}
      \centering
        \includegraphics[width=\linewidth]{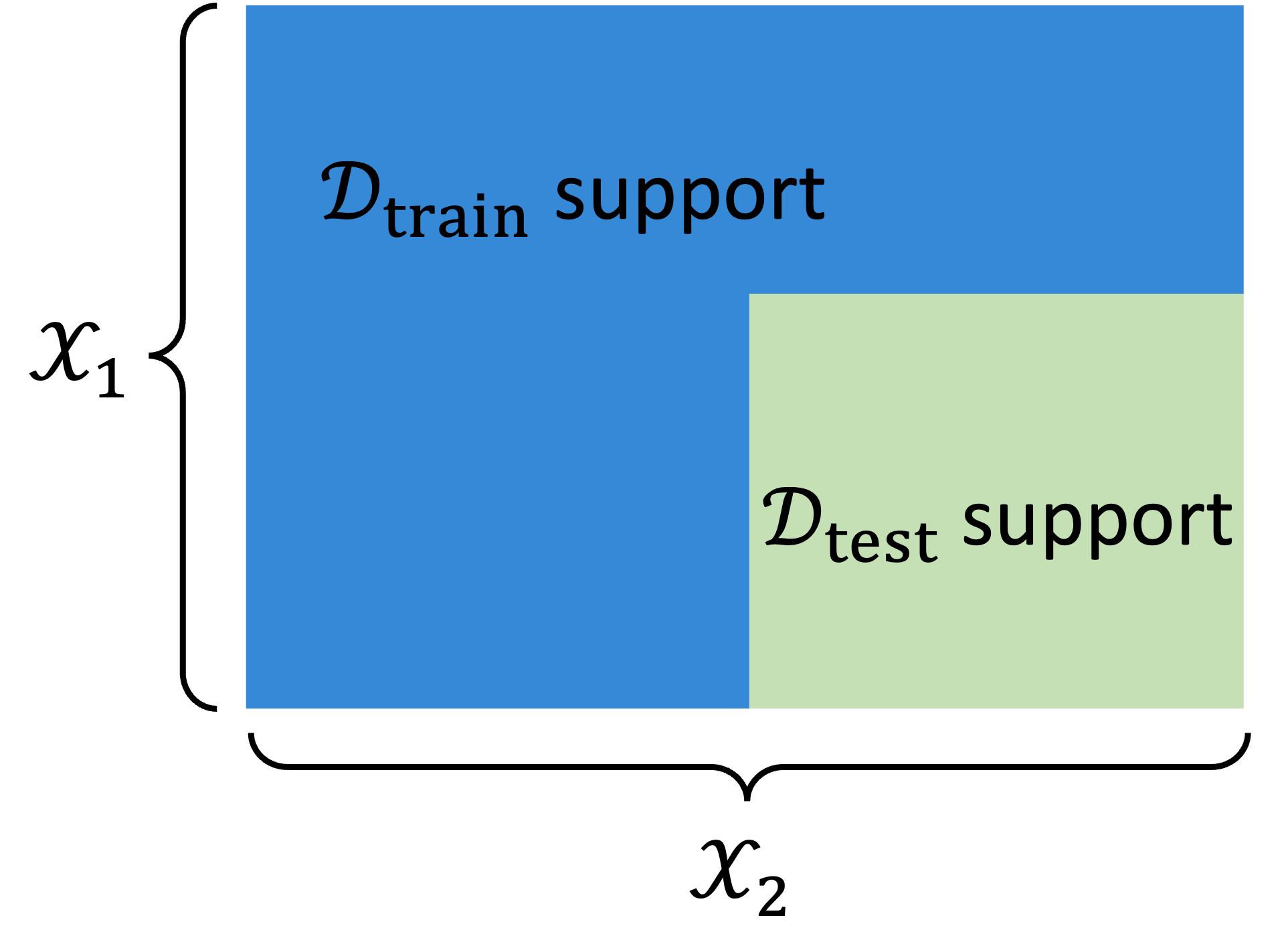}
      \caption{\footnotesize{Structured  \oos{}: \ooc{}}}
      \label{subfig:struct-oos}
    \end{subfigure} 
    \caption{Illustration of different learning settings. {\textbf{(a)}} in-support out-of-distribution (\ood) learning; {\textbf{(b)}} general out-of-support (\oos) learning in 1-D (on the left)  and 2-D (on the right); {\textbf{(c)}}  out-of-combination (\ooc) learning in 2-D. 
    }  
    \label{fig:illus_settings}
\end{figure}

\newcommand{\Dbartrain}{\bar{\cD}_{\mathrm{train}}}
\newcommand{\Dbartest}{\bar{\cD}_{\mathrm{test}}}
\newcommand{\hbarthet}{\bar{h}_{\btheta}}
\newcommand{\hbarthetk}[1][k]{\bar{h}_{\btheta,#1}}

\newcommand{\hbarstk}[1][k]{\bar{h}_{\star,#1}}

\newcommand{\DelX}{\Delta \cX}
\newcommand{\delx}{\Delta x}
\newcommand{\ftilthetk}{\tilde{f}_{\btheta,k}}
\newcommand{\gtilthetk}{\tilde{g}_{\btheta,k}}

\section{Conditions and Algorithm for Out-of-Support Generalization}
\label{section:oos-to-ooc}
While general \oos{} extrapolation can be arbitrarily challenging, we show that there exist some concrete conditions under which \ooc{} extrapolation is feasible. Under these conditions, an \oos{} prediction problem can be converted into an \ooc{} prediction problem. 


\newcommand{\Dtrans}{\cD_{\mathrm{trns}}}
\newcommand{\DXtrain}{\Delta\cX_{\mathrm{train}}}
\newcommand{\Xtrain}{\cX_{\mathrm{train}}}

 \newcommand{\Dtraindelx}{\cD_{\mathrm{train,\Delx}}}
	 \newcommand{\Dtestdelx}{\cD_{\mathrm{test,\Delx}}}

\begin{figure}[!h]
  \begin{center}
    \includegraphics[width=0.85\textwidth]{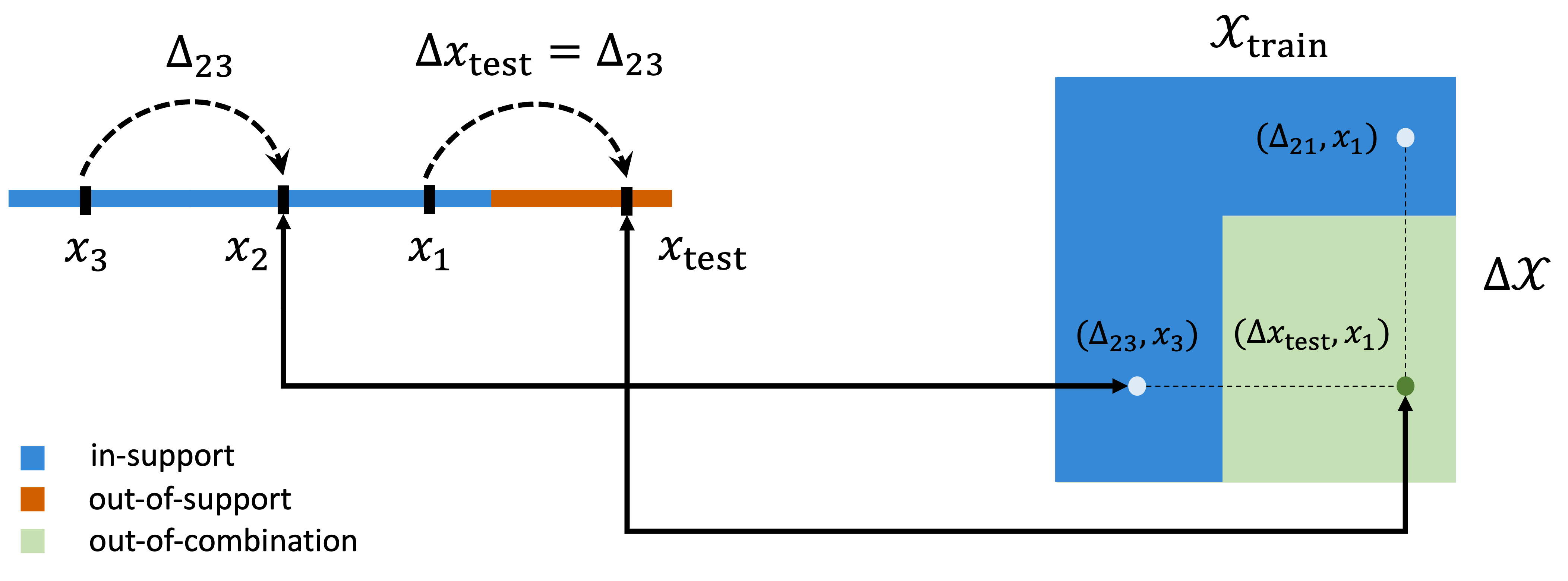} 
  \end{center}   
  \caption{\footnotesize{
  Illustration of converting \oos{} to \ooc{}.
  (\textbf{Left}) 
  Consider training points $x_1,x_2,x_3\in\mathcal{X}_{\mathrm{train}}$ and \oos{} test point $x_{\mathrm{test}}$.
  During training, we predict $h_{\btheta}(x_2)$ by transducing $x_3$ to $\hbarthet(\Delta_{23}, x_3)$, where $\Delta_{23}=x_2-x_3$. 
  Similarly, at test time, we predict $h_{\btheta}(x_{\mathrm{test}})$ by  transducing training point $x_1$, via $\hbarthet(\Delta x_{\mathrm{test}}, x_1)$, where $\Delta x_{\mathrm{test}}=x_{\mathrm{test}}-x_1$. In this example note that $\Delta_{23}=\Delta x_{\mathrm{test}}$.
  (\textbf{Right}) 
  This conversion yields an \ooc{} generalization problem in space $\Delta\mathcal{X}\times\mathcal{X}_{\mathrm{train}}$: marginal distributions $\Delta\mathcal{X}$ and $\mathcal{X}_{\mathrm{train}}$ are covered by the training distribution, but their \textit{combination} is not.
}}
\label{fig:illus_oos_ooc}
\end{figure}

\subsection{Transductive Predictors: Converting \oos{} to \ooc{}}
\label{ssec:oos-to-ooc}

We require that input space $\cX$ has group structure,  i.e., addition and subtraction operators $x+x',x-x'$ are well-defined for $x,x' \in \cX$. Let $\DelX := \{x-x': x,x' \in \cX\}$.
We propose a \textit{transductive reparameterization} $\hthet: \cX \to \simplex(\cY)$ with a \emph{deterministic} function $\hbarthet:\Delta\cX \times \cX \to \cY$ as
\begin{align}
\hthet(x) := \hbarthet(x-x',x'),  
\label{eq:param_transduct}
\end{align}
where $x'$ is referred to as an \textit{anchor} point for a \textit{query} point $x$. 

Under this reparameterization, the \textbf{training} distribution can be rewritten as a \textit{joint} distribution of $\delx = x - x'$  and $x'$ as follows:
\begin{align}
\Pr_{\Dbartrain}[(\delx, x') \in \cdot] &:= \Prob[ (\delx, x') \in \cdot ~ \mid  x\sim \Dtrain, ~x'{\sim \Dtrain}, ~\delx = x - x']. \label{equ:def_Dbartrain}
\end{align} 
This is just representing the prediction for querying   every point from the training distribution in terms of its relationship to other anchor points in the training distribution.

At \textbf{test} time, we are presented with query point $x \sim \Dtest$ that may be from an \oos{} distribution. To make a prediction on this \oos{} $x$, we observe that with a careful selection of an anchor point $x'$ from $\Dtrain$, our reparameterization may be able to convert this \oos{} problem into a more manageable \ooc{} one, since representing the test point $x$ in terms of its difference from training points can still be an ``in-support" problem. 
For a radius parameter $\rho>0$, define the distribution of chosen anchor points $\Dtrans(x)$ (referred to as a transducing distribution) as 
\vspace{-.1em}
\begin{align}
\Pr_{\Dtrans(x)}[x' \in \cdot] = \Prob[x' \in \cdot \mid x' \sim \Dtrain, ~ \inf_{\delx \in \DXtrain}\|(x - x') - \delx\| \le \rho], \label{eq:Dtrans}
\end{align}
where $\Xtrain$ denotes the set of $x$ in our training set, and we  denote $\DXtrain := \{x_1 - x_2 : x_1, x_2 \in \Xtrain\}$. Intuitively, our choice of $\Dtrans(x)$ selects anchor points $x'$ to transduce from the training distribution, subject to the resulting differences $(x - x')$ being close to a ``seen'' $\delx \in \DXtrain$. In doing so, both the anchor point $x'$ and the difference $\delx$ have been seen individually at training time, albeit not in combination. This allows us to express the prediction for a \oos{} query point in terms of an in-support anchor point $x'$ and an in-support difference $\delx$ (but not \emph{jointly} in-support). 
This choice of anchor points induces a joint \emph{test} distribution of $\delx = x - x'$  and $x'$: 
\begin{align}
\Pr_{\Dbartest}[(\delx, x') \in \cdot]  &:= \Prob[(\delx,x') \in \cdot ~ \mid x \sim \Dtest,~x' \sim \Dtrans(x), ~\delx = x -x']. \label{equ:def_Dbartest} 
\end{align} 

As seen from Fig~\ref{fig:illus_oos_ooc}, the marginals 
of $\delx$ and $x'$ under 
$\Dbartest$, are individually in the support of those under
 $\Dbartrain$. 
Still, as Fig~\ref{fig:illus_oos_ooc} reveals, since $x_{\text{test}}$ is out-of-support, the joint distribution of $\Dbartest$  is not covered by that of  $\Dbartrain$ (i.e., the combination of $x_1$ and $x_{\mathrm{test}}$ have not been seen together before); precisely the \ooc{} regime. 
Moreover, if one tried to transduce \emph{all} $x' \sim \Dtrain$ to $x \sim \Dtest$ at test time (e.g., transduce point $x_3$ to $x_{\mathrm{test}}$ in the figure) 
then we would lose coverage of the $\delx$-marginal. By applying transduction to keep both the marginal $x'$ and $\delx$ in-support, we are ensuring that we can convert difficult \oos{} problems into (potentially) more manageable \ooc{} ones. 


\vspace{-0.3em}
\subsection{Bilinear representations for \ooc{} learning}
\label{ssec:bilinear-ooc}
\vspace{-0.3em}

Without additional assumptions, \ooc{} extrapolation may be just as challenging as \oos. However, with certain low-rank structure it can be feasible~\citep{shah2020sample,agarwal2021causal,athey2021matrix}. 
Following \cite{agarwal2021persim}, we recognize that this low-rank property can be leveraged \emph{implicitly} for our reparameterized \ooc{} problem even in the continuous case (where $x, \Delta x$ do not explicitly form a finite dimensional matrix), using a bilinear representation of the \textit{transductive} predictor in \Cref{eq:param_transduct}, $\hbarthet(\delx ,x') = \langle f_{\btheta}(\delx), g_{\btheta}(x')\rangle$. Here $f_{\btheta}, g_{\btheta}$ map their respective inputs into a vector space of the same dimension, $\R^p$.\footnote{In our implementation, we take $\btheta = (\btheta_f,\btheta_g)$, with separate parameters for each embedding.} If the output space is $K$-dimensional, then we independently model the prediction for each dimension using a set of $K$ bilinear embeddings: 
\begin{align}\label{eq:pibarthet}
\hbarthet(\delx ,x') = (\hbarthetk[1](\delx,x'),\dots,\hbarthetk[K](\delx,x'));\quad 
\hbarthetk(\delx,x') = \langle \fthetk(\delx), \gthetk(x')\rangle.
\end{align} 
While $\hbarthetk$ are bilinear in embeddings $\fthetk,\gthetk$, the embeddings themselves may be parameterized by general function approximators. The effective ``rank" of the transductive predictor is controlled by the dimension of the continuous embeddings $\fthetk(\delx), \gthetk(x')$. 

We now introduce three conditions under which extrapolation is possible. As a preliminary, we write $\mu_1 \ll_{\kappa} \mu_2$ for two (unnecessarily normalized) measures $\mu_1,\mu_2$ to denote  $\mu_1(A) \le \kappa\mu_2(A)$ for all events $A$. This notion of coverage has been studied in the reinforcement learning community (termed as ``concentrability'' \citep{munos2008finite}), and implies that the support of $\mu_1$ is contained in the support of $\mu_2$. The first condition is the following notion of combinatorial coverage. 
\begin{assumption}[Bounded combinatorial density ratio]\label{defn:kap_comb_density_trans} We assume that $\Dbartest$ has {\it $\kappa$-bounded combinatorial density ratio} with respect to $\Dbartrain$, written as $\Dbartest \combll[\kappa] \Dbartrain$. This means that there exist distributions $\cD_{\DelX,i}$ and $\cD_{\cX,j}$, $i, j \in \{1,2\}$, over $\DelX$ and $\cX$, respectively, such that $\cD_{i\otimes j} :=  \cD_{\DelX,i}\otimes\cD_{\cX,j}$ satisfy  
    \begin{align*}
     \sum_{(i,j) \ne (2,2)}\cD_{i \otimes j} \ll_{\kappa} \Dbartrain, \quad \text{and } \Dbartest \ll_{\kappa} \sum_{i,j = 1,2}\cD_{i \otimes j}.
    \end{align*}
    \end{assumption}
    {\Cref{defn:kap_comb_density_trans} can be interpreted as a generalization of a certain coverage condition in matrix factorization. The following distribution is elaborated upon in \Cref{sec:app_mat_complete} and Fig~\ref{subfig:ooc-mat-comp} therein. 
    Intuitively, $\cD_{1\otimes 1}$ corresponds to the top-left block of a matrix, where all rows and columns are in-support. $\cD_{1\otimes 2}$ and $\cD_{2 \otimes 1}$ correspond to the off-diagonals, and $\cD_{2\otimes 2}$ corresponds to the bottom-right block. The condition requires that $\Dbartrain$ covers the entire top-right block and off-diagonals, but need not cover the bottom-right block $\cD_{2\otimes 2}$. $\Dbartest$, on the other hand, can contain samples from this last block. For problems with appropriate low-rank structure, samples from the top-left block, and samples from the off-diagonal blocks, uniquely determine the bottom-right block. 
     }

    Recall that $p$ denotes the dimension of the parametrized embeddings in \Cref{eq:pibarthet}. Our next  assumptions are that the ground-truth embedding is low-rank (we refer to as ``bilinearly transducible''), and that under $\cD_{1\otimes 1}$, its component factors do not lie in a strict subspace of $\R^p$, i.e., not degenerate.  
    
\begin{assumption}[Bilinearly  transducible]\label{asm:bilin_tranduc} For each component $k \in [K]$, there exists $\fstk : \Delta \cX \to \R^p$ and $\gstk: \cX \to \R^p$ such that for all $x \in \cX$, the following holds with probability $1$ over $x' \sim \Dtrans(x)$: $\hstk(x) = \hbarstk(\delx,x') := \langle \fstk(\delx),\gstk(x')\rangle$, where $ \delx = x - x'$. Further,   the ground-truth predictions are bounded: $\max_k\sup_{\delx,x'} |\hbarstk(\delx,x')| \le M$.
\end{assumption}
\begin{assumption}\label{asm:sing_val_lb} 
For all components $k \in [K]$, the following lower bound holds for some $\sigma^2 > 0$:
    \begin{align}
     \min\{\sigma_p(\Exp_{\cD_{1\otimes 1}}[\fstk\fstk^\top]),\sigma_p(\Exp_{\cD_{1\otimes 1}}[\gstk\gstk^\top])\} \ge \sigma^2 \label{eq:sig_p}.
    \end{align}
    \end{assumption}

The following is proven in \Cref{app:theory}; due to limited space, we defer further discussion of the bound to \Cref{app:further_remarks_thm_main}. 
\begin{theorem}\label{thm:main}  Suppose \Cref{defn:kap_comb_density_trans,asm:bilin_tranduc,asm:sing_val_lb} all hold, and that the loss $\ell(\cdot,\cdot)$ is the square loss. Then, if the training risk satisfies $\Risk(\hthet;\Dtrain) \le \frac{\sigma^2}{4\kappa}$, then  the test risk is bounded as
    \begin{align*}
    \Risk(\hthet;\Dtest) \le \Risk(\hthet;\Dtrain) \cdot \kappa^2\left(1+64\frac{M^4}{\sigma^4}\right) = \Risk(\hthet;\Dtrain) \cdot \mathrm{poly}(\kappa,\frac{M}{\sigma}).
    \end{align*}  
\end{theorem}

\subsection{Our Proposed Algorithm: Bilinear Transduction}
\label{subsec:algorithm}

Our basic proposal for \oos{} extrapolation is \textbf{bilinear transduction}, depicted in \Cref{alg:unweighted}:
at training time, a predictor $\Bar{h}_{\theta}$ is trained to make predictions for training points $x_i$ drawn from the training set $\Xtrain$ based on their similarity with other points $x_j\in\Xtrain$: $\Bar{h}_{\theta}(x_i-x_j,x_j)$.
The training pairs $x_i, x_j$ are sampled uniformly from $\Xtrain$. 
At test time, for an \oos{} point $x_{\text{test}}$, we first select an anchor point $x_i$ from the training set which has similarity with the test point $x_{\text{test}}-x_i$ that is within some radius $\rho$ of the training similarities $\DXtrain$. We then predict the value for $x_{\text{test}}$ based on the anchor point $x_i$ and the similarity of the test and anchor points: $\bar{h}_{\theta}(x_{\text{test}}-x_i,x_i)$. 

For supervised learning, i.e., the  regression setting, we compute differences directly between inputs $x_i,x_j\in\mathcal{X}$. For the goal-conditioned imitation learning setting, we compute difference between states $x_i,x_j\in\mathcal{X}$ sampled uniformly over demonstration trajectories. At test time, we select an anchor \textit{trajectory} based on the goal, and transduce each anchor state in the anchor trajectory to predict a sequence of actions for a test goal. As explained in Appendix~\ref{app:weighted}, there may be scenarios where it is beneficial to more carefully choose which points to transduce. To this end, we outline a more sophisticated proposal, \textbf{weighted transduction}, that achieves this. Pseudo code for bilinear transduction is depicted in Algorithm~\ref{alg:unweighted}, and for the weighted variant in Algorithm~\ref{alg:weighted} in the Appendix.

Note that we assume access to a representation of $\cX$ for which the subtraction operator captures differences that occur between training points and between training and test points. In Appendix~\ref{app:implementation} we describe examples of such representations as defined in standard simulators or which we extract with dimensionality reduction techniques (PCA). Alternatively, one might potentially extract such representations via self-supervision \citep{kingma2013auto} and pre-trained models \citep{upchurch2017deep}, or compare data points via another similarity function.

\begin{algorithm}[!h]
    \begin{algorithmic}[1]
    \STATE{}\textbf{Input:} \small Distance parameter $\rho$,   training set $(x_1,y_1),\dots,(x_n,y_n)$.
    \STATE{}\textbf{Train:} Train $\btheta$ on loss  $
    \cL(\btheta) = \textstyle\sum_{i=1}^n\sum_{j\ne i} \ell(\hbarthet(x_i -x_j,x_j),y_i)$
    \STATE{}\textbf{Test:} For each new $\xtest$, let $\cI(\xtest) := \{i: \inf_{\delx \in \DXtrain}\|\xtest - x_i - \delx\| \le \rho\}$, and predict
    \begin{align*}
    &y = \hbarthet(\xtest - x_{\mathbf{i}},x_{\mathbf{i}}), \text{~where}
    \,\, \mathbf{i} \sim \mathrm{Uniform}(\cI(\xtest))
    \end{align*}
    \end{algorithmic}
      \caption{Bilinear Transduction } \label{alg:unweighted}
    \end{algorithm}

%% file: body/experiments_iclr.tex
\section{Experiments}
\label{sec:experiments}
We answer the following questions through an empirical evaluation: 
\textbf{(1)} Does reformulating \oos{} extrapolation as a combinatorial generalization problem allow for extrapolation in a variety of supervised and sequential decision-making problems? 
\textbf{(2)} How does the particular choice of training distribution and data-generating function affect performance of the proposed technique?
\textbf{(3)} How important is the choice of the low-rank bilinear function class for generalization?
\textbf{(4)} Does our method scale to high dimensional state and action spaces? 
We first analyze \oos extrapolation on analytical domains and then on more real world problem domains depicted in Fig~\ref{fig:eval_domains}. 
\subsection{Analyzing \oos{} extrapolation on analytical problems}
\label{sec:analysis}
We compare our method on regression problems generated via 1-D analytical functions (described in Appendix~\ref{app:analytic-functions}) against standard neural networks with multiple fully-connected layers trained and tested in ranges $[20,40]$ and $[10,20]\cup[40,50]$ (with the exception of Fig~\ref{subfig:polynomial} trained in $[-1,1]$ and tested in $[-1.6,-1]\cup[1,1.6]$).
We use these domains to gain intuition about the following questions:

\textbf{What types of problems satisfy the assumptions for extrapolation?} While we outlined a set of conditions under which extrapolation is guaranteed, it is not apparent which problems satisfy these assumptions. To understand this better, we considered learning functions with different structure: a periodic function with mixed periods (Fig~\ref{subfig:periodic}), 
a sawtooth function (Fig~\ref{subfig:sawtooth})
and a polynomial function (Fig~\ref{subfig:polynomial}).  
Standard deep networks (yellow) fit the training points well (blue), but fail to extrapolate to \oos{} inputs (orange). In comparison, our approach (pink) accurately extrapolates on periodic functions but is much less effective on polynomials. This is because the periodic functions have symmetries which induce low rank structure under the proposed reparameterization.

\begin{figure}[!h]
    \centering
    \includegraphics[width=0.91\linewidth]{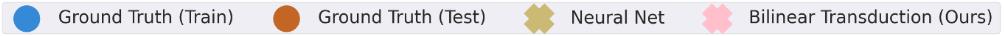}
    \begin{subfigure}{.3\textwidth}
      \centering
        \includegraphics[width=\linewidth]{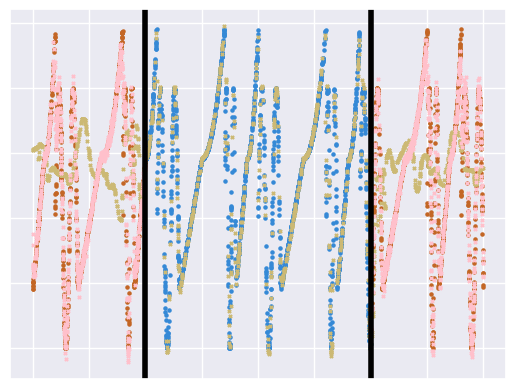}
      \caption{\footnotesize{Mixed periodic function}}
      \label{subfig:periodic}
    \end{subfigure}
    \begin{subfigure}{.3\textwidth}
      \centering
    \includegraphics[width=\linewidth]{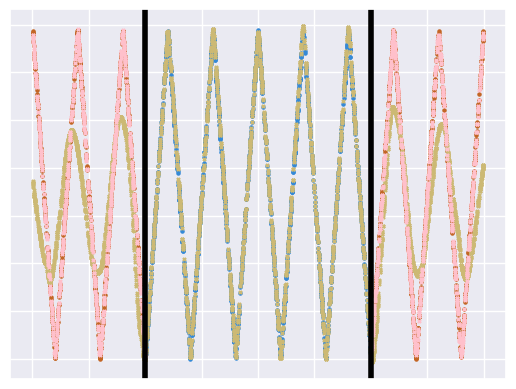}
      \caption{\footnotesize{Sawtooth function}}
    \label{subfig:sawtooth}
    \end{subfigure} 
    \begin{subfigure}{.3\textwidth}
      \centering
        \includegraphics[width=\linewidth]{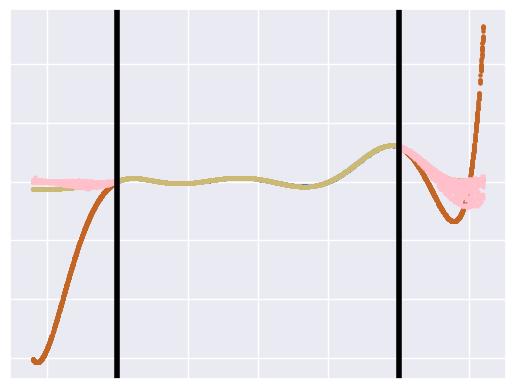}
          \caption{\footnotesize{Degree-8 polynomial}}
    \label{subfig:polynomial}
    \end{subfigure}
    \caption{\footnotesize{Bilinear transduction behavior on 1-D regression problems. Bilinear transduction performs well on functions with repeated structure, whereas they struggle on arbitrary polynomials. Standard neural nets fail to extrapolate in most settings, even when provided periodic activations~\citep{tancik20fourier}.}}
    \label{fig:diff_problems}
\end{figure}

\textbf{Going beyond known inductive biases.}  Given our method extrapolates to periodic functions in Fig~\ref{fig:diff_problems} (which displays shift invariance), one might argue that a similar extrapolation can be achieved by building in an inductive bias for periodicity / translation invariance. In Fig~\ref{fig:equivariant_function}, we show that bilinear transduction in fact is able to extrapolate even in cases that the ground truth function is not simply translation invariant, but is translation equivariant, showing that bilinear transduction can capture equivariance. Moreover, bilinear transduction can in fact go beyond equivariance or invariance. To demonstrate broader generality of our method, we consider a piecewise periodic function that also grows in magnitude (Fig~\ref{fig:beyond_equivariance}). This function is neither invariant nor equivariant to translation, as the group symmetries in this problem do not commute. The results in Fig~\ref{fig:beyond_equivariance} demonstrate that while the baselines fail to do so (including baselines that bake in equivariance (green)), bilinear transduction successfully extrapolates. The important thing for bilinear transduction to work is when comparing training instances, there is a simple (low-rank) relationship between how their labels are transformed. While it can capture invariance and equivariance, as Fig~\ref{fig:beyond_equivariance} shows, it is more general. 


\textbf{How does the relationship between the training distribution and test distribution affect extrapolation behavior?}
We try and understand how the range of test points that can be extrapolated to, depends on the training range. We show in Fig~\ref{fig:train_dependence} that for a particular ``width" of the training distribution (size of the training set), \oos{} extrapolation only extends for one ``width" beyond the training range since the conditions for $\DelX$ being in-support are no longer valid beyond this point. 

\begin{minipage}[c]{\linewidth}
    \includegraphics[width=\linewidth]{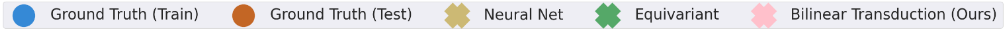}
\end{minipage}
\begin{minipage}[c]{0.3\linewidth}
    \centering
    \includegraphics[width=\linewidth]{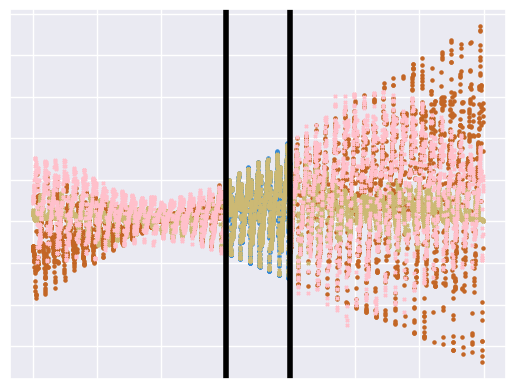}
    \captionof{figure}{\footnotesize{Performance of transductive predictors as test points go more and more \oos{}. Predictions are accurate for one ``data-width" outside training data.}}
    \label{fig:train_dependence}
\end{minipage}
\hfill
\begin{minipage}[c]{0.3\linewidth}
    \centering
    \includegraphics[width=\linewidth]{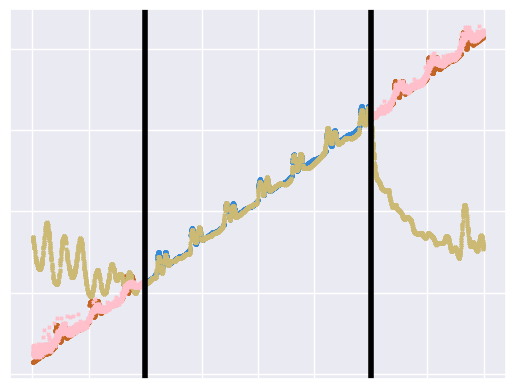}
    \captionof{figure}{\footnotesize{Prediction on function that displays affine equivariance. Bilinear trandsuction is able to capture equivariance without this being explicitly encoded.}}
    \label{fig:equivariant_function}
\end{minipage}
\hfill
\begin{minipage}[c]{0.3\linewidth}
    \centering
    \includegraphics[width=\linewidth]{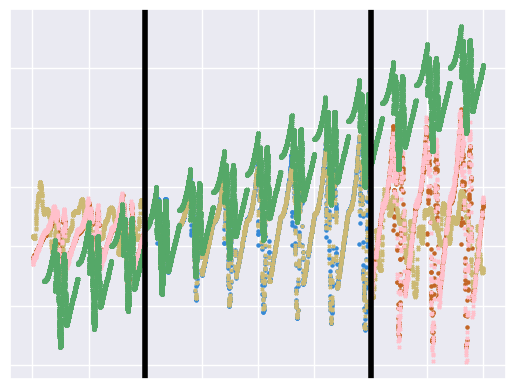}
    \captionof{figure}{\footnotesize{Prediction on function that is neither invariant nor equivariant. Bilinear transduction is able to extrapolate while an equivariant predictor fails.}}
    \label{fig:beyond_equivariance}
\end{minipage}

\subsection{Analyzing \oos{} extrapolation on larger scale decision-making problems}

To establish that our method is useful for complex and real-world problem domains, we also vary the complexity (i.e., working with high-dimensional observation and action spaces) and the learning setting (regression and sequential decision making). 

\begin{figure}[!h]
    \centering
    \includegraphics[width=0.2\linewidth]{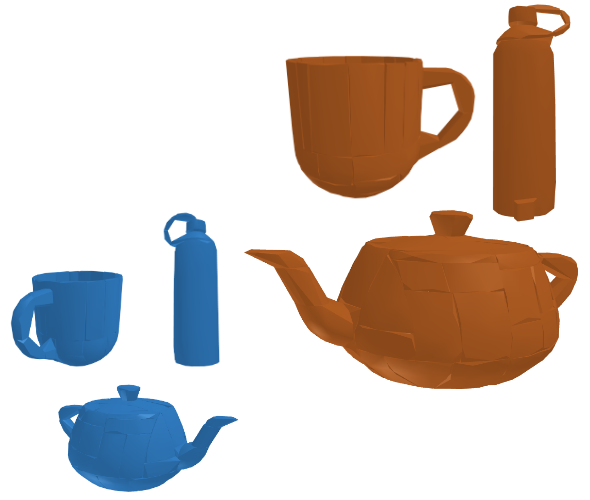}
    \hfill
    \includegraphics[width=0.25\linewidth]{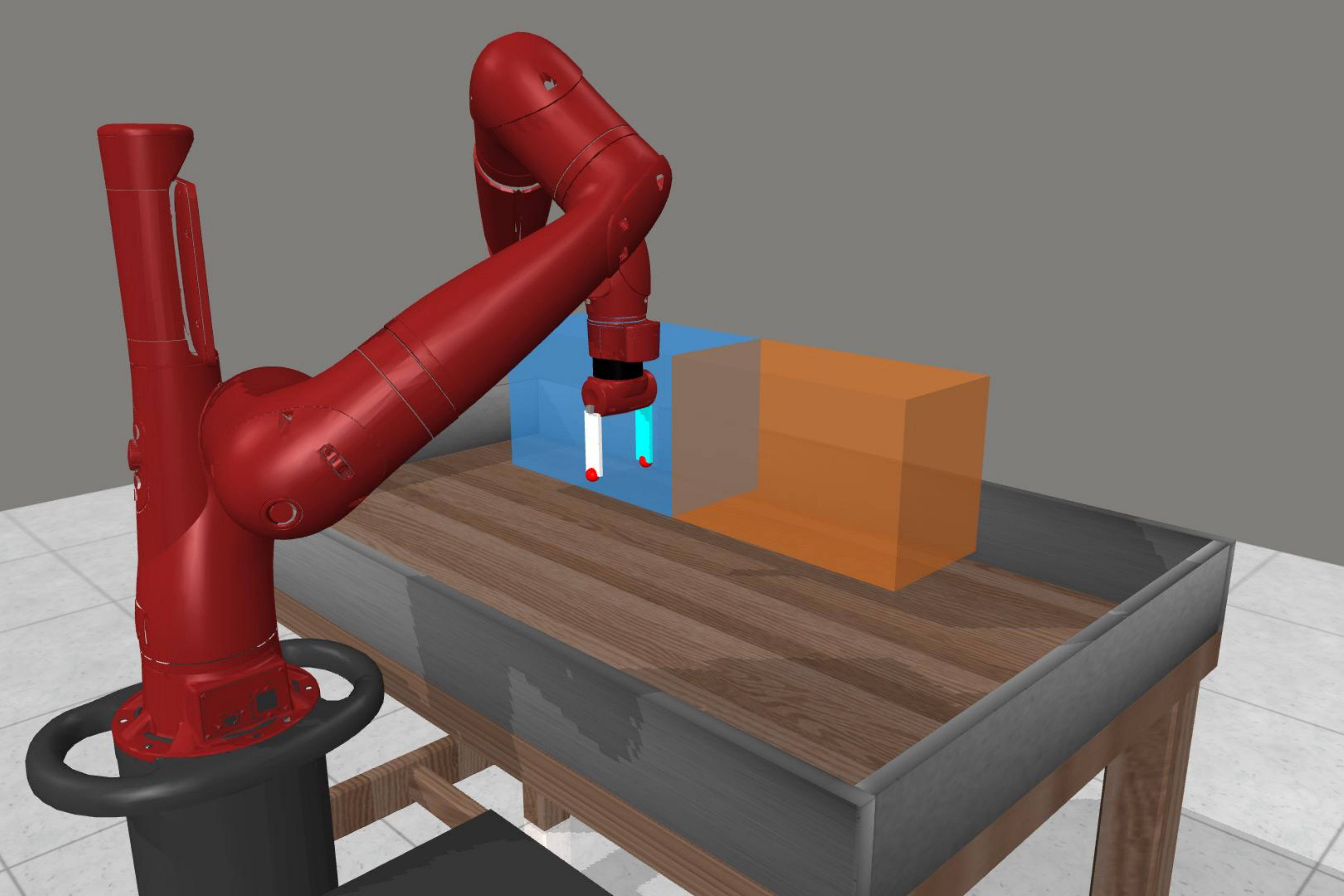}
    \hfill
    \includegraphics[width=0.25\linewidth]{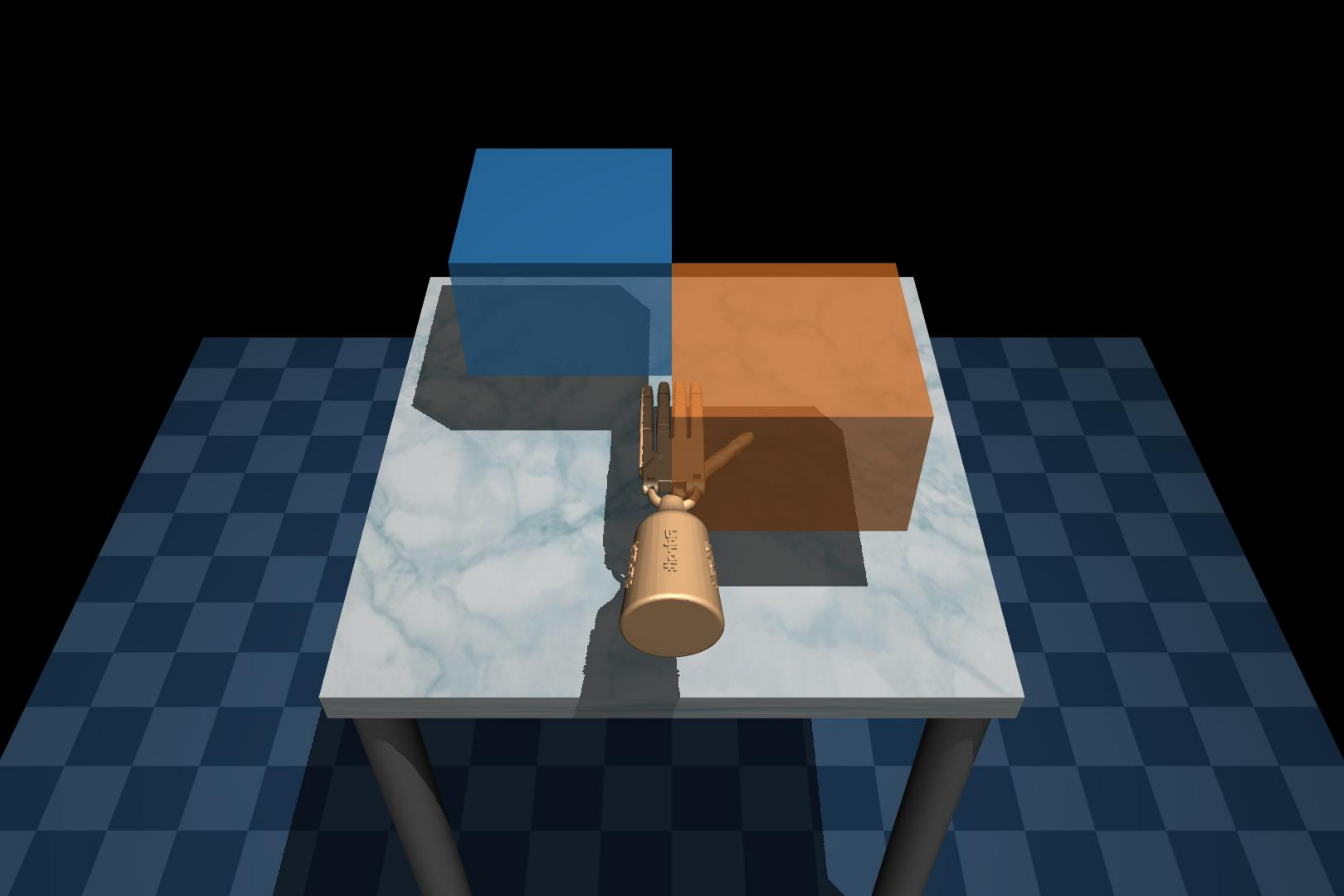}
    \hfill
    \includegraphics[width=0.24
    \linewidth]{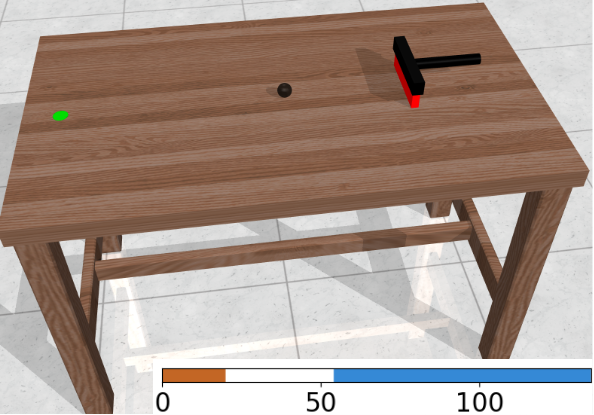}  
    \caption{\footnotesize{Evaluation domains at training (blue) and \oos{} (orange). (\textbf{Left to Right}:) grasp prediction for various object orientations and scales, table-top robotic manipulation for reaching and pushing objects to various targets, dexterous manipulation for relocating objects to various targets, slider control for striking a ball of various mass.
    }}
    \label{fig:eval_domains}
    \vspace{-0.3cm}
\end{figure}

\paragraph{Baselines.} To assess the effectiveness of our proposed scheme for extrapolation via transduction and bilinear embeddings, we compare with the following non-transductive baselines.
\emph{\underline{Linear Model}}: linear function approximator to understand whether the issue is one of overparameterization, and whether linear models would solve the problem.
\emph{\underline{Neural Networks}}: typical training of overparameterized neural network function approximator via standard empirical risk minimization. 
\emph{\underline{Alternative Techniques with Neural Networks (DeepSets)}}: an alternative architecture for combining multiple inputs (DeepSets \citep{zaheer2017deep}), that are meant to be permutation invariant and encourage a degree of generalization between different pairings of states and goals. 
Finally, we compare with a \emph{\underline{Transductive Method without a Mechanism for Structured Extrapolation (Transduction)}}: transduction with no special structure for understanding the impact of bilinear embeddings and the low-rank structure. This baseline uses reparameterization, but just parameterizes $\hbarthet$ as a standard neural network. We present additional comparison results introducing periodic activations~\citep{tancik20fourier} in Table~\ref{tab:fourier} in the Appendix. 


\textbf{\oos{} extrapolation in sequential decision making.}   
Table~\ref{tab:scaled-problems} contains results for different types of extrapolation settings. More training and evaluation details are provided in Appendices~\ref{app:domain_details} and \ref{app:training_details}.
\begin{itemize}[wide, labelindent=0pt, labelwidth=0pt, itemsep=-0.1em]
\vspace{-0.8em}
\item \textit{\underline{Extrapolation to \oos Goals:}} We considered two tasks from the Meta-World benchmark~\citep{yu2020meta} where a simulated robotic agent needs to either reach or push a target object to a goal location (column 2 in Fig~\ref{fig:eval_domains}). Given a set of expert demonstrations reaching/pushing to goals in the blue box ($[0, 0.4]\times[0.7, 0.9]\times[0.05, 0.3]$ and $[0, 0.3]\times[0.5, 0.7]\times\{0.01\}$), we tested generalization to \oos goals in the orange box ($[-0.4, 0]\times[0.7, 0.9]\times[0.05, 0.3]$ and $[-0.3, 0]\times[0.5, 0.7]\times\{0.01\}$), using a simple extension of our method described in Appendix~\ref{app:training_details} to perform transduction over trajectories rather than individual states. We quantify performance by measuring the distance between the conditioned and reached goal. 
Results in Table~\ref{tab:scaled-problems} show that on the easy task of reaching, training a linear or typical neural network-based predictor extrapolate as well as our method. However, for the more challenging task of pushing an object, our extrapolation is better by an order of magnitude than other baselines, showing the ability to generalize goals in a completely different direction. 

\item \textit{\underline{Extrapolation with Large State and Action Space:}} Next we tested our method on grasping and placing an object to \oos goal-locations in $\R^{3}$ with an anthropomorphic ``Adroit" hand with a much larger action ($\R^{30}$) and state ($\R^{39}$) space (column 3 in Fig~\ref{fig:eval_domains}). Results confirm that bilinear transduction scales up to high dimensional state-action spaces as well and is naturally able to grasp the ball and move it to new target locations ($[-0.3, 0]\times[-0.3, 0]\times[0.15, 0.35]$) after trained on target locations in ($[0, 0.3]\times[0, 0.3]\times[0.15, 0.35]$).
These results are better appreciated in the video attached with the supplementary material, but show the same trends as above, with bilinear transduction significantly outperforming standard inductive methods and non-bilinear architectures.

\item \textit{\underline{Extrapolation to \oos Dynamics:}} Lastly, we consider problems involving extrapolation not just in location of the goals for goal-reaching problems, but across problems with varying dynamics. Specifically, we consider a slider task where the goal is to move a slider on a table to strike a ball such that it rolls to a fixed target position (column 4 in Fig~\ref{fig:eval_domains}). The mass of the ball varies across episodes and is provided as input to the policy. We train and test on a range of masses ($[60, 130]$ and $[5, 15]$). We find that bilinear transduction adjusts behavior and successfully extrapolates to new masses, showing the ability to extrapolate not just to goals, but also to varying dynamics.
\end{itemize}

Importantly, bilinear transduction is significantly less prone to variance than standard inductive or permutation-invariant architectures. This can be seen from a heatmap over various training architectures and training seeds as shown in Fig~\ref{fig:heatmap} and Table~\ref{tab:3seeds} in Appendix~\ref{app:imitation}. While other methods can sometimes show success, only bilinear transduction is able to consistently accomplish extrapolation.

\textbf{\oos{} extrapolation in higher dimensional regression problems.}  
To scale up the dimension of the input space, we consider the problem of predicting valid grasping points in $\mathbb{R}^3$ from point clouds of various objects (bottles, mugs, and teapots) with different orientations, positions, and scales (column 1 in Fig~\ref{fig:eval_domains}). 
At training and test, objects undergo $z$-axis orientation ($[0,1.2\pi]$ and $ [1.2\pi, 2\pi]$), translation ($[0, 0.5]\times[0, 0.5]$ and $[0.5, 0.7]\times[0.5, 0.7]$) and scaling ($[0.7, 1.3]$ and $[1.3, 1.6]$). 
In this domain, we represent entire point clouds by a low-dimensional representation of the point cloud obtained via PCA. 
We consider situations, where we train on individual objects where the training set consists of various rotations, translations, and scales, and standard bilinear transduction is applied. We also consider situations where the objects are not individually identified but instead, a single grasp point predictor is trained on the entire set of bottles, mugs, and teapots. We assume access to category labels at training time, but do not require this at test time. For a more in-depth discussion of this assumption and a discussion of how this can be learned via \emph{weighted}-transduction, we refer readers to Appendix~\ref{app:weighted}. 
While training performance is comparable in all instances, we find that extrapolation behavior is significantly better for our method.  This is true both for single object cases as well as scenarios where all objects are mixed together (Table~\ref{tab:scaled-problems}). These experiments show that bilinear transduction can work on feature spaces of high-dimensional data such as point clouds. Note that while we only predict a single grasp point here, a single grasp point prediction can be easily generalized to predict multiple keypoints instead ~\citep{manuelli19kpam}, enabling success on more challenging control domains. For further visualizations and details, please see Appendices~\ref{app:additional-results} and \ref{app:implementation}.  






\begin{table}[t!]
  \caption{\footnotesize{
  Mean and standard deviation over prediction (regression) or final state (sequential decision making) error for \oos{} samples and over a hyperparameter search. 
  }
  }
  \label{tab:scaled-problems}
  \centering
 \scalebox{0.8}{
  \begin{tabular}{llllllll}
    Task & Expert 
     & Linear & Neural Net & DeepSets & Transduction & Ours \\
    \midrule
    Mug & & $0.068\pm0.013$ & $0.075\pm0.046$ & & $0.055\pm0.043$ & $\mathbf{0.013\pm0.007}$ \\   
    \midrule
    Bottle && $0.026\pm0.005$ & $0.05\pm0.051$ && $0.027\pm0.016$ & $\mathbf{0.008\pm0.004}$ \\  
    \midrule
    Teapot && $0.095\pm0.02$ & $0.101\pm0.078$ && $0.043\pm0.02$ & $\mathbf{0.022\pm0.014}$ \\  
    \midrule
    All && $0.143\pm0.116$ & $0.118\pm0.075$ && $0.112\pm0.08$ & $\mathbf{0.018\pm0.012}$ \\     
    \midrule
    \midrule    
    %
    Reach & $0.006\pm0.008$ & $0.007\pm0.006$ & $0.036\pm0.054$ & $0.19\pm0.209$ & $0.036\pm0.048$ & $\mathbf{0.007\pm0.006}$ \\
    %
    %
    %
    %
    \midrule
    Push & $0.012\pm0.001$ & $0.258\pm0.063$ & $0.258\pm0.167$ & $0.199\pm0.114$ & $0.159\pm0.116$ & $\mathbf{0.02\pm0.017}$\\
    %
    %
    %
    %
    \midrule
    Slider & $0.105\pm0.066$ & $0.609\pm0.07$ & $0.469\pm0.336$ & $0.274\pm0.262$ & $0.495\pm0.339$ & $\mathbf{0.149\pm0.113}$ \\  
    %
    %
    \midrule
    Adroit & $0.035\pm0.015$ & $0.337\pm0.075$ & $0.331\pm0.203$ & $0.521\pm0.457$ & $0.409\pm0.32$ & $\mathbf{0.147\pm0.117}$ \\    
    %
  \end{tabular}
  }
\vspace{-0.3cm}
\end{table}



%% file: body/related.tex

\vspace{-0.2cm}
\section{Abridged Related Work}
\vspace{-0.2cm}


Generalization from training data to test data of the same distribution has been studied extensively, both practically   \citep{simonyan2014very} and theoretically \citep{vapnik2006estimation,bousquet2002stability,bartlett2002rademacher}. 
Our focus in on performance on distributions with examples which may not be covered by the training data, as described formally in Section~\ref{sec:prelim}. Here we provide a discussion of the most directly related approaches, and defer an extended discussion to \Cref{app:unabridged_related}. Even more so than generalization, extrapolation necessitates leveraging structure in both the data and the learning algorithm.  Along these lines, past work has focused on \emph{structured neural networks}, which hardcode such symmetries as equivariance  \citep{cohen2016group, simeonov2022neural}, Euclidean symmetry \citep{smidt2021euclidean} and periodicity \citep{abu2021periodic, parascandolo2016taming}  into the learning process. Other directions have focused on general \emph{architectures} which seem to exhibit combinatorial generalization, such as transformers \citep{vaswani2017attention}, graph neural networks \citep{cappart2021combinatorial} and bilinear models \citep{hong2021bi}.


In this work, we focus on bilinear architectures with what we term a \emph{transductive parametrization.} We demonstrate that this framework can in many cases learn the symmetries (e.g. equivariance and periodicity) that structured neural networks harcode for, and achieve extrapolation in some regimes in which these latter methods cannot. A bilinear model with low inner dimension is equivalent to enforcing a low-rank constraint on one's predictions. Low-rank models have commanded broad popularity for matrix completion \citep{mnih2007probabilistic,mackey2015distributed}. And whereas early analysis focused on the \emph{missing-at-random} setting \citep{candes2009exact,recht2011simpler,candes2010matrix} equivalent to classical in-distribution statistical learning, we adopt more recent perspectives on missing-not-at-random data, see e.g., \citep{shah2020sample,agarwal2021causal,athey2021matrix}, which tackle the combinatorial-generalization setting described in \Cref{ssec:bilinear-ooc}. 

In the classical \emph{transduction} setting \citep{joachims2003transductive,gammerman1998transduction,cortes2006transductive}, the goal is to make predictions on a known set of \emph{unlabeled} test examples for which the features are known; this is a special case of semi-supervised learning.  We instead operate in a standard supervised learning paradigm, where test labels \emph{and} features are only revealed at test time.
Still, we find it useful to adopt a ``transductive parametrization'' of our predictor, where instead of compressing our predictor into parameters alone, we express predictions for labels of one example as a function of other, labeled examples. An unabridged version of related work is in  \Cref{app:unabridged_related}.


%% file: body/discussion.tex
\vspace{-0.3cm}
\section{Discussion}

Our work serves as an initial study of the circumstances under which problem structure can be both \emph{discovered} and \emph{exploited} for extrapolation, combining parametric and non-parametric approaches.
The main limitations of this work are assumptions regarding access to a representation of $\cX$ and similarity measure for obtaining $\Delta x$. Furthermore, we are only guaranteed to extrapolate to regions of $\cX$ that admit $\Delta x$ within the training  distribution.
A number of natural questions arise for further research. 
First, can we classify which set of real-world domains fits our assumptions, beyond the domains we have demonstrated? Second, can we \textit{learn} a representation of $\cX$ in which differences $\Delta x$ are meaningful for high dimensional domains? And lastly, are there more effective schemes for selecting anchor points? For instance, analogy-making for anchor point selection may reduce the complexity of $\hbarthet$, guaranteeing low-rank structure.





%% file: body/ethics.tex
\section{Ethics Statement}


\paragraph{Bias} In this work on extrapolation via bilinear transduction, we can only aim to extrapolate to data that shifts from the training set in ways that exist within the training data. Therefore the performance of this method relies on the diversity within the training data, allowing some forms of extrapolation but not to ones that do not exist within the training distribution.
\paragraph{Dataset release} We provide in detail in the Appendix the parameters and code bases and datasets we use to generate our data. In the future, we plan to release our code base and expert policy weights with which we collected expert data.

%% file: body/reproducibility.tex
\section{Reproducibility Statement}


We describe our algorithms in Section~\ref{subsec:algorithm} and the complete proof of our theoretical results and assumptions in Appendix~\ref{app:theory}.
Extensive implementation details regarding our algorithms, data, models, and optimization are provided in Appendix~\ref{app:implementation}.
In addition, we plan in the future to release our code and data.

%% file: body/acknowledgments.tex
\section{Acknowledgments}

We thank Anurag Ajay, Tao Chen, Zhang-Wei Hong, Jacob Huh, Leslie Kaelbling, Hannah Lawrence, Richard Li, Gabe Margolis, Devavrat Shah and Anthony Simeonov for the helpful discussions and feedback on the paper. We are grateful to MIT Supercloud and the Lincoln Laboratory Supercomputing Center for providing HPC resources.
This research was also partly sponsored by the DARPA Machine Common Sense Program, MIT-IBM Watson AI Lab, the United States Air Force Research Laboratory and the United States Air Force Artificial Intelligence Accelerator and was accomplished under Cooperative Agreement Number FA8750-19- 2-1000. The views and conclusions contained in this document are those of the authors and should not be interpreted as representing the official policies, either expressed or implied, of the United States Air Force or the U.S. Government. The U.S. Government is authorized to reproduce and distribute reprints for Government purposes, notwithstanding any copyright notation herein.

%% file: appendix/full_related_app.tex
\section{Unabridged Related Work}
\label{app:unabridged_related}
Here we discuss various prior approaches to extrapolation in greater depth, focusing on areas not addressed in our abridged discussion in the main text. 

\paragraph{Approaches which encode explicit structure.} One popular approach to designing networks that extrapolate to novel data has been the \emph{equivariant neural networks}, first proposed by \cite{cohen2016group}. The key idea is that, if it is known there is a group $G$ which acts on both the input and target domains of the predictor, and if it is understood that the true predictor must satisfy the \emph{equivariance property} 
\begin{align*}
h_{\star}(g \cdot x) = g \cdot h_{\star}(x)
\end{align*}
for all $g \in G$ (here $g\cdot ()$ denotes group action), then one can explicitly encode for predictors  $h_{\theta}$ satisfying the same identity. Similarly, one can encode for invariances, where $h_{\star}(g \cdot x) = h_{\star}(x)$ for all $g \in G$.

\cite{deng21vectorneuron} extended the  original equivariance framework to accommodate situations where the group $G$ corresponds to rotations ($SO(3)$), and \cite{simeonov2022neural} proposed neural descriptor fields which handle rigid-body transformations $SE(3)$).  For a broader review on how other notions of symmetry can be encoded into machine learning settings, consult \cite{smidt2021euclidean}, and  \cite{abu2021periodic, parascandolo2016taming}  for how periodic structure can be explicitly built in. We remark that in many of these approaches, the group/symmetry must be known ahead of time. While attempting to learn global \citep{benton2020learning} or local \citep{dehmamy2021automatic} equivariances/invariances, there are numerous forms of structure that can be represented as a group symmetry, and for which these methods do not apply. As we show in our experiments, there are a number of group structures which are \emph{not} captured by equivariance (such as Fig~\ref{fig:beyond_equivariance}), which bilinear transduction captures. 

\paragraph{Architectures favorable to extrapolation. } There are numerous learning architectures which are purported to be favorable to extrapolation to novel-domains. For example, graph neural networks (GNNs) have been use to facilitate reasoning behavior in combinatorial environments \citep{battaglia2018relational,cappart2021combinatorial}, and the implicit biases of GNNs have received much theoretical study in recent years (see \cite {jegelka2022theory} and the references therein). Another popular domain for extrapolations has been sequence-to-sequence modeling, especially in the context of natural language processing. Here, the now-renowned Transformer architecture due to \cite {vaswani2017attention}, as well as its variants based on the same ``attention mechanism'' (e.g. \cite{kitaev2020reformer}) have become incredibly popular. Supervised by a sufficiently diverse set of tasks and staggering amount of data, these have shown broad population in various language understanding tasks \citep{devlin2018bert}, text-to-image generation \citep{ramesh2021zero}, and even quantitative reasoning \citep{lewkowycz2022solving}. Attention-based models tend to excel best in tasks involving language and when trained on massively large corpora; their ability to extrapolate in manipulation tasks \citep{zhou2022policy} and in more modest data regimes remains an area of ongoing research. Moreover, while recent research has aimed to study their (in-distribution) generalization properties \citep{edelman2022inductive}, their capacity for ``reasoning'' more generally remains mysterious \citep{zhang2022unveiling}. 

Lastly, a line of research has studied various bilinear models (e.g. \cite{hong2021bi,shah2020sample}), motivated by the extension of literature on matrix factorization discussed in our abridged related work. However, these methods require additional fine tuning when applied to novel goals and do not display zero shot extrapolation. Moreover, \cite{shah2020sample} requries a discretization of the state-action space in order to formulate the matrix completion problem, whereas bilinear transduction shows how the results hold in continuous bilinear form, without the need for discretization. Along similar lines, the Deep-Sets architecture of \cite{zaheer2017deep} aims for combinatorial extrapolation by embedding tokens of interest in a latent vector space on which addition operators can be defined. \cite{zhou2022policy} compares the Deep Sets and Transformers approaches as policy architectures in reinforcement learning, finding that neither architecture uniformly outperforms the other.

\paragraph{Distributional Robustness.} One popular approach to out-of-distribution learning has been distributional robustness, which seeks predictors that perform on a family of ``nearby'' shifted distributions \cite{sinha18dro, rahimian19dro}. These approaches are well suited to OOD settings where the test distributions have the same support, but have differing (but boundedly different) probability densities.

\paragraph{Meta-Learning and Multi-Task Learning:} Meta-learning and multi-task learning methods aim to learn policies/predictors that can either zero-shot or very quickly generalize to new problems~\cite{thrunmetalearn, finn2017model, santoro16mann, vinyals16matching, kalashnikov21mtopt, crawshaw20surveymtl, rosenbaum18mtl, caruana97mtl}. These methods however usually make strict distributional assumptions about the training and test distribution of tasks being the same. So while the particular data distribution may be different, the meta-level assumption is still an in-distribution one \cite{fallah20mamltheory}. \cite{yin2019memorization} shows that under distribution shift, meta-learning methods can fail dramatically, and this problem will be exacerbated when supports shift as well. Multi-task learning methods often make the assumption that training on some set of training tasks can provide good pre-training before finetuning on new tasks ~\cite{julian20neverstoplearning, su22mtdialogue, meftah-etal-2020-multi}. However, it is poorly understand how these tasks actually relate to each other and the study has been largely empirical. 

\paragraph{Generalization in Reinforcement Learning and Imitation Learning:} Reinforcement learning problems are unique in that they are sequential and interactive. Distribution shift may not just be across different MDPs but within the states of a single MDP as well ~\cite{filos20shift, kirk21surveygen, ghosh21epistemic}. While we do not explicitly focus on the reinforcement learning setting, our results are largely in the imitation learning setting. Moreover, we are not looking at the finetuning setting but rather evaluating zero-shot performance on out of support goals, dynamics and arbitrary contexts. While certain works ~\cite{simeonov2022neural, wang22equivariant} do explore ideas of equivariance and invariance in RL and imitation learning, they are largely restricted to the either the $SE(3)$ or  $SO(2)$ groups. Moreover, as we discussed in our experiments, bilinear transduction is able to capture more complex notions than equivariance. 

The empirical evaluations we perform are in a large part on domains in goal conditioned reinforcement learning. Most goal conditioned RL works do not consider extrapolation beyond the range of training goals, but rather standard statistical generalization amongst training goals ~\cite{kaelbling93goals, Andrychowicz17HER, nair18rig}. While certain architectures ~\cite{hong2021bi} do seem to show some level of extrapolation, they still require significant finetuning. We show that with low-rank assumptions under reparameterization, zero-shot extrapolation is possible. 

Our work can also loosely be connected to analogy making techniques in machine learning ~\cite{reed15analogy, ichien21analogy, estruch22bisim, mitchell21analogy}. We formalize this notion in the context of out-of-support extrapolation and show that transduction allows for analogy making under low-rank structure assumptions.

%% file: appendix/theoretic_app.tex

\newcommand{\sfp}{\mathsf{p}}
\newcommand{\sfq}{\mathsf{q}}
\newcommand{\hhat}{\hat{h}}

\section{Theoretical Results}\label{app:theory}

This appendix is organized into four parts.
\begin{itemize}
    \item Appendix~\ref{sec:gen_bounded_dens} introduces the bounded density condition, and reproduces a folklore guarantee for extrapolation when one distribution has bounded density with respect to the other.
    \item Appendix~\ref{sec:extrap_matrix_completion} provides a basic guarantee for extrapolation in the context of missing-not-at-random matrix completion depicted in Fig~\ref{fig:matrix_illus}, based on results in \cite{shah2020sample}.
    \item Appendix~\ref{sec:extra_comb_support} formalizes a notion of bounded combinatorial density ratio (\Cref{defn:kap_comb_density_general}), in terms of which we can establish out-of-combination extrapolation guarantees by leveraging Appendix~\ref{sec:extrap_matrix_completion}.
    \item Finally, Appendix~\ref{sec:trans_extrap_bound} applies the results of the previous section, stating and proving formal guarantees for our proposed transductive predictors. 
\end{itemize} 

\subsection{Generalization under bounded density ratio}
\label{sec:gen_bounded_dens}

The following gives a robust, quantitative notion of when one distribution is in the support of another. For generality, we state this condition in terms of general positive measures $\mu_1,\mu_2$, which need not be normalized and sum to one.

\begin{definition}[$\kappa$-bounded density ratio]\label{defn:density_ratios} Let $\mu_1,\mu_2$ be two measures over a space $\Omega$. We say $\mu_1$ has $\kappa$-bounded density with respect to $\mu_2$, which we denote $\mu_1 \ll_{\kappa} \mu_2$, if for all measurable event\footnote{For simplicity, we omit concrete discussion of measurability concerns throughout.} $A \subset \Omega$, $\mu_1[A] \le \kappa \mu_2[A]$.
\end{definition}
Stating \Cref{defn:density_ratios} for general probability affords us the flexibility to write example, $\Pr_1 \ll_{\kappa} \Pr_2 + \Pr_3$, as $\Pr_2 + \Pr_3$ is a nonnegative measure with total mass $1+1 = 2$.

\begin{remark}[Connection to concentrability] The parameter $\kappa$ is known in the off-policy reinforcement learning literature as the \emph{concentratability} coefficient (see, e.g. \cite{munos2008finite}), and appears in controlling the performance of a rollout policy $\pi_1$  trained on a behavior policy $\pi_2$ by asserting that $\Pr_{\pi_1}\ll_{\kappa} \Pr_{\pi_2}$, where $\Pr_{\pi_i}$ is, for example, a state-action visitation probability under $\pi_i$. 
\end{remark}
\begin{remark}[Density nomenclature]
The nomenclature ``density'' refers to the alternative definition in terms of Radon-Nikodym derivatives. To avoid technicalities, this can be best seen when $\mu_1$ and $\mu_2$ are continuous densitives over $\Omega = \R^d$ with densities $p_{1}(\cdot)$ and $p_2(\cdot)$; e.g. for $i \in \{1,2\}$. $\mu_i[A]= \int_{x \in A} p_i(x)\rmd x$. Then $\mu_1 \ll_{\kappa} \mu_2$ is equivalent to 
\begin{align*}
\sup_{x} \frac{p_1(x)}{p_2(x)} \le \kappa.
\end{align*} 
\end{remark}

    The following lemma motivates the use of \Cref{defn:density_ratios}. Its proof is standard but included for completeness.
    \begin{lemma}\label{lem:change_of_measure} Let $\mu_1,\mu_2$ be measures on the same measurable space $\Omega$, and suppose that $\mu_2 \ll_{\kappa} \mu_1$.  Then, for any nonnegative function $\phi$, $\mu_2[\phi] \le \kappa \mu_1[\phi]$.\footnote{Here, $\mu[\phi] := \int \phi(\omega) \rmd \mu(\omega)$ denotes  the integration with respect to $\mu$.} In particular, if $\Dtest \ll_{\kappa} \Dtrain$, then as long as our loss function $\ell(\cdot,\cdot)$ is nonnegative,
    \begin{align*}
    \Risk(\hthet;\Dtest) \le \kappa \Risk(\hthet;\Dtrain).
    \end{align*}
    Thus, up to a $\kappa$-factor, $\Risk(\hthet;\Dtest)$ inherits any in-distribution generalization guarantees for $\Risk(\hthet;\Dtrain)$.
    \end{lemma}
    \begin{proof} As in standard measure theory (c.f. \citet[Chapter 1]{ccinlar2011probability}), we can approximate any $\phi \ge 0$ by a sequence of {\it simple functions} $\phi_n \uparrow \phi$, where $\phi_n(\omega) = \sum_{i=1}^{k_n} c_{n,i} \I\{\omega \in A_{n,i}\}$, with $A_{n,i} \subset \Omega$ and $c_{n,i} \ge 0$. For each $\phi_n$, we have
    \begin{align*}
    \mu_2[\phi_n] = \sum_{i=1}^{k_n} c_{n,i} \mu_2[A_{n,i}] \le \kappa \sum_{i=1}^{k_n} c_{n,i} \mu_1[A_{n,i}]  = \mu_1[\phi_n].
    \end{align*}
    The result now follows from the monotone convergence theorem. To derive the special case for $\Dtest$ and $\Dtrain$, apply the general result with nonnegative function $\phi(x) = \Exp_{y \sim \hthet(x)}\ell(y,\hst(x))$ (recall $\ell(\cdot,\cdot) \ge 0$ by assumption), $\mu_1 = \Dtrain$ and $\mu_2 = \Dtest$.
    \end{proof}

\subsection{Extrapolation for Matrix Completion}\label{sec:extrap_matrix_completion}

    In what follows, we derive a simple extrapolation guarantee for matrix completion. The following is in the spirit of the Nystr\"om column approximation (see e.g. \cite{gittens2013revisiting}), and our proof follows the analysis due to \cite{shah2020sample}. Throughout, consider
    \begin{align*}
    \bhatM = \begin{bmatrix}  \bhatM_{11} & \bhatM_{12}\\
    \bhatM_{21} & \bhatM_{22}
    \end{bmatrix}, \quad \bstM = \begin{bmatrix}  \bstM_{11} & \bstM_{12}\\
    \bstM_{21} & \bstM_{22}
    \end{bmatrix},
    \end{align*}
    where we decompose $\bhatM,\bstM$ into blocks $(i,j) \in \{1,2\}^2$ for dimension $n_i \times m_j$. 
    \begin{lemma}\label{lem:bhatM_diff} Suppose that $\bhatM$ is rank at most $p$, $\bstM$ is rank $p$, and
    \begin{align*}
    \forall (i,j) \ne (2,2), \quad \|\bhatM_{i,j} - \bstM_{i,j}\|_{\fro} \le \epsilon, \quad \text{and } \|\bstM_{i,j}\|_{\fro} \le M,
    \end{align*}
    where $\epsilon \le \sigma_p(\bstM_{11})/2$. Then,
    \begin{align*}
    \|\bhatM_{22} - \bstM_{22}\|_{\fro} \le8\epsilon\frac{M^2}{\sigma_p(\bstM_{11})^{2}}. 
    \end{align*}
    \end{lemma}
    \begin{proof} The proof mirrors that of \citet[Proposition 13]{shah2020sample}. We shall show below that $\bhatM$ is of rank exactly $p$. Hence, \citet[Lemma 12]{shah2020sample} gives the following exact expression for the bottom-right blocks, 
    \begin{align*}
    \bhatM_{22} = \bhatM_{21}\bhatM_{11}^{\dagger} \bhatM_{12}, \quad \bstM_{22} = \bstM_{21}(\bstM_{11})^{\dagger} \bstM_{12},
    \end{align*}
    where above $(\cdot)^{\dagger}$ denotes the Moore-Penrose pseudoinverse.
    Since $\|\bhatM_{11} - \bstM_{11}\|_{\op} \le\|\bhatM_{11} - \bstM_{11}\|_{\fro} \le \epsilon \le \sigma_{p}(\bstM_{11})/2$, Weyls inequality implies that $\bhatM_{11}$ is rank $p$ (as promised), and $\|\bhatM_{11}^{\dagger}\|_{\op} \le 2\sigma_p(\bstM_{11})^{-1}$. Similarly, as $\|\bhatM_{12} - \bstM_{12}\|_{\op} \le \sigma_{p}(\bstM_{11})/2 \le M/2$, so  $\|\bhatM_{12}\|_{\op} \le \frac{3}{2}M$. Thus,
    \begin{align*}
     \|\bhatM_{22} - \bstM_{22}\|_{\fro} &\le \|\bhatM_{21} - \bstM_{21}\|_{\fro}\|\bhatM_{11}^{\dagger}\|_{\op} \|\bhatM_{12}\|_{\op} + \|\bstM_{21}\|_{\op}\|\bhatM_{11}^{\dagger}\|_{\op} \| \bstM_{12} - \bhatM_{12}\|_{\fro} 
     \\
     &\qquad \|\bstM_{12}\|_{\op}\|\bstM_{21}\|_{\op}\|\bhatM_{11}^\dagger - (\bstM_{11})^\dagger\|_{\fro}\\
     &\le  \frac{5\epsilon M}{2\sigma_p(\bstM)}   + M^2\|\bhatM_{11}^\dagger - (\bstM_{11})^\dagger\|_{\fro}. 
    \end{align*}
    Next, using a perturbation bound on the pseudoinverse\footnote{Unlike \cite{shah2020sample}, we are interested in the Frobenius norm error, so we elect for the slightly sharper bound of \cite{meng2010optimal} above than the classical operator norm bound of \cite{stewart1977perturbation}.} 
    due to  \citet[Theorem 2.1]{meng2010optimal},
    \begin{align*}
    \|\bhatM_{11}^\dagger - (\bstM_{11})^\dagger\|_{\fro} &\le \|\bhatM_{11} - \bstM_{11}\|_{\fro} \max\{\|\bhatM_{11}^\dagger\|_{\op}^2, \|(\bstM_{11})^\dagger\|_{\op}^2\} \\
    &\le \epsilon \cdot 4 \sigma_{p}(\bstM_{11})^{-2}.
    \end{align*}
    Therefore, we conclude 
    \begin{align*}
    \|\bhatM_{22} - \bstM_{22}\|_{\fro} \le \frac{5\epsilon M}{2\sigma_p(\bstM)}   + \epsilon \frac{4M^2}{\sigma_p(\bstM_{11})^{2}} \le 8\epsilon\frac{M^2}{\sigma_p(\bstM_{11})^{2}}. 
    \end{align*}
    \end{proof}

\subsection{General Analysis for Combinatioral Extrapolation} \label{sec:extra_comb_support}

    We now provide our general analysis for combinatorial extrapolation. To avoid excessive subscripts, we write $\cX = \cW \times \cV$ rather than $\cX = \Xone \times \Xtwo$ as in the main body. We consider extrapolation under the following definition of combinatorial support. 
    \begin{restatable}[Bounded combinatorial density ratio]{definition}{defncombdens}\label{defn:kap_comb_density_general} Let $\cD,\cD'$ be two distributions over a product space $\cW\times \cV$.  We say $\cD'$ has \emph{ $\kappa$-bounded combinatorial density ratio} with respect to $\cD$, written as $\cD'  \combll[\kappa] \cD$, if there exist distributions $\cD_{\cW,i}$ and $\cD_{\cV,j}$, $i, j \in \{1,2\}$, over $\cW$ and $\cV$, respectively, such that $\cD_{i\otimes j} :=  \cD_{\cW,i}\otimes\cD_{\cV,j}$ satisfy  
    \begin{align*}
    \textstyle \sum_{(i,j) \ne (2,2)}\cD_{i \otimes j} \ll_{\kappa} \cD, \quad \text{and~~~} \cD' \ll_{\kappa} \sum_{i,j = 1,2}\cD_{i \otimes j}.
    \end{align*}
    \end{restatable}

    For simplicity, we consider scalar predictors, as the general result for vector valued estimators can be obtained by stacking the components. Specifically, we consider a ground-truth predictor $\hst$ and estimator $\hhat$ of the form
    \begin{align}
    \hst =\langle \fst,\gst \rangle, \quad \hhat = \langle \fhat, \ghat \rangle, \quad\fst,\fhat:\cW \to \R^p, ~\gst,\ghat:\cV \to \R^p. \label{eq:inner_prod}
    \end{align}
    Lastly, we choose the (scalar) square-loss, yielding the following risk
    \begin{align*}
    \Risk(\hhat;\cD) := \Exp_{(w,v) \sim \cD}[(\hst(w,v) - \hhat(w,v))^2].
    \end{align*}
    Throughout, we assume that all expectations that arise are finite.  Our main guarantee is as follows.
    \begin{theorem}\label{prop:extrap} Let $\cD,\cD'$ be two distributions on $\cW \times \cV$ satisfying $\cD' \combll[\kappa] \cD$, with corresponding factor distributions $\cD_{\cW,i}$ and $\cD_{\cV,j}$, $i,j \in \{1,2\}$. Define the effective singular value
    \begin{align}\label{eq:min_sing_val}
    \sigma_{\star}^2 := {\sigma_p(\Exp_{\cD_{\cW,1}}[\fst(w)\fst(w)^\top])\sigma_p\left(\Exp_{\cD_{\cV,1}}[\gst(v)\gst(v)^\top]\right)},
    \end{align}
    and suppose that $\max_{1 \le i,j \le 2} \Exp_{\cD_{i \otimes j}}|\hst(w,v)|^2 \le M_{\star}^2$. $\Risk(\hhat;\cD) \le \frac{\sigma_{\star}^2}{4\kappa}$,
    \begin{align*}
    \Risk(\hhat;\cD') \le\Risk(\hhat;\cD) \cdot \kappa^2\left(1+64\frac{M_{\star}^4}{\sigma_{\star}^4}\right) = \Risk(\hhat;\cD) \cdot \mathrm{poly}\Bigg(\kappa, \frac{M_{\star}}{\sigma_{\star}}\Bigg).
    \end{align*}
    \end{theorem}

    \subsubsection{Proof of \Cref{prop:extrap}}
        First, let us assume the following two conditions hold; we shall derive these conditions from the conditions of \Cref{prop:extrap} at the end of the proof:\footnote{Notice that here we take $M_\star^2$ as an upper bound of $\Exp_{\cD_{i \otimes j}}[\hst(w,v)^2]$, rather than a pointwise upper bound in \Cref{prop:extrap}. This is  for convenience in a limiting argument below.}
        \begin{align}
        \forall (i,j) \ne (2,2), \quad \Risk(\hhat;\cD_{i\otimes j}) \le \epsilon^2, \quad \Exp_{\cD_{i \otimes j}}[\hst(w,v)^2]\le M_\star^2, \quad \epsilon < \sigma_{\star}/2. \label{eq:cond_intermediate}
        \end{align}
        Our strategy is first to prove a version of \Cref{prop:extrap} for when $\cW$ and $\cV$ have finite cardinality by reduction to the analysis of matrix completion in \Cref{lem:bhatM_diff}, and then extend to arbitrary domains via a limiting argument.
        \begin{lemma}\label{lem:finite_out_of_dist} Suppose that \Cref{eq:cond_intermediate} hold, and in addition, that $\cW$ and $\cV$ have finite cardinality. Then, 
        \begin{align*}
        \Risk(\hhat;\cD_{2 \otimes 2}) = \|\bhatM_{22} - \bstM_{22}\|_{\fro}^2 \le  64\epsilon^2\frac{M_\star^4}{\sigma_{\star}^4}.
        \end{align*}
        \end{lemma}
        \begin{proof}[Proof \Cref{lem:finite_out_of_dist}] By adding additional null elements to either $\cW$ or $\cV$, we may assume without loss of generality that $|\cW| = |\cV| = d$, and enumerate their elements $\{w_1,\dots,w_d\}$ and $\{v_1,\dots,v_d\}$. Let $\sfp_{i,a} = \Prob_{w \sim \cD_{\cW,i}}[w = w_a]$ and $\sfq_{j,b} = \Prob_{v \sim \cD_{\cV,j}}[v = v_b]$.  Consider matrices $\bhatM,\bstM \in \R^{2d \times 2d}$, with $d \times d$ blocks
        \begin{align*}
        (\bhatM_{ij})_{ab} = \sqrt{\sfp_{i,a}\sfq_{j,b}}\cdot\hhat(w_a,v_b), \quad (\bstM_{ij})_{ab} = \sqrt{\sfp_{i,a}\sfq_{j,b}}\cdot\hst(w_a,v_b).
        \end{align*}
        We then verify that
        \begin{equation}
        \begin{aligned}
        \|\bhatM_{ij} - \bstM_{ij}\|_{\fro}^2 &= \sum_{a,b =1}^d \sfp_{i,a}\sfq_{j,b} (\hhat(w_a,v_b) - \hst(w_a,v_b))^2 \\
        &= \Exp_{\cD_{i \otimes j}}[(\hhat(w,v) - \hst(w,v))^2] = \Risk(\hhat;\cD_{i\otimes j}), 
        \end{aligned}\label{eq:bhatM_to_risk}
        \end{equation}
        and thus $\|\bhatM_{ij} - \bstM_{ij}\|_{\fro}^2 \le \epsilon^2$ for $(i,j) \ne (2,2)$.
        Furthermore, define the matrices $\bhatA_i,\bhatB_j$ via
        \begin{align*}
        (\bhatA_{i})_a := \sqrt{\sfp_{i,a}}\fhat(w_a)^\top, \quad (\bhatB_{j})_b := \sqrt{\sfq_{j,b}}\ghat(v_b)^\top,
        \end{align*}
        and define $\bstA_i,\bstB_j$ similarly. Then,
        \begin{align*}
        \bhatM = \begin{bmatrix} \bhatA_1\\
        \bhatA_2 \end{bmatrix}\begin{bmatrix} \bhatB_1\\
        \bhatB_2 \end{bmatrix}^\top, \quad 
        \bstM = \begin{bmatrix} \bstA_1\\
        \bstA_2 \end{bmatrix}\begin{bmatrix} \bstB_1\\
        \bstB_2 \end{bmatrix}^\top, 
        \end{align*}
        showing that $\rank(\bhatM_1), \rank(\bhatM_2) \le p$. Finally, by \Cref{eq:min_sing_val}, 
        \begin{align*}
        \sigma_p(\bstM_{11})^2 &= \sigma_p(\bstA_1(\bstB_1)^\top)^2 \ge \sigma_p^2(\bstA_1)\sigma_p^2(\bstB_1)\\
        &={\sigma_p\left( (\bstA_1)^\top\bstA_1\right)\sigma_p\left((\bstB_1)^\top\bstB_1\right)}\\
        &={\sigma_p\left( \sum_{a=1}^d \sfp_{1,a} \fhat(w_a)\fhat(w_a)^\top \right)\sigma_p\left(\sum_{b=1}^d \sfq_{1,b} \ghat(v_b)\ghat(v_b)^\top\right)}\\
        &={\sigma_p(\Exp_{\cD_{\cW,1}}[\fhat(w)\fhat(w)^\top])\sigma_p\left(\Exp_{\cD_{\cV,1}}[\ghat(v)\ghat(v)^\top]\right)} = \sigma_\star^2.
        \end{align*}
        Lastly, we have
        \begin{align*}
        \|\bstM_{i,j}\|_{\fro}^2 = \sum_{a,b} \sfp_{i,a}\sfq_{j,b} \hst(w_a,v_b)^2  = \Exp_{\cD_{i\otimes j}} \hst(w,v)^2 \le M_\star^2. 
        \end{align*}
        Thus, \Cref{eq:bhatM_to_risk,lem:bhatM_diff} imply that
        \begin{align*}
        \Risk(\hhat;\cD_{2 \otimes 2}) = \|\bhatM_{22} - \bstM_{22}\|_{\fro}^2 \le  64\epsilon^2\frac{M_\star^4}{\sigma_\star^4}.
        \end{align*}
        \end{proof}
        \begin{lemma}\label{lem:non_finite_dist}
        Suppose that \Cref{eq:cond_intermediate} hold, but unlike \Cref{lem:non_finite_dist}, $\cW$ and $\cV$ need not be finite spaces. Then, still, it holds that
        \begin{align*}
        \Risk(\hhat;\cD_{2 \otimes 2}) = \|\bhatM_{22} - \bstM_{22}\|_{\fro}^2 \le  64\epsilon^2\frac{M_\star^4}{\sigma_{\star}^4}.
        \end{align*}
        \end{lemma}
        \begin{proof}[Proof of \Cref{lem:non_finite_dist}] For $n \in \N$, define $h_{\star,n} = \langle f_{\star,n} , g_{\star,n} \rangle$ and $\hhat_{n} = \langle \hat{f}_n, \hat{g}_n \rangle$, where $f_{\star,n},\hat{f}_n,\hat{g}_n,g_{\star,n}$ are simple functions (i.e. finite range, see the proof of \Cref{lem:change_of_measure}) converging to $\fst,\fhat,\gst,\ghat$.  Define
        \begin{align*}
        \sigma_{\star,n}^2 &= {\sigma_p(\Exp_{\cD_{\cW,1}}[f_{\star,n}(w)f_{\fst,n}(w)^\top])\sigma_p\left(\Exp_{\cD_{\cV,1}}g_{\star,n}(v)g_{\star,n}(v)^\top]\right)},\\
        M_{\star,n}^2 &= \max_{i,j \ne (2,2)} \Exp_{\cD_{i\otimes j}}[\hst(w,v)^2]\\
        \epsilon_{n}^2 &= \max_{i,j \ne (2,2)} \Risk(\hhat_n;\cD_{i\otimes j}).
        \end{align*}
        By the dominated convergence theorem\footnote{Via standard arguments, one can construct the limiting embeddings $f_{\star,n},\hat{f}_n,\hat{g}_n,g_{\star,n}$ in such a way that their norms are dominated by integrable functions.}, 
        \begin{align*}
        \liminf_{n \ge 1}\sigma_{\star,n}^2 \ge \sigma_{\star}^2, \quad \limsup_{n \ge 1} M_{\star,n}^2 \le M_\star^2, \quad  \limsup_{n \ge 1}\epsilon_n^2 \le \epsilon^2.
        \end{align*}
        In particular,  as $\epsilon^2 \le \sigma^2/4$, then applying \Cref{lem:finite_out_of_dist} for $n$ sufficiently large, 
        \begin{align*}
        \Risk(\hhat_n;\cD_{2 \otimes 2}) \le 64\frac{M_{\star,n}^4}{\sigma_{\star,n}^4}\epsilon_n^2.
        \end{align*}
        Indeed, for any fixed $n$, all of $\fhat_n,\ghat_n,f_{\star,n},g_{\star,n}$ are simple functions, so we can partition $\cW$ and $\cV$ into sets on which these embeddings are constant, and thus treat $\cW$ and $\cV$ as finite domains; this enables the application of   \Cref{lem:finite_out_of_dist} applies.   Finally, using the dominated covergence theorem one last time,
        \begin{align*}
        \Risk(\hhat;\cD_{2 \otimes 2}) = \lim_{n \to \infty}\Risk(\hhat_n;\cD_{2 \otimes 2}) \le \limsup_{n\ge 1} 64\frac{M_{\star,n}^4}{\sigma_{\star,n}^4}\epsilon_n^2 \le 64\frac{M_\star^4}{\sigma_{\star}^4}\epsilon^2.
        \end{align*}
        \end{proof}
        We can now conclude the proof of our proposition. 
        \begin{proof}[Proof of \Cref{prop:extrap}]
        As $\cD' \ll_{\kappa} \sum_{i,j} \cD_{i \otimes j}$ and $\sum_{i,j \ne (2,2)} \cD_{i \otimes j} \ll_{\kappa} \cD$, \Cref{lem:change_of_measure} and additivity of the integral implies
        \begin{align*}
        \Risk(\hhat;\cD') &\le \kappa \Risk(\hhat;\cD_{2\otimes 2}) + \kappa \sum_{(i,j) \ne (2,2)}\Risk(\hhat;\cD_{i\otimes j})\\
        &\le \kappa \Risk(\hhat;\cD_{2\otimes 2}) + \kappa^2 \Risk(\hhat;\cD). \numberthis \label{eq:extra_intermediate_bound}
        \end{align*}
        Moreover,  setting $\epsilon^2 := \kappa \Risk(\hhat;\cD)$, we have
        \begin{align*}
        \max_{(i,j) \ne (2,2)}\Risk(\hhat;\cD_{i\otimes j}) \le \sum_{(i,j) \ne (2,2)}\Risk(\hhat;\cD_{i\otimes j}) \le \kappa \Risk(\hhat;\cD) := \epsilon^2.
        \end{align*}
        Thus, for $\Risk(\hhat;\cD) < \frac{\sigma_{\star}^2}{4\kappa}$, \Cref{eq:cond_intermediate} holds and thus \Cref{lem:non_finite_dist} entails
        \begin{align*}
        \Risk(\hhat;\cD_{2\otimes 2}) \le 64 \epsilon^2 \frac{M_{\star}^4}{\sigma_{\star}^4} = 64\kappa\Risk(\hhat;\cD) \frac{M_{\star}^4}{\sigma_{\star}^4} .
        \end{align*}
        Thus, combining with  \Cref{eq:extra_intermediate_bound},
        \begin{align*}
        \Risk(\hhat;\cD') \le\kappa^2\Risk(\hhat;\cD) \cdot\left(1+64\frac{M_{\star}^4}{\sigma_{\star}^4}\right),
        \end{align*} 
        completing the proof.
        \end{proof}

\subsection{Extrapolation for Transduction}\label{sec:trans_extrap_bound}

    Leveraging \Cref{prop:extrap}, this section proves \Cref{thm:main}, thereby providing a formal  theoretical justification for predictors of the form \Cref{eq:pibarthet}.

\begin{proof}[Proof of \Cref{thm:main}] We argue by reducing to \Cref{prop:extrap}. The reparameterization of the stochastic predictor $\hthet$ in \Cref{eq:param_transduct}, followed by \Cref{asm:bilin_tranduc} allows us to write
\begin{align*}
\Risk(\hthet;\Dtrain) &= \Exp_{x \sim \Dtrain}\Exp_{y \sim \hthet(x)} \ell(y,\hst(x)) \\
&= \Exp_{x \sim \Dtrain}\Exp_{x' \sim \Dtrans(x)} \ell(\hbarthet(x- x',x'),\hst(x))\\
&= \Exp_{x \sim \Dtrain}\Exp_{x' \sim \Dtrans(x)} \ell(\hbarthet(x - x',x'),\hbarst(x - x', x')).
\end{align*} 
In the above display, the joint distribution of $(x-x',x')$ is precisely given by $\Dbartrain$ (see \Cref{equ:def_Dbartrain}). Hence, 
\begin{align*}
\Risk(\hthet;\Dtrain) = \Exp_{\Dbartrain}\ell(\hbarthet(\delx,x'),\hbarst(\delx, x')).
\end{align*}
Further, as $\ell(y,y') = \|y - y'\|^2$ is the square loss and decomposes across coordinates, 
\begin{align}
\Risk(\hthet;\Dtrain) = \sum_{k=1}^K\Exp_{\Dbartrain}(\hbarthetk(\delx,x') - \hbarstk(\delx, x'))^2. \label{eq:Dtrain_decomp}
\end{align}
By the same token, 
\begin{align*}
\Risk(\hthet;\Dtest) = \sum_{k=1}^K\Exp_{\Dbartest}(\hbarthetk(\delx,x') - \hbarstk(\delx, x'))^2.
\end{align*}
To conclude the proof, we remain to show that for all $k \in [K]$, we have
\begin{equation}
\begin{aligned}
&\Exp_{\Dbartest}(\hbarthetk(\delx,x') - \hbarstk(\delx, x'))^2 \le C_{\mathrm{prob}} \cdot \Exp_{\Dbartrain}(\hbarthetk(\delx,x') - \hbarstk(\delx, x'))^2, \\
&\quad \text{where } C_{\mathrm{prob}} = \kappa^2\left(1+64\frac{M^4}{\sigma^4}\right).
\end{aligned}\label{eq:Cprob_wts}
\end{equation}
Indeed, for each $k \in [K]$, we have 
\begin{align*}
\Exp_{\Dbartrain}(\hbarthetk(\delx,x') - \hbarstk(\delx, x'))^2 \overset{\text{(\Cref{eq:Dtrain_decomp})}}{\le} \Risk(\hthet;\Dtrain) \overset{(\text{by assumption})}{\le} \frac{\sigma^2}{4\kappa}.
\end{align*}
Hence \Cref{eq:Cprob_wts} holds by invoking \Cref{prop:extrap} with the correspondences $\cW \gets \Delta \cX$, $\cV \gets \cX$, $\sigma_{\star} \gets \sigma$, $M_{\star} \gets M$ and $\kappa \gets \kappa$. This concludes the proof.
\end{proof}

\subsubsection{Further Remarks on \Cref{thm:main}}\label{app:further_remarks_thm_main}
\begin{itemize}
\item The singular value condition, $\min\{\sigma_p(\Exp_{\cD_{\DelX,1}}[\fstk\fstk^\top]),\sigma_p(\Exp_{\cD_{\cX,1}}[\gstk\gstk^\top])\} \ge 
 \sigma^2 > 0$, mirrors the non-degeneracy conditions given in the past work in matrix completion  (c.f. \citep{shah2020sample}).
\item  The support condition $ \sup_{\delx,x'} |\hbarstk(\delx,x')| \le M$ is a mild boundedness condition, which (in light of \Cref{prop:extrap}) can be weakened further  to 
\begin{align*}
\max_{1 \le i,j \le 2} \Exp_{\cD_{i\otimes j}}[\hbarstk(\delx,x')^2] \le M^2,
\end{align*}
where $\cD_{i \otimes j}$ are the contituent distributions witnessing $\Dbartest  \combll[\kappa]\Dbartrain$. 
\item The final condition, $\Risk(\hthet;\Dtrain) \le \frac{\sigma^2}{4\kappa}$, is mostly for convenience. Indeed, as $M \ge \sigma$ and $\kappa \ge 1$, then as soon as $\Risk(\hthet;\Dtrain) > \frac{\sigma^2}{4\kappa}$, our upper-bound on $\Risk(\hthet;\Dtest)$ is no better than 
\begin{align*}
\kappa M^2 \cdot \frac{64}{4} \cdot \frac{M^2}{\sigma^2} \ge 6M^2,
\end{align*}
which is essentially vacuous. Indeed, if we also inflate $M$ and stipulate that $\sup_{\delx,x'}|\hbarthet(\delx,x')| \le \sqrt{6}M$, we can remove the condition $\Risk(\hthet;\Dtrain) \le \frac{\sigma^2}{4\kappa}$ altogether. 
\end{itemize}

\subsection{Connections to matrix completion}\label{sec:app_mat_complete}

\begin{figure}[!h]
    \centering
    \hspace{-5pt} 
    \begin{subfigure}{.23\textwidth}
      \centering
        \includegraphics[width=\linewidth]{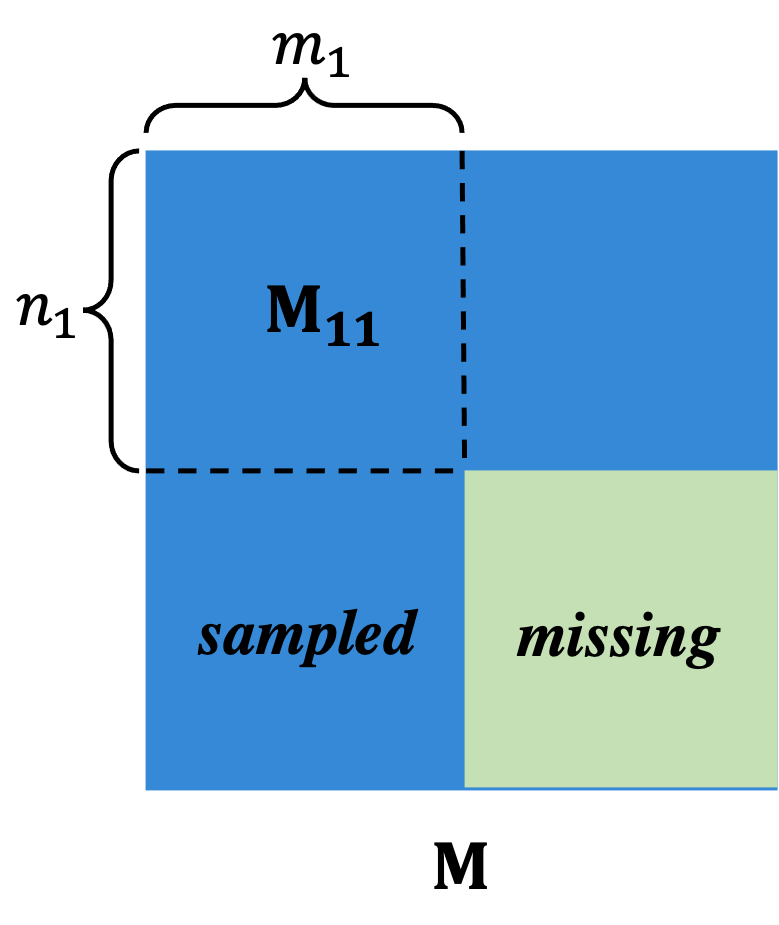}
      \caption{\footnotesize{Matrix completion}}
      \label{subfig:mat-comp}
    \end{subfigure}
    \hspace{35pt} 
    \begin{subfigure}{.62\textwidth}
      \centering
        \includegraphics[width=\linewidth]{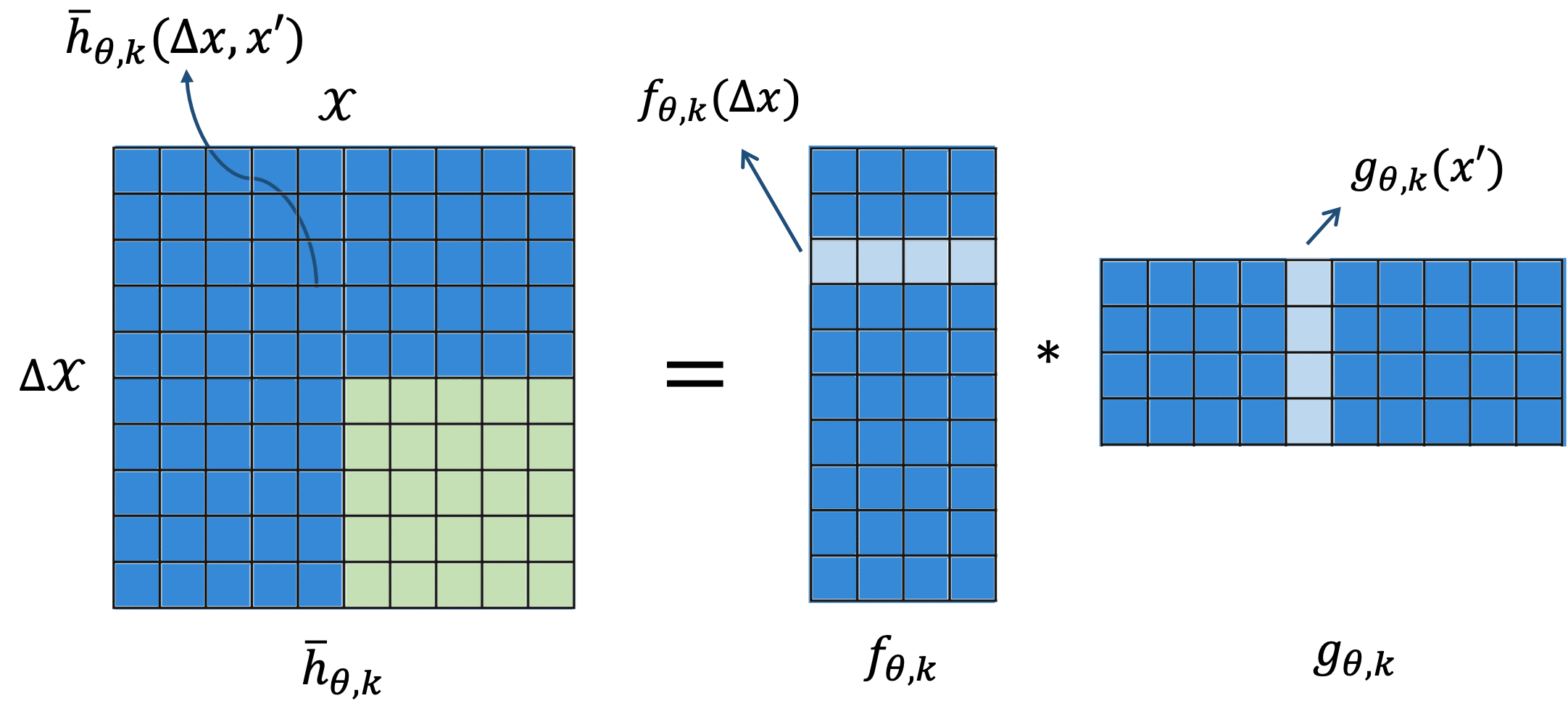}
      \caption{\footnotesize{Connecting \ooc{} to matrix completion}}
      \label{subfig:ooc-mat-comp}
    \end{subfigure}
   \caption{\footnotesize{Illustration of bilinear representations for \ooc{} learning, and connection to matrix completion. {\textbf{(a)}} An example of low-rank matrix completion, where both $\bM$ and $\bM_{11}$ have rank-$p$. Blue: support where entries can be accessed, green: entries are missing. {\textbf{(b)}} An example  that low-rank structure facilitates certain forms of \ooc{}, i.e.  for each $k\in[K]$, the predictor can be represented by bilinear  embeddings as $\hbarthetk(\delx,x') = \langle \fthetk(\delx), \gthetk(x')\rangle$.}}
    \label{fig:matrix_illus}
    \vspace{-0.4cm}
\end{figure}

Building upon \Cref{defn:density_ratios}, \Cref{defn:kap_comb_density_trans} introduces a notion of bounded density ratio between $\Dbartrain$ and $\Dbartest$ in the \ooc{} setting. Take the discrete case of matrix completion as an example, as illustrated  in Fig~\ref{fig:matrix_illus}, the training distribution of $(\delx,x')$ covers the support of the $(1,1),~(1,2),~(2,1)$ blocks of the matrix, while the test distribution of $(\delx,x')$ might be covered by any product of the marginals of the $2\times 2$ blocks. With this connection in mind, it is possible to establish the \ooc{} guarantees on  $\Dbartest$ as in matrix completion tasks, if the bilinear embedding admits some low-rank structure. In other words, samples from the top-left and off-diagonal blocks uniquely determine the bottom-right block.  

This is a standard setting studied extensively in the matrix completion literature, see \citep{shah2020sample,agarwal2021causal,athey2021matrix} for example. Our results in \Cref{thm:main} can be viewed as a generalization of these results to the continuous space embeddings. Finally, note that \Cref{asm:bilin_tranduc,asm:sing_val_lb} are also adopted correspondingly from the matrix completion literature, where \Cref{asm:bilin_tranduc} concerns  the ground-truth bilinear structure of the embeddings, which mirrors the {\it low-rank} structure of the underlying matrix in the matrix completion case;  \Cref{asm:sing_val_lb} corresponds to the non-degeneracy assumption on the the top-left block distribution, as referred to as ``anchor'' columns/rows  in the past work when dealing with matrices \citep{shah2020sample}.

%% file: appendix/Additional_Results.tex
\section{Additional Results}
\label{app:additional-results}

\subsection{Learning which points to transduce: Weighted Transduction}
\label{app:weighted}

There may be scenarios where it is beneficial to more carefully choose which points to transduce rather than comparing every training point to every other training point. To this end, we outline a more sophisticated proposal, \textbf{weighted transduction}, that achieves this.

Sampling training pairs uniformly may violate the low rank structure needed for extrapolation discussed in Section~\ref{ssec:bilinear-ooc}. Consider a dataset of point clouds of various 3D objects in various orientations where the task is to predict their 3D grasping point. A predictor required to make predictions from similarities between point clouds of any two different objects and orientations may need to be far more complex than simpler predictors trained on only similar objects. Therefore, it may be useful to identify which training points to transduce, rather than transduce between all pairs of points during training.

Hence, we introduce \textbf{weighted transduction} described in Algorithm~\ref{alg:weighted}, which identifies which training points to transduce, rather than transduce between all pairs of points during training.
During training, a transductive predictor $\bar{h}_{\theta}$ is trained together with a scalar bilinear \emph{weighting function} $\omega_{\theta}$. $\bar{h}_{\theta}$ is trained to predict values as in the bilinear transduction algorithm, yet weighted by $\omega_{\theta}$. $\omega_{\theta}$ predicts weights for data points and their potential anchors based on anchor point similarity and the anchor point, thereby deciding which points to transduce. We can either pre-train $\omega_{\theta}$, or jointly optimize $\omega_{\theta}$ and $\bar{h}_{\theta}$. At test time, we select an anchor point based on the learned weighting function, and make predictions via $\bar{h}_{\theta}$. We show results on learning with weighted transduction on analytic functions in Appendix~\ref{app:2D-analytic} and on learning how to grasp from point
clouds in Appendix~\ref{app:grasp_points}.

\begin{figure}[!h]
    \begin{algorithm}[H]
    \begin{algorithmic}[1]
    \label{app:alg_weighted}
    \STATE{}\textbf{Input:} \small distance parameter $\rho$,  training set $(x_1,y_1),\dots,(x_n,y_n)$, regularizer $\mathrm{Reg}(\cdot)$
    \STATE{}\textbf{Train:} Train $\btheta$ on loss
    \begin{align*} 
    \cL(\btheta) &= \sum_{i=1}^n\sum_{j\ne i} \omega_{\btheta}(x_i-x_j,x_j)
    \cdot\ell(\hbarthet(x_i -x_j,x_j),y_i) + \mathrm{Reg}(\btheta)
    \end{align*} 
    \
    \STATE{}\textbf{Test:} for each new $\xtest$,  predict
    \begin{align*}
    & y = \hbarthet(\xtest - x_\mathbf{i}, x_\mathbf{i}), \text{~where}
    \,\,\mathbf{i} \sim \Pr[\mathbf{i} = i] \propto  \omega_{\btheta}(\xtest-x_i,x_i)
    \end{align*}  
    \end{algorithmic}
      \caption{Weighted Transduction}
      \label{alg:weighted}
    \end{algorithm}
\end{figure}

\subsection{Additional Results on a Simple 2-D Example}
\label{app:2D-analytic}
\paragraph{Weighted transduction with a 2-D example}
We also considered a 2-D analytic function as shown in Fig~\ref{fig:2d_bilinear}. This function has a 2-D input $(x_1,x_2)$, and 2-D output $(y_1,y_2)$. The function is constructed by sampling a tiling of random values in the central block (between $(1, 1)$ and $(5, 5)$). The remaining tiles are constructed by shifting the central tile by $6$ units along the $x_1$ or $x_2$ direction. For tiles shifted in the $x_1$-direction ($(6, 0)$ or $(0, 6)$) the label has the same first dimension and a negated second dimension. For tiles shifted in the $x_2$-direction ($(6, 0)$ or $(0, 6)$) the label has the same second dimension and a negated first dimension. This has a particular symmetry in its construction where the problem has a very simple relationship between instances, but not every instance is related to every other instance (only those off by a constant offset). We show in Fig~\ref{fig:2d_bilinear} that weighted transduction is particularly important in this domain as not every point corresponds to every other point. We perform weighted transduction (seeded by a few demonstration pairs), and this allows for extrapolation to \oos{} pairs. While this may seem like an \ooc{} problem, we show through the standard bilinear (without reparameterization) comparison that there is no low rank structure between the dimensions but instead there is low rank structure on reparameterization. This shows the importance of using weighted transduction and the ability to find symmetries in the data with our proposed method. 

\begin{figure}[!h]
    \centering
    \centering
    \includegraphics[width=\linewidth]{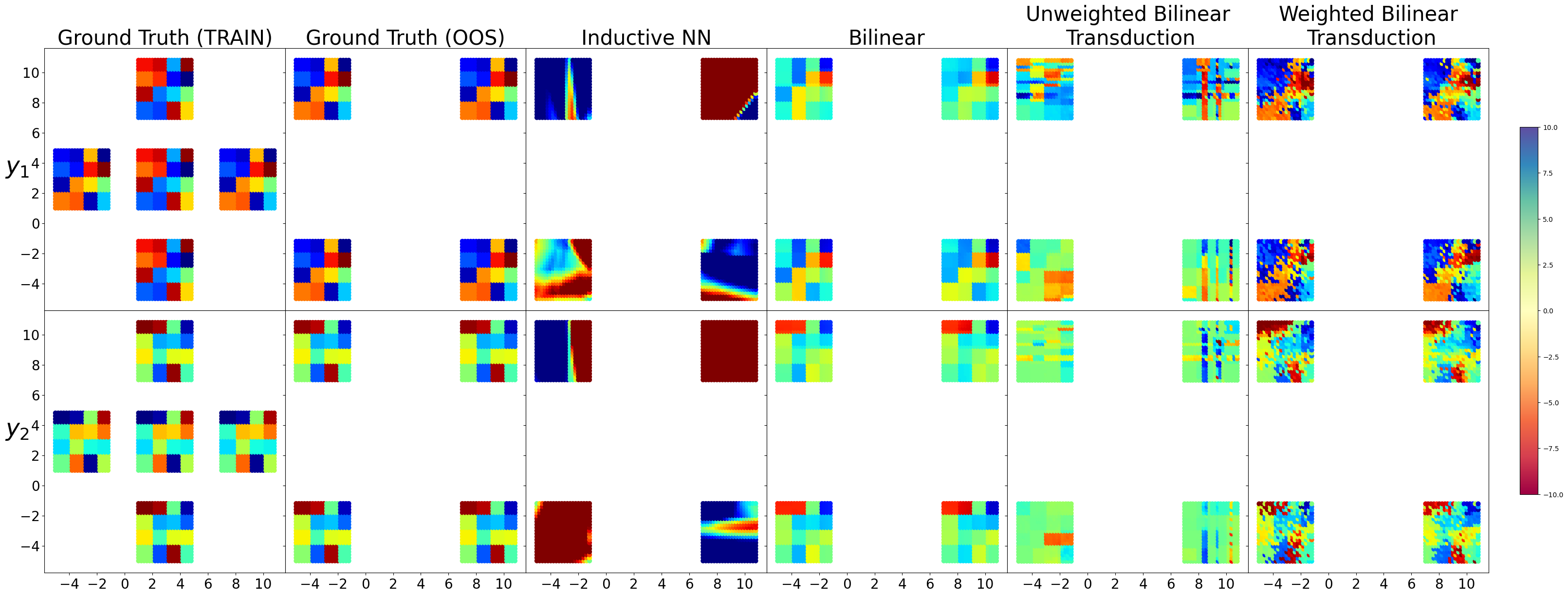}
    \caption{\footnotesize{Predictions for the 2-D analytic function from $(x_1,x_2)$ to $(y_1,y_2)$. (\textbf{Top}) $y_1$ output values for inputs $(x_1,x_2)$, (\textbf{Bottom}) $y_2$ output values.  (\textbf{Left to Right}:) In-distribution ground truth values, \oos{} ground truth values, neural network predictions, bilinear predictions, bilinear transduction predictions and weighted transduction predictions.
    Transduction weighting is important in this domain and it is able to discover problem symmetry. While this may seem like an \ooc{} problem, since bilinear prediction directly doesn't work, the \ooc{} view on this problem does not have low rank structure, while the \oos{} view does.}}
    \label{fig:2d_bilinear}
\end{figure}

\subsection{Additional Results on Supervised Grasp prediction}
\label{app:grasp_points}

\paragraph{Qualitative grasping prediction results} For qualitative results of all models on the grasping point prediction task, please see Fig~\ref{fig:mug results}.

\paragraph{Bilinear transduction is comparable to an SE(3) equivariant architecture on \oos{} orientations and exceeds its performance on \oos{} scales}
We compare bilinear transduction to a method that builds in SE(3) equivariance domain knowledge into its architecture and operates on point clouds. 
Specifically, we compare to tensor field networks \cite{thomas2018tensor}, and report the results in Table~\ref{tab:se3}.
To highlight that SE(3) equivariance architectures can extrapolate to new rotations but lack guarantees on scale, we train on $100$ noisy samples of rotated mugs, bottles or teapots and on a separate set of scaled mugs, bottles or teapots by in-distribution values described in Appendix~\ref{app:domain_details}.
We test extrapolation to $50$ mugs, bottles or teapots with out-of-sample orientations and scales sampled from the  distribution described in Appendix~\ref{app:domain_details}.
We adapt the Tetris object classification tensor field network to learn grasp point prediction by training the mean of the last equivariant layer to predict a grasping point and applying the mean squared error (MSE) loss. Following the implementation in \cite{thomas2018tensor}, we provide six manually labeled key points of the point cloud that uniquely identify the grasp point: two points on the handle and mug intersection and four points on the mug bottom.
We compare to the linear, neural network and transduction baselines, as well as bilinear transduction.
We fix these architectures to one layer and $32$ units.
Under various rotations, the tensor field network is able to extrapolate. For scaling however, the network error increases significantly, which is in accordance with the fact that scale transformation guarantees do not exist in SE(3) architectures.
Bilinear transduction can extrapolate in both cases, is comparable to tensor field networks on \oos{} orientations and exceeds its performance on \oos{} scales.

\paragraph{Weighted grasp point prediction}
The grasp point prediction results for three objects described in Table~\ref{tab:scaled-problems} are trained with priviledged knowledge of the object types at training time. I.e., the predictor at training time was provided only with pairs of same object categories but with different positions, orientations and scale. At test time an anchor was selected based on in-distribution similarities without knowledge of the test sample's object type.
While this provides extra information, it is still not building in information about the distribution shift we would like to extrapolate to, i.e. the types of transformations the objects undergo: translation, rotation and scale. 
We apply the \textbf{weighted} transduction algorithm (Algorithm~\ref{alg:weighted}) to this domain and show that with more relaxed assumptions on the required object labels at training time we can learn weights and a predictor that perform as well as with priviledged training (Table~\ref{tab:weighted-mugs}). 

For grasp point prediction, during training, we train a weighting function $\omega_{\theta}(x_j-x_i,x_i)$ to predict whether $x_i$ should be transduced to $x_j$ based on object-level labels of the data (transduce all mugs to other mugs and bottles to other bottles). Then we learn a predictor $h_{\theta}$ weighted by fixed $\omega_{\theta}$ as described in Algorithm~\ref{alg:weighted}.
At test time, given point $x_j$ we select the training point $x_i$ with highest value $\omega_{\theta}(x_j-x_i,x_i)$ and predict $y_j$ via $h_{\theta}(x_j-x_i,x_i)$, without requiring privileged information about the object category at test time.

\begin{figure}[!h]
    \centering
    \includegraphics[width=0.23\linewidth]{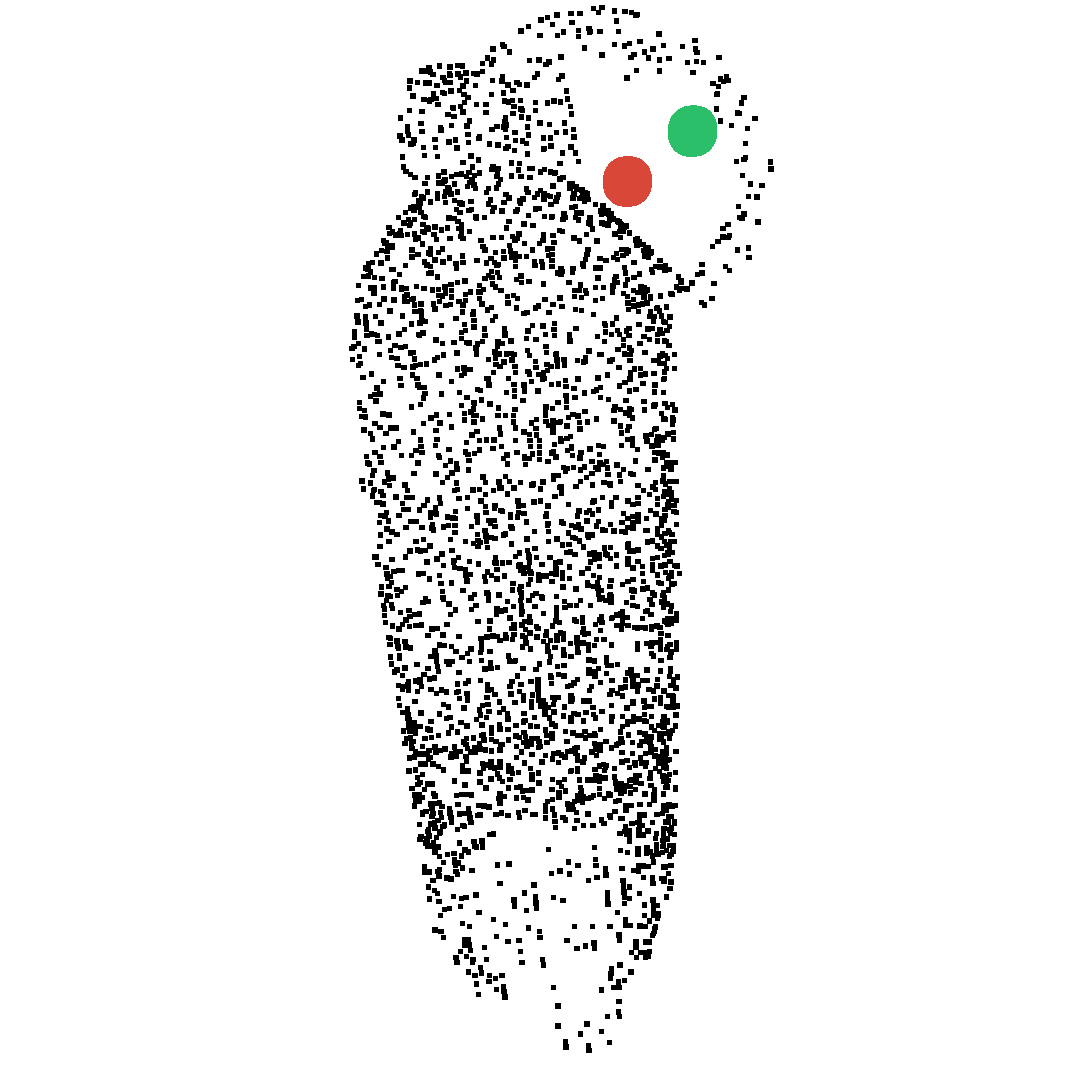}
    \includegraphics[width=0.23\linewidth]{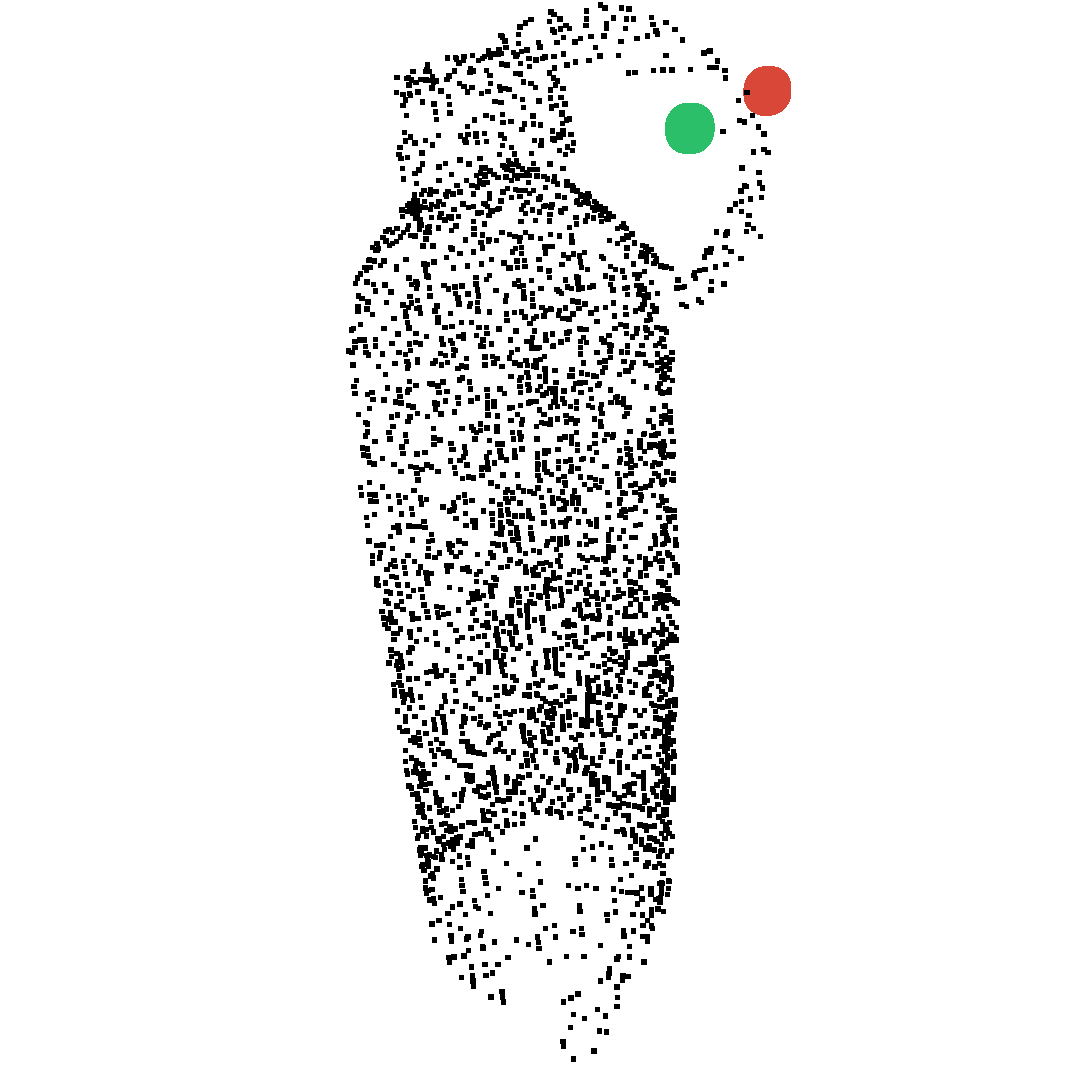}
    \includegraphics[width=0.23\linewidth]{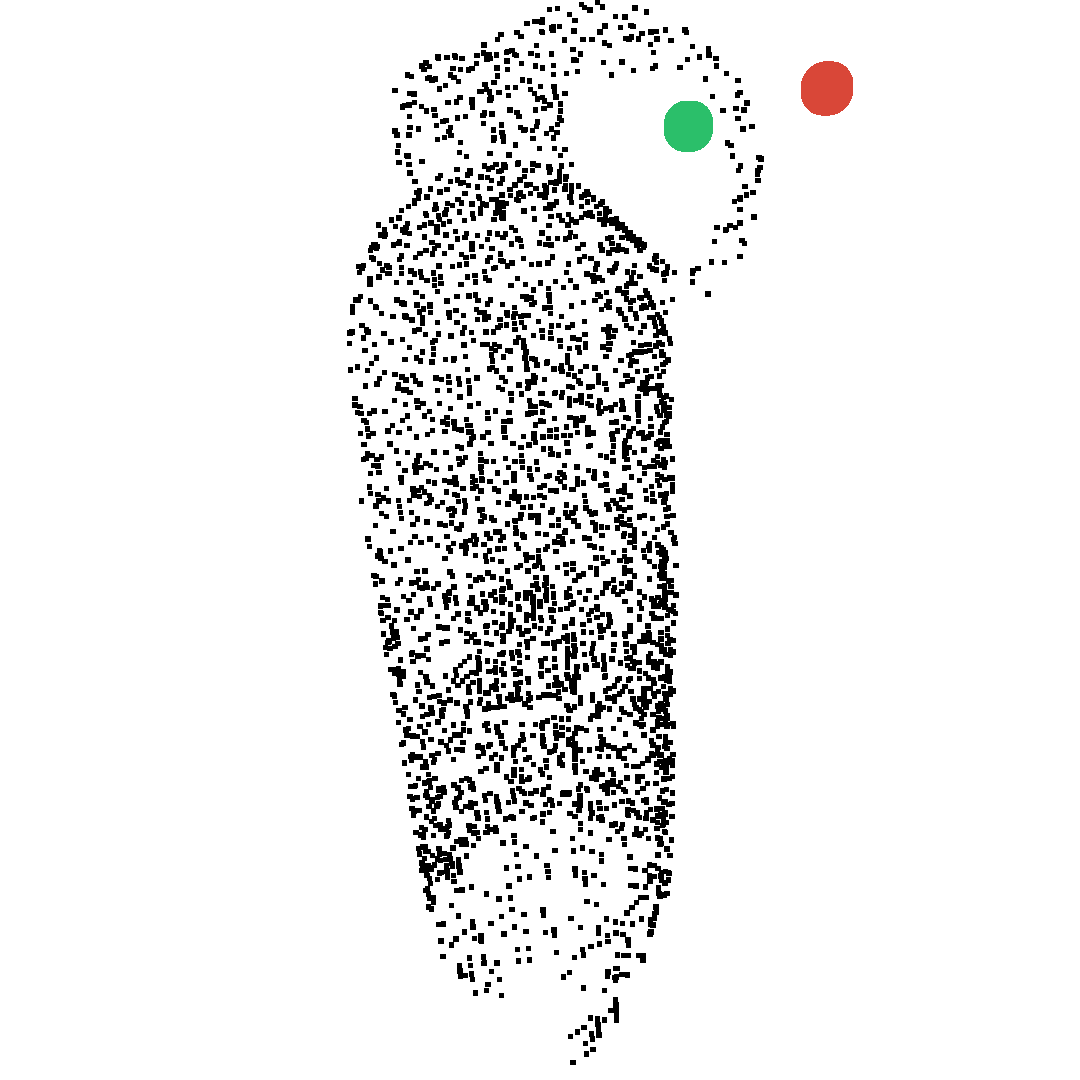}
    \includegraphics[width=0.23\linewidth]{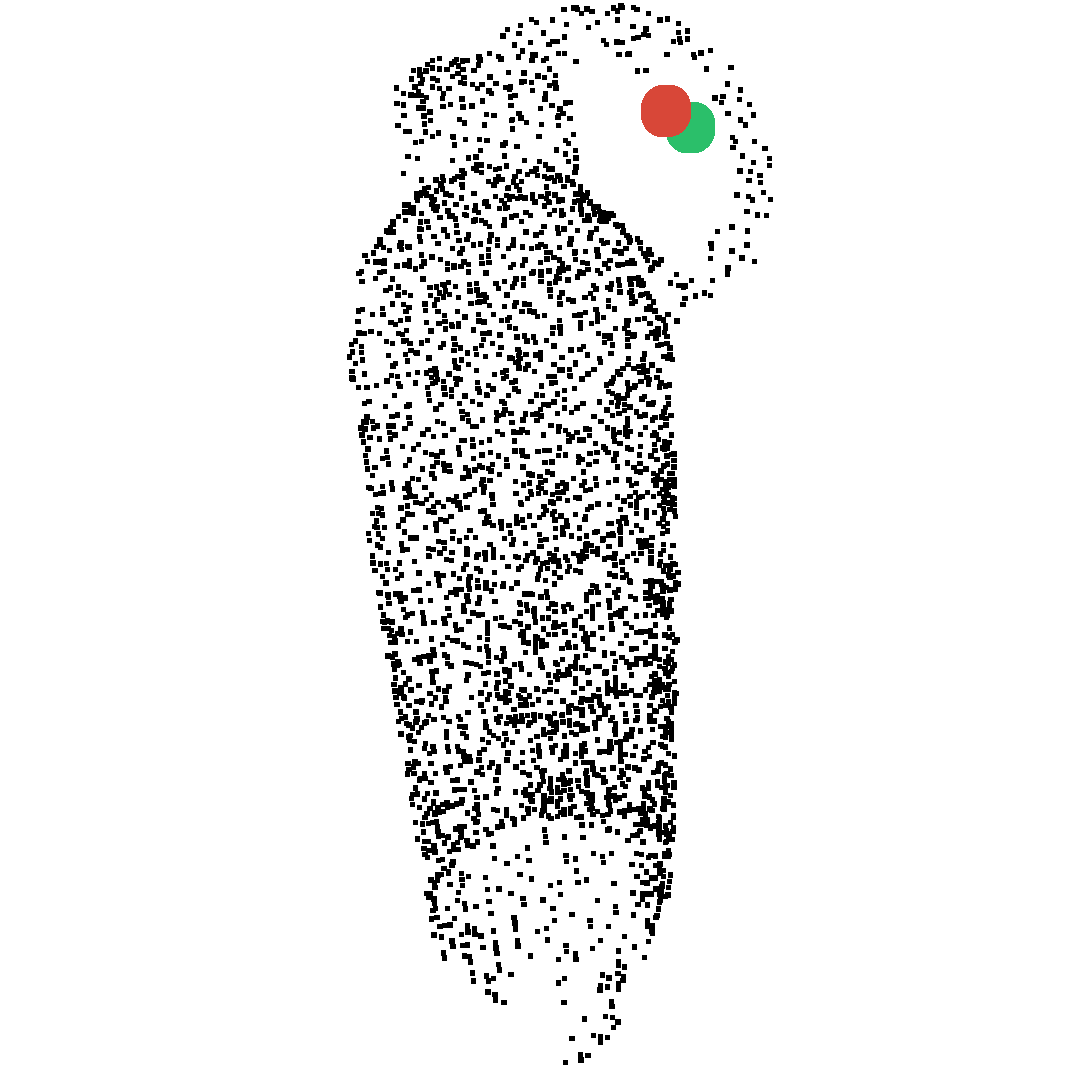}\\
    \includegraphics[width=0.23\linewidth]{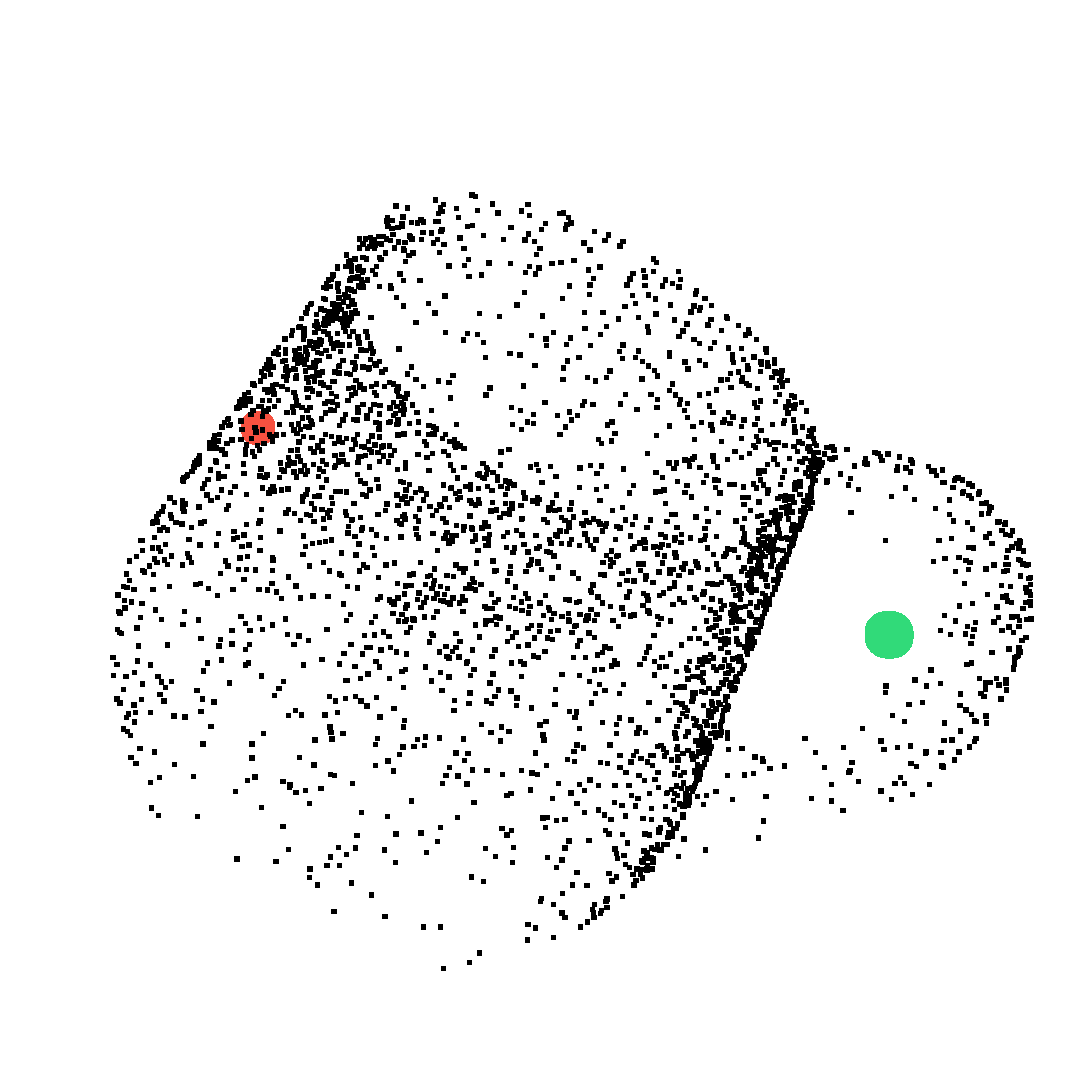}
    \includegraphics[width=0.23\linewidth]{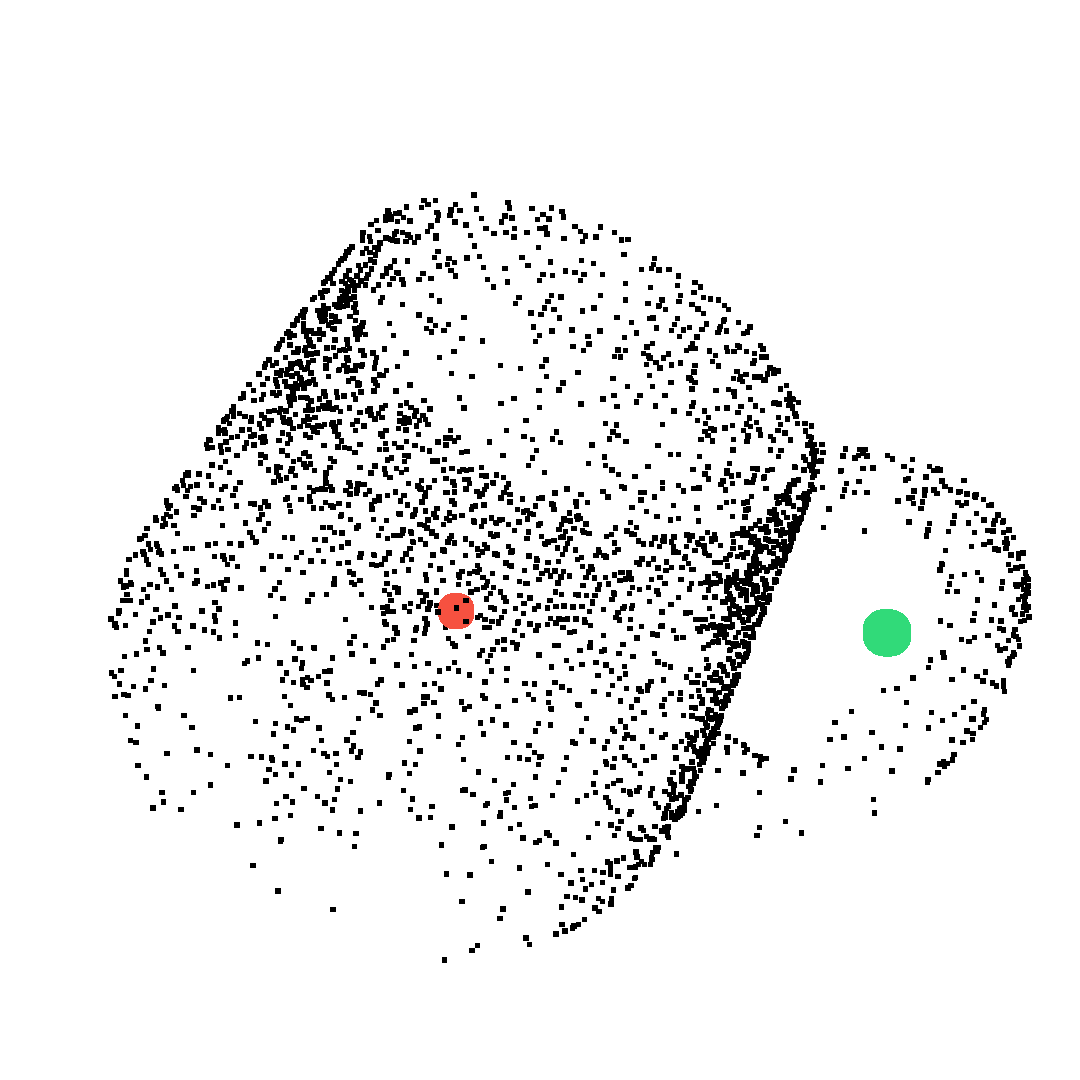}
    \includegraphics[width=0.23\linewidth]{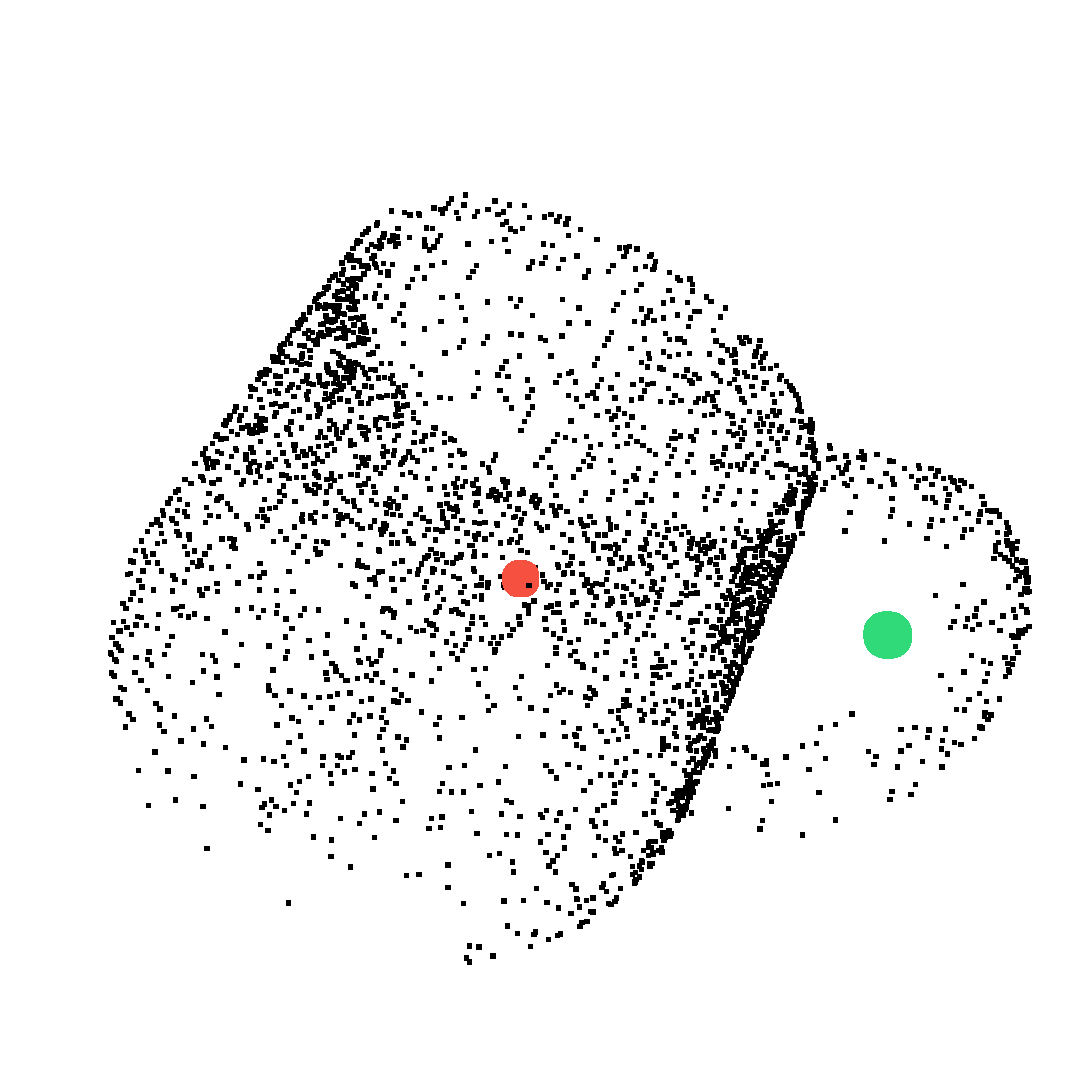}
    \includegraphics[width=0.23\linewidth]{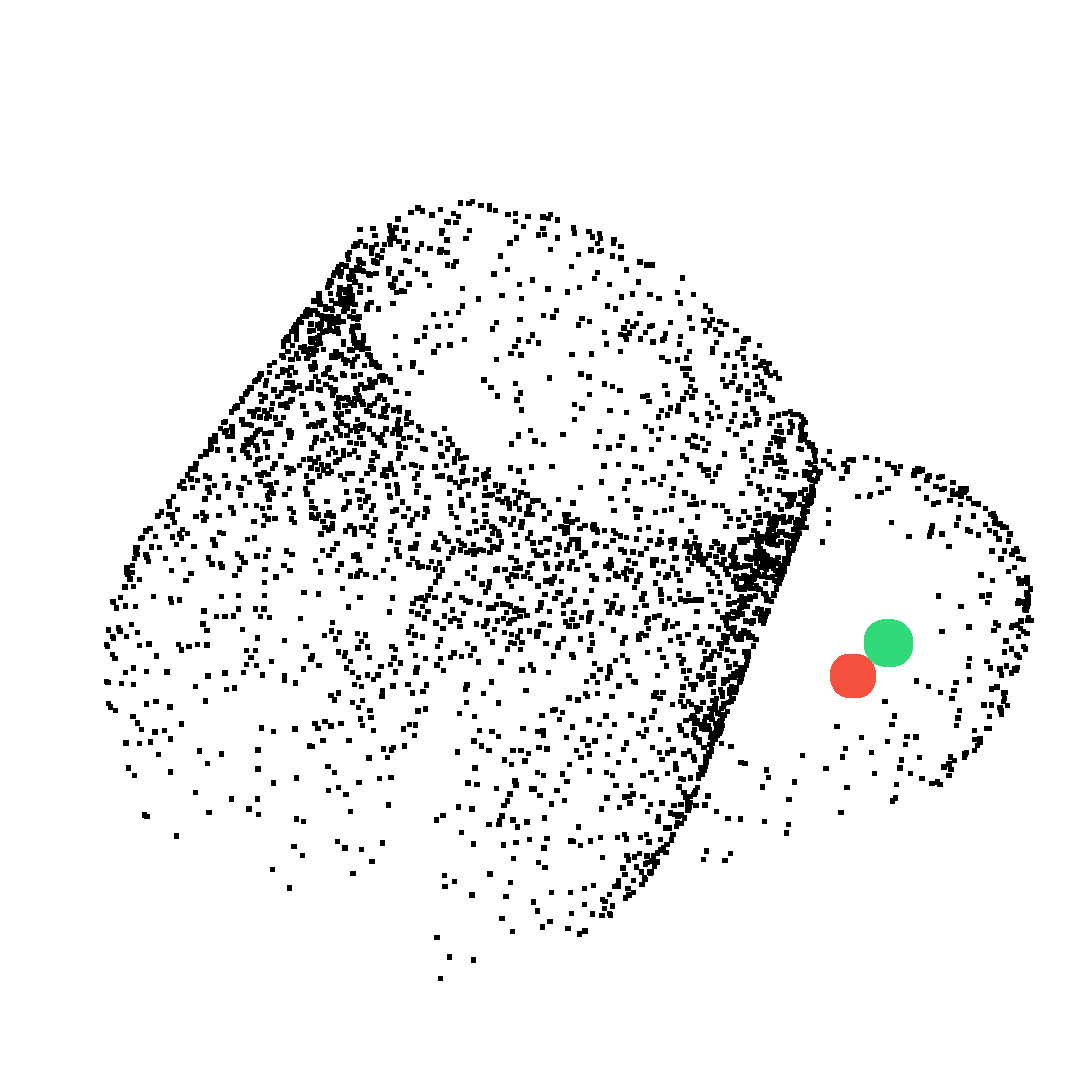}\\
     \includegraphics[width=0.23\linewidth]{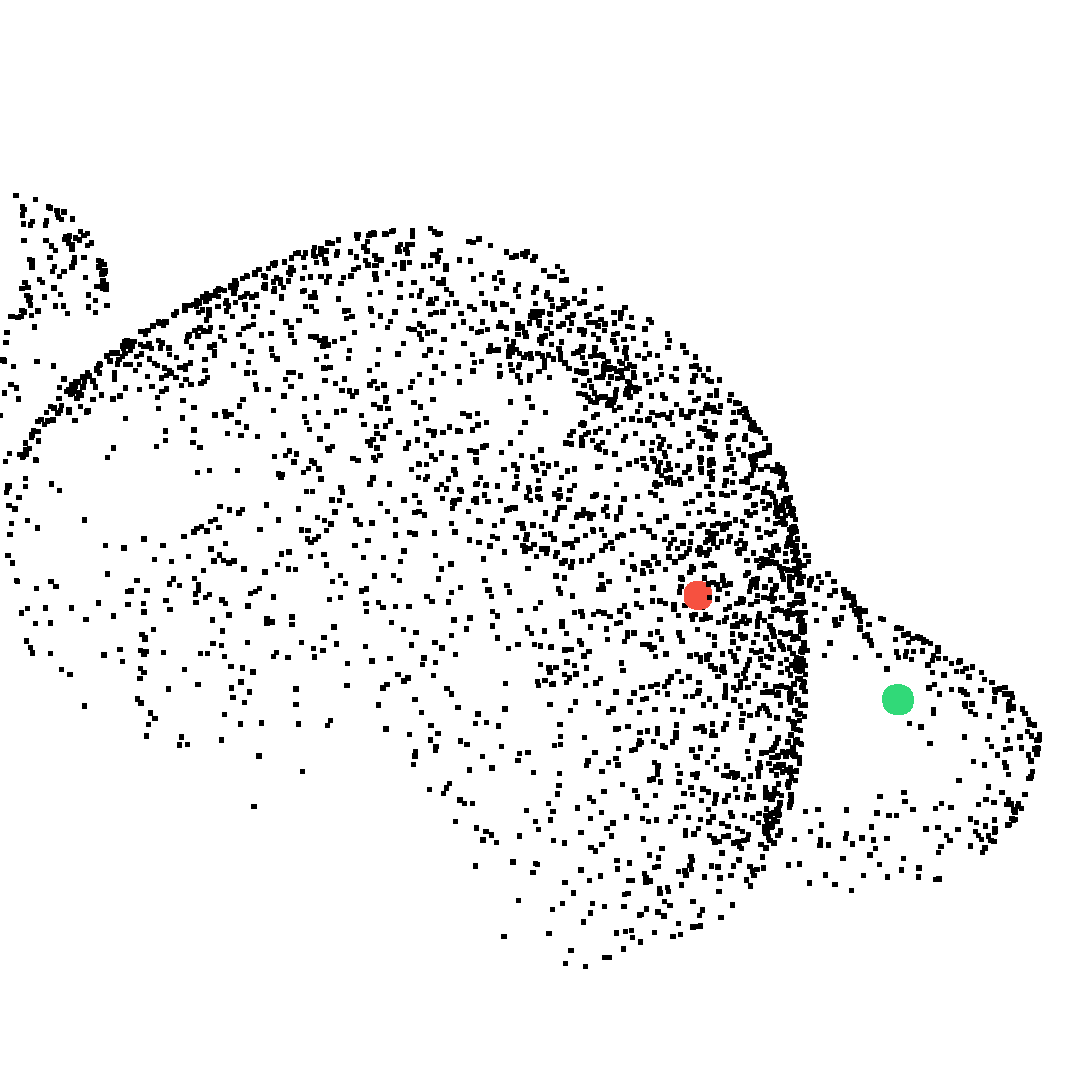}
    \includegraphics[width=0.23\linewidth]{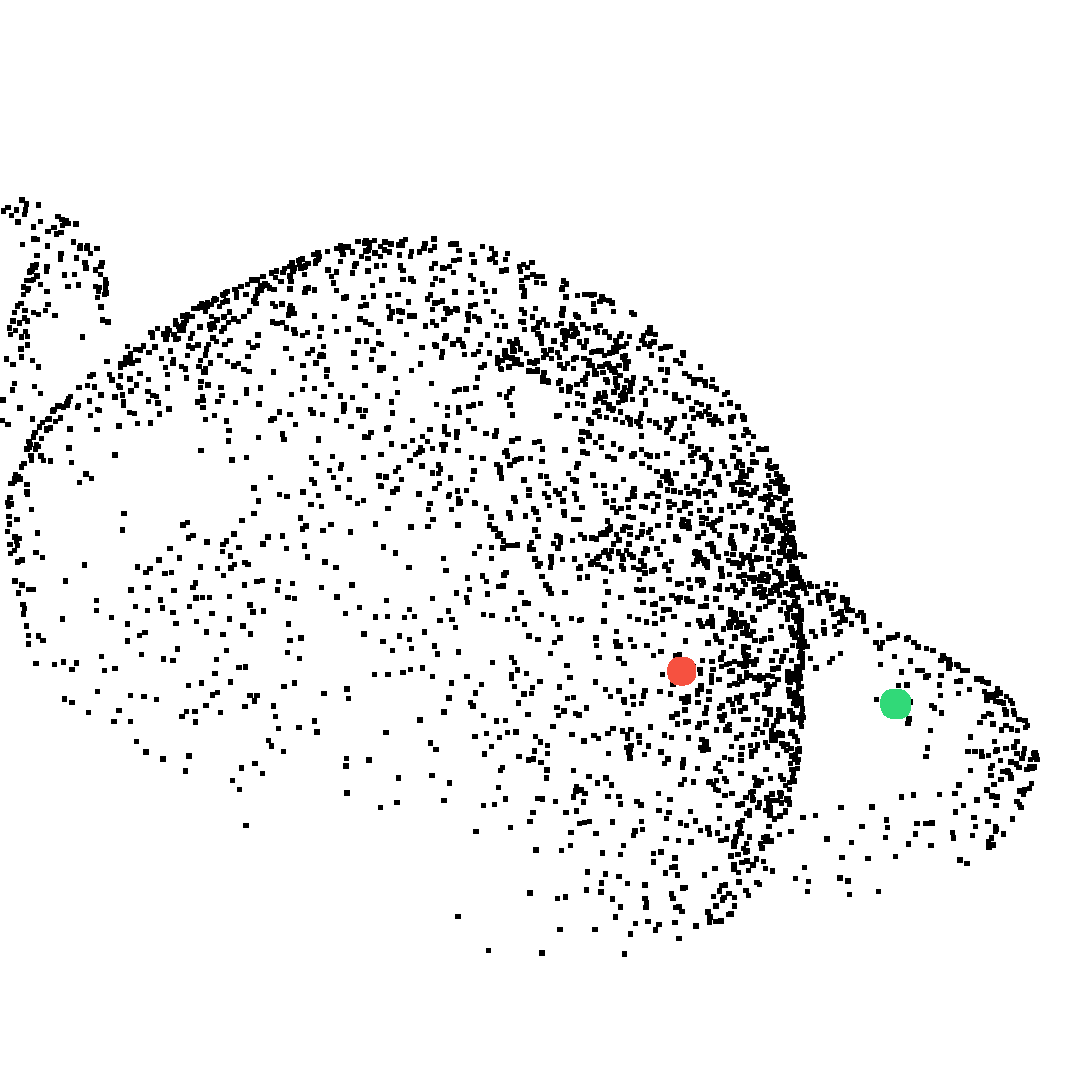}
    \includegraphics[width=0.23\linewidth]{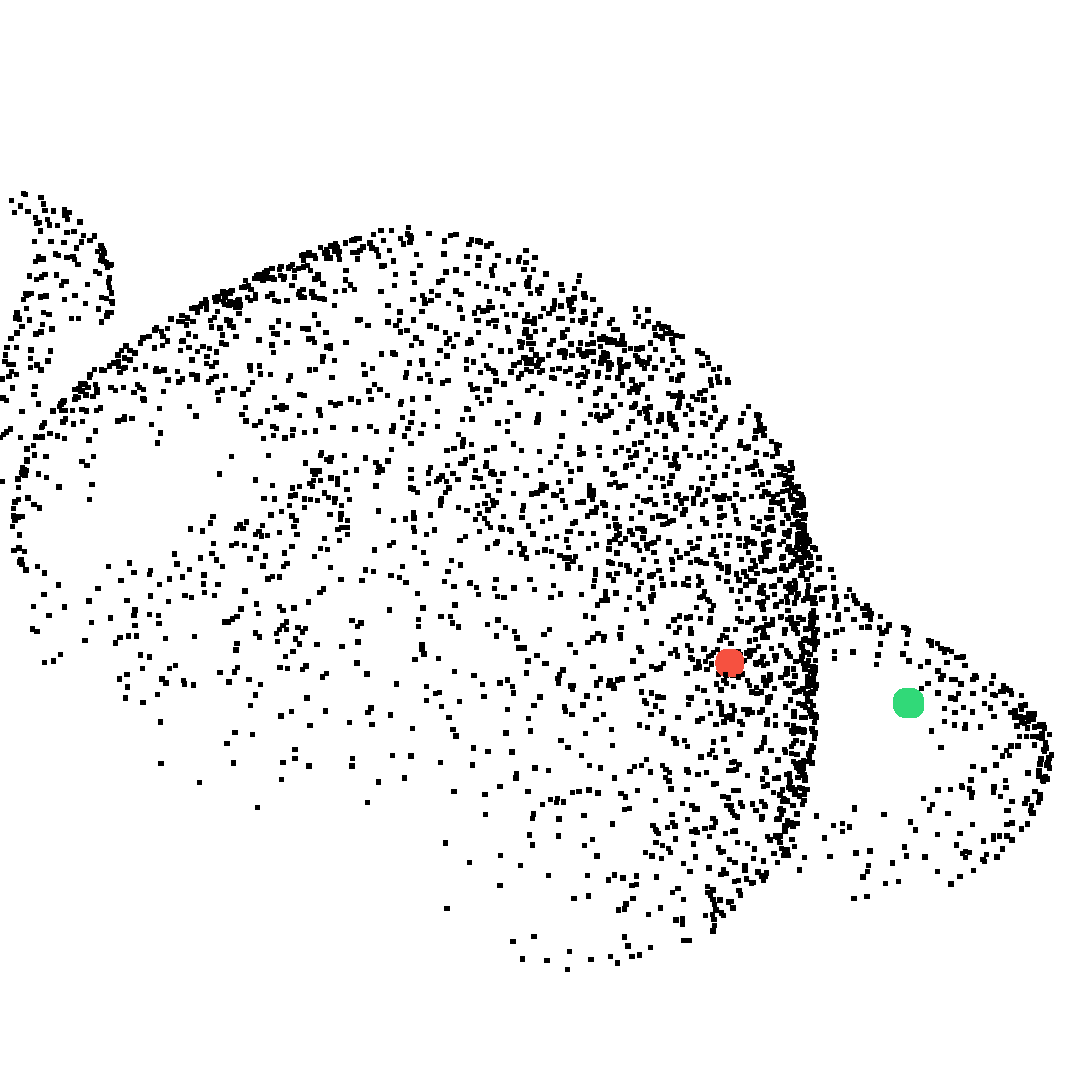}
    \includegraphics[width=0.23\linewidth]{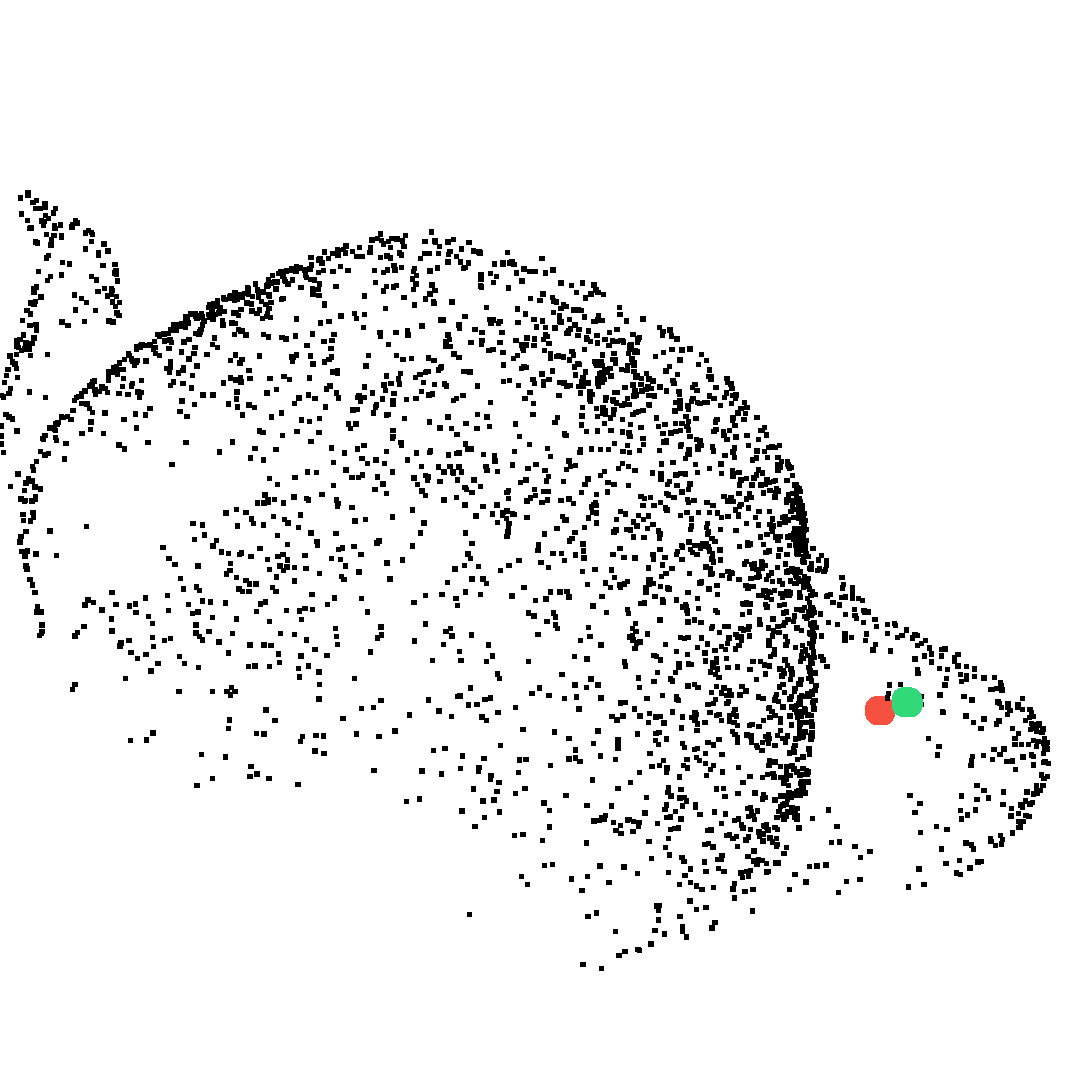} 
    
    \caption{\oos{} grasping point prediction for bottles, mugs and teapots. Each \textbf{row} represents one test datapoint. Columns (\textbf{Left to Right}): linear, neural network, transductive and bilinear model predictions (red) and ground truth (green) grasp points. 
    We show that baslines may predict infeasible grasping points (e.g. in the middle of an object) whereas bilinear transduction makes near accurate predictions. 
    }
    \label{fig:mug results}
\end{figure}

\begin{table}[t!]
  \caption{\footnotesize{
  Comparing bilinear transduction on grasp point prediction with tensor field networks \cite{thomas2018tensor}, an \textbf{SE(3) equivaraint} network for point clouds.
  We report mean and standard deviation for for $50$ \oos{} samples.
  }
  }
  \label{tab:se3}
 \centering
 \scalebox{0.8}{
  \begin{tabular}{llllll}
    Task
     & Tensor Field & Linear & Neural Net & Transduction & Ours \\
    \midrule
    Mug rotation & $0.009 \pm 0.003$ & $0.005\pm 0.001$  & $0.066\pm 0.035$ & $0.005\pm 0.001$ & $0.01\pm 0.003$ \\   
    \midrule
    Mug scale & $0.027 \pm 0.012$ & $0.006\pm 0.001$ & $0.004\pm 0.001$ & $0.004\pm 0.001$ & $0.004\pm 0.001$ \\
    %
    \midrule\midrule
    Bottle rotation & $0.008 \pm 0.004$ & $0.004\pm 0.004$  & $0.085\pm 0.046$ & $0.011\pm 0.026$ & $0.004\pm 0.002$ \\   
    \midrule
    Bottle scale & $0.019 \pm 0.008$ & $0.003\pm 0.001$ & $0.003\pm 0.001$ & $0.01\pm 0.006$ & $0.003\pm 0.001$ \\
    \midrule\midrule
    Teapot rotation & $0.011 \pm 0.006$ & $0.003\pm 0.001$ & $0.055\pm 0.03$ & $0.017\pm 0.012$ & $0.006\pm 0.002$ \\   
    \midrule
    Teapot scale & $0.093\pm 0.039$ & $0.003\pm 0.001$ & $0.003\pm 0.001$ & $0.003\pm 0.001$ & $0.003\pm 0.001$
  \end{tabular}
  }
\end{table}

\begin{table}[!h]
  \caption{\footnotesize{Grasping point prediction on three objects with various rotations, translations and scales with ground truth weighted training, standard training and \textbf{learned weighted} training for bilinear transduction with a fixed architecture.
  By training a weighting function with fewer labels, learned weighted bilinear transduction achieves similar performance to ground truth weighted training.}}
  \label{tab:weighted-mugs}
  \centering
  \scalebox{0.8}{
  \begin{tabular}{llll}
  &Ground Truth & Bilinear Transduction & Learned Weighted\\
  &Weighted&&\\
  \midrule
  (Train)&$0.007\pm0.004$&$0.01\pm0.005$&$0.006\pm0.003$\\
  (\oos{})&$0.027\pm0.014$&$0.068\pm0.041$&$0.024\pm0.014$
  \end{tabular}
  }
\end{table}

\subsection{Additional Results on Imitation Learning}
\label{app:imitation}

\paragraph{Qualitative decision making results} For qualitative results of all models on imitation learning domains, please view videos in the supplementary zip file.  

\subsection{Choice of Architecture}

\paragraph{Periodic activation functions} We provide further results on the grasp point prediction and imitation learning domains in Table~\ref{tab:fourier}. This experiment is similar to the one reported in Table~\ref{tab:scaled-problems}, but with fourier pre-processing as the initial layer in each model.
For the various domains, adding this pre-processing step is beneficial for some architecture-domain combinations, but does not surpass bilinear transduction on all domains for any method. Moreover, in most domains bilinear transduction achieves the best results, or close results to the best model. For implementation details see Section~\ref{app:training_details}.

\paragraph{Robustness across seeds and architectures:} We show that our results are stable across the choice of architecture and seed. 
In Fig~\ref{fig:heatmap} we show the performance of our method and baselines over Meta-World tasks and architectures.
In Table~\ref{tab:3seeds} we show the average performance of our method and baselines over the complex domains for fixed architecture parameters over three seeds.

\begin{table}[!h]
  \caption{\footnotesize{
  Mean and standard deviation over prediction (regression) or final state (sequential decision making) error for \oos{} samples and over a hyperparameter search with \textbf{fourier pre-processing}. While fourier activations are useful for some combinations of models and domains, it is not beneficial across all.
  }
  }
  \label{tab:fourier}
  \centering
 \scalebox{0.8}{
  \begin{tabular}{llllllll}
    Task & Expert 
     & Linear & Neural Net & DeepSets & Transduction & Ours \\
    \midrule
    Mugs & & $0.119\pm0.05$ & $0.219\pm0.114$ && $0.3\pm0.148$ & $\mathbf{0.086\pm0.081}$ \\  
    \midrule
    Bottle && $0.099\pm0.041$ &
    $0.139\pm0.075$ && $0.158\pm0.096$ & $\mathbf{0.019\pm0.011}$\\   
    \midrule
    Teapot && $0.112\pm0.045$ & $0.216\pm0.107$ && $0.301\pm0.14$ & $\mathbf{0.041\pm0.026}$ \\    
    \midrule
    All && $0.142\pm0.064$ & $0.252\pm0.119$ && $0.33\pm0.157$ & $\mathbf{0.047\pm0.033}$ \\
    \midrule
    \midrule    
    %
    %
    Reach & $0.006\pm0.008$ & $\mathbf{0.007\pm0.006}$ & $0.075\pm0.082$ & $0.087\pm0.092$ & $0.117\pm0.176$ & $0.008\pm0.008$ \\    
    %
    %
    %
    \midrule
    %
    Push & $0.012\pm0.001$ & $0.182\pm0.13$ & $0.113\pm0.124$ & $0.121\pm0.091$ & $0.17\pm0.12$ & $\mathbf{0.014\pm0.006}$ \\    
    %
    %
    %
    \midrule
    %
    Slider & $0.105\pm0.066$ & $0.38\pm0.032$ & $0.151\pm0.081$ & $0.241\pm0.077$ & $\mathbf{0.134\pm0.114}$ & $0.256\pm0.214$ \\
    %
    \midrule
    Adroit & $0.035\pm0.015$ & $0.551\pm0.508$ & $0.398\pm0.435$ & $0.373\pm0.195$ & $\mathbf{0.341\pm0.218}$ & $0.365\pm0.31$ 
  \end{tabular}
  }
\end{table}

\begin{figure}[!h]
    \centering
    \includegraphics[width=0.55\linewidth]{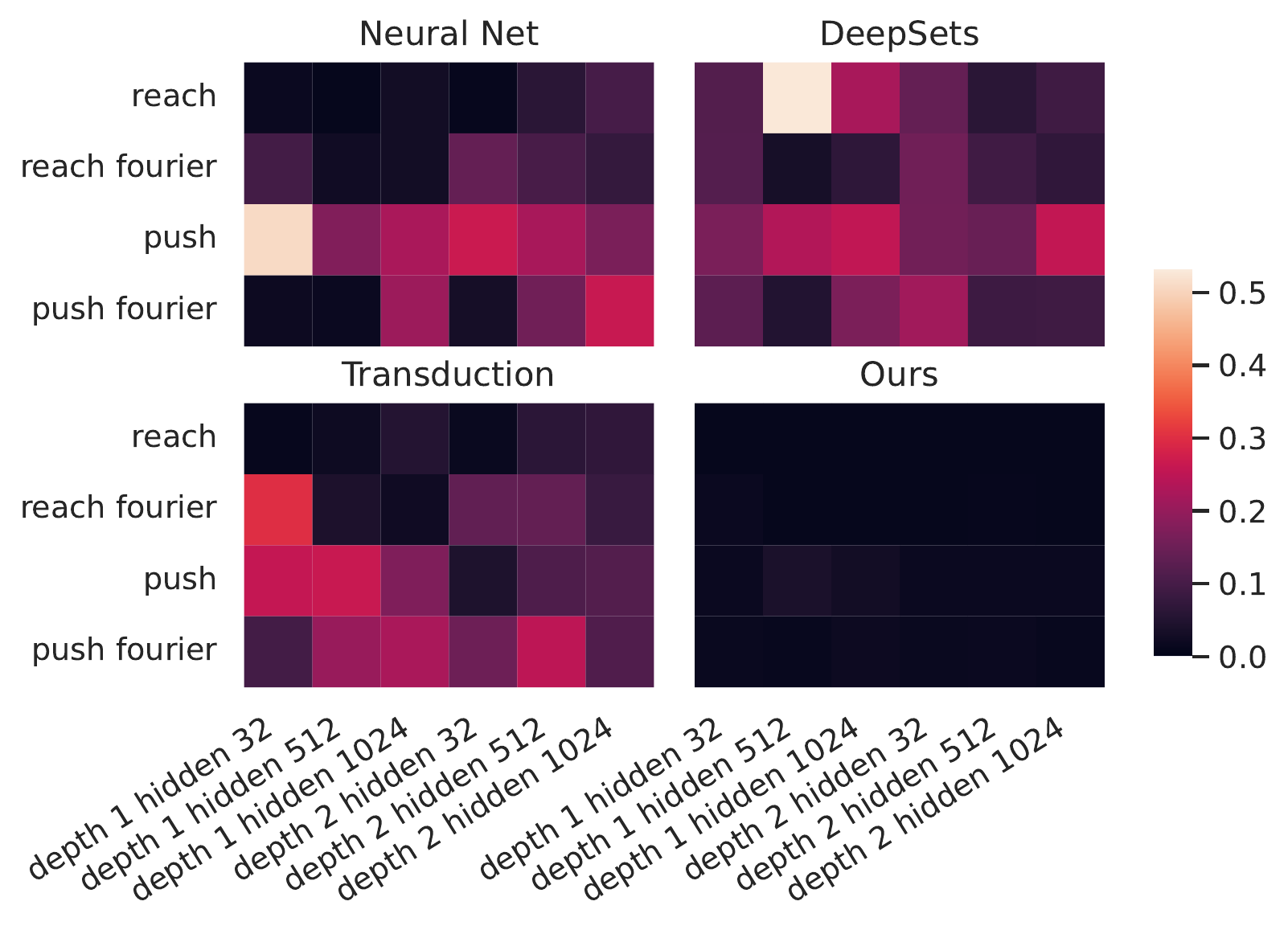}
    \caption{\footnotesize{Heatmap complementing Table~\ref{tab:scaled-problems} showing mean \oos{} errors on Meta-World. While some architectures can be suitable for some domains, our method is more robust to hyperparameter search and almost always achieves low error.
    }}
    \label{fig:heatmap}
\end{figure}

\begin{table}[!h]
  \caption{We report mean and standard deviation error over evaluation samples and three \textbf{seeds}. For each domain, the first row is evaluated on in-distribution samples, the second row is on \oos{} samples.
  We show (1) bilinear transduction performs well in-distribution as well as \oos{} (2) our algorithm is stable across several seeds.}
  \label{tab:3seeds}
  \centering
 \scalebox{0.8}{
  \begin{tabular}{llllllll}
    Task & Expert 
    & Linear & Neural Net & DeepSets & Transduction & Ours \\
    \midrule
    Grasping (Train) & & $0.06\pm0.059$&$0.006\pm0.004$& &$0.006\pm0.003$&$0.007\pm0.003$ \\ 
    Grasping \oos{} && $0.141\pm0.114$&$0.124\pm0.074$& &$0.099\pm0.07$&$0.024\pm0.014$\\
    \midrule
    \midrule
    Reach (Train) & $0.006\pm0.008$ & $0.004\pm0.005$&$0.005\pm0.005$&$0.005\pm0.005$&$0.004\pm0.005$&$0.005\pm0.005$ \\ 
    Reach (\oos{}) & & $0.007\pm0.006$&$0.013\pm0.015$&$0.088\pm0.075$&$0.007\pm0.007$&$0.007\pm0.006$ \\    
    \midrule
    Push (Train) & $0.012\pm0.001$ & $0.22\pm0.103$&$0.046\pm0.049$&$0.01\pm0.001$&$0.022\pm0.011$&$0.016\pm0.007$ \\ 
    Push \oos{} & & $0.26\pm0.063$&$0.355\pm0.138$&$0.188\pm0.124$&$0.337\pm0.151$&$0.022\pm0.016$ \\
    \midrule
    Slider (Train) & $0.105\pm0.066$ & $0.425\pm0.077$&$0.344\pm0.197$&$0.079\pm0.057$&$0.32\pm0.072$&$0.17\pm0.092$ \\ 
    Slider \oos{} & & $0.512\pm0.06$&$0.438\pm0.191$&$0.385\pm0.288$&$0.55\pm0.209$&$0.163\pm0.126$ \\ 
    \midrule
    Adroit (Train) & $0.035\pm0.015$ & $0.66\pm0.596$&$0.139\pm0.231$&$0.103\pm0.18$&$0.129\pm0.215$&$0.047\pm0.024$ \\  
    Adroit \oos{} && $0.337\pm0.075$&$0.703\pm0.561$&$0.422\pm0.241$&$0.773\pm0.796$&$0.198\pm0.255$ 
  \end{tabular}
  }
\end{table}

%% file: appendix/implementation_details.tex
\newpage
\section{Implementation Details}
\label{app:implementation}

\subsection{Train/Test distributions and Analytic Function Details}
\label{app:domain_details}

\paragraph{Analytic functions:}
\label{app:analytic-functions}
We used the following analytic functions for evaluation in Section~\ref{sec:analysis}. 

    Figure~\ref{subfig:periodic}. Periodic mixture of functions with different periods: We consider the period growing mixture of functions, and do a similar function but with two functions with different periods repeating.
    Let us define $x_v = x \mod 9$, then the 
        \begin{equation}
            f(x) = \begin{cases}
                  x_m*\sin(10*x_v) & \text{if } x_v < 1 \\
                  x_m*(\sin(10) + (x_v - 1)) & \text{if } 1 < x_v < 2 \\
                  x_m*(\sin(10) + 1 + (x_v - 2)^2) & \text{if } 2 < x_v < 3 \\
                  x_m*\sin(10*\frac{(x_v - 3)}{2}) & \text{if } 3 < x_v < 5 \\
                  x_m*(\sin(10) + (\frac{x_v - 3}{2} - 1)) & \text{if } 5 < x_v < 7 \\
                  x_m*(\sin(10) + 1 + (\frac{x_v - 3}{2} - 2)^2) & \text{if } 7 < x_v < 9 \\
                  \end{cases}
        \end{equation}
      
    We can see that this function mixes 2 different periods, to show that our framework can deal with varying functions as well. 
    
    Figure~\ref{subfig:sawtooth}. Sawtooth function: we use a classic sawtooth function \begin{equation}
        f(x) = \begin{cases}
          (x \mod \mathrm{Period})*\frac{\mathrm{Amplitude}}{\mathrm{Period}} & \text{if } (\lfloor \frac{x}{\mathrm{Period}} \rfloor) \mod 2 == 0\\
          \mathrm{Amplitude} - (x \mod \mathrm{Period})*\frac{\mathrm{Amplitude}}{\mathrm{Period}} & \text{if } (\lfloor \frac{x}{\mathrm{Period}} \rfloor) \mod 2 == 1
          \end{cases}
    \end{equation}
    
    Figure~\ref{subfig:polynomial}. Randomly chosen polynomial: We sampled a random degree $8$ polynomial: $f(x) = (x - 0.1)(x + 0.4)(x + 0.7)(x - 0.5)(x - 1.5)(x + 1.75)(x + 1)(x - 1.2)$
    
    Figures~\ref{fig:train_dependence} $\&$ ~\ref{fig:beyond_equivariance}. Periodic growing mixture of functions: Let us define $x_v = x \mod 3$ and $x_m = \lfloor \frac{x}{3} \rfloor$ 
    \begin{equation}
        f(x) = \begin{cases}
              x_m*\sin(10*x_v) & \text{if } x_v < 1 \\
              x_m*(\sin(10) + (x_v - 1)) & \text{if } 1 < x_v < 2 \\
              x_m*(\sin(10) + 1 + (x_v - 2)^2) & \text{if } 2 < x_v < 3 \\
              \end{cases}
    \end{equation}
    Figure~\ref{fig:equivariant_function}. Shift equivariant mixture functions: Let us define $x_v = x \mod 3$ and $x_m = \lfloor \frac{x}{3} \rfloor$ 
    \begin{equation}
        f(x) = \begin{cases}
              x_m*3 + \sin(10*x_v) & \text{if } x_v < 1 \\
              x_m*3 + (\sin(10) + (x_v - 1)) & \text{if } 1 < x_v < 2 \\
              x_m*3 + (\sin(10) + 1 + (x_v - 2)^2) & \text{if } 2 < x_v < 3 \\
              \end{cases}
    \end{equation}
      
     This represents a mixture of different functions - a sinusoid function, a linear function and a quadratic function that repeat and are equivariant on shifting. These are meant to show that bilinear transduction can capture equivariance.  
    
    

\paragraph{Grasp point prediction:} A training dataset is generated from rotating, translating and scaling a bottle, mug and teapot from the ShapeNet dataset \cite{chang2015shapenet} with various parameters to predict a grasping point in $\mathbb{R}^3$ from their point clouds. A test dataset is constructed by applying the same transformations to the objects with \oos{} parameters.
The point cloud in $\mathbb{R}^{N\times3}$ is represented by a 12-tuple containing the point cloud's mean point and its three PCA components in $\mathbb{R}^3$. In Fig~\ref{fig:mug pca} we plot the reduced point cloud state space representing point cloud mean and PCA components for in-distribution and \oos{} object orientations, positions and scales.
We add Gaussian noise to the point cloud sampled from $\mathcal{N}(0,0.005)$ and to the grasp point label sampled from $\mathcal{N}(0,0.002)$.
During training, objects are rotated around the $z$ axis $\theta\in[0,1.2\pi]$, translated by $(x,y)\in [0,0.5]\times [0,0.5]$ and scaled by a factor of $\alpha\in[0.7,1.3]$. 
At test time these transformation parameters are sampled from the following ranges: $\theta\in[1.2\pi,2\pi]$, $(x,y)\in[0.5,0.7]\times[0.5,0.7]$ and $\alpha\in[1.3,1.6]$.

\paragraph{Meta-World:} We evaluate on the reach-v2 and push-v2 tasks. For \textbf{reach-v2}, we reduce the observation space to a 6-tuple of the end effector 3D position and target 3D position. The action space is a 4-tuple of the end effector 3D position and the gripper's degree of openness.  
During training, the 3D target position for reach-v2 is sampled from range $[0,0.4]\times[0.7,0.9]\times[0.05,0.3]$ and at test time from range $[-0.4,0]\times[0.7,0.9]\times[0.05,0.3]$. The gripper initial position is fixed to $[0,0.6,0.2]$. 
For \textbf{push-v2} the observation space is a 10-tuple of the end effector 3D position, gripper openness, object 3D position and target 3D position. The action space remains the same as in reach-v2.
The training 3D target position is sampled from $[0,0.3]\times[0.5,0.7]\times\{0.01\}$ and at test time from $[-0.3,0]\times[0.5,0.7]\times\{0.01\}$. The gripper initial position is fixed to $[0,0.4,0.08]$ and the object initial position to $[0,0.4,0.02]$.
For both environments, expert and evaluation trajectories are collected for $100$ steps.

\paragraph{Adroit hand:}
We evaluate on the object relocation task in the ADROIT hand manipulation benchmark.
The observation space is a 39-tuple of 30D hand joint positions, 3D object position, 3D palm position and 3D target position.
The action space is a 30-tuple of 3D position and 3D orientation of the wrist, and 24D for finger positions commanding a PID controller.
During training, the target location is sampled from $[0,0.3]\times[0,0.3]\times[0.15,0.35]$ and at test time from $[-0.3,0]\times[-0.3,0]\times[0.15,0.35]$. Expert and evaluation trajectories are collected for $200$ steps.

\paragraph{Slider:} 
We introduce an environment for controlling a brush to slide an object to a target.
The observation space is an 18-tuple composed of a 9-tuple representing the brush 2D position, 3D object position and 4D quaternion, an 8-tuple representing the 2D brush velocity, 3D object velocity and 3D rotational object velocity, and the object mass.
The action space is torque applied to the brush.
During training, the object mass is sampled from $[60,130]$ and at test time from $[5,15]$.
Expert and evaluation trajectories are collected for $200$ steps.
In Figure \ref{fig:slider act} we plot the actions as a function of time steps for demonstrations, baselines described in Section~\ref{sec:experiments} and Bilinear Transduction for a set of fixed architecture parameters.

\begin{figure}[!h]
    \centering
    \includegraphics[width=0.23\linewidth]{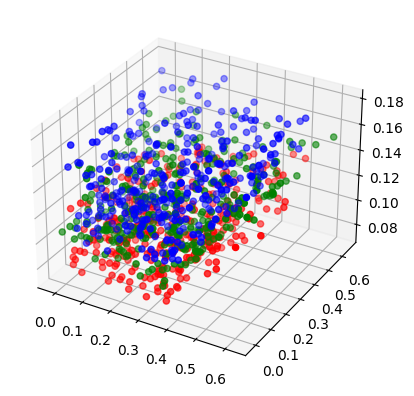}
    \includegraphics[width=0.23\linewidth]{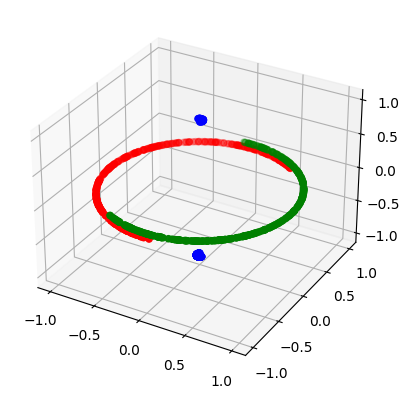}
    \includegraphics[width=0.23\linewidth]{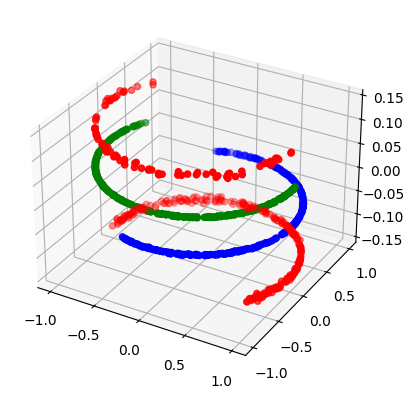}
    \includegraphics[width=0.23\linewidth]{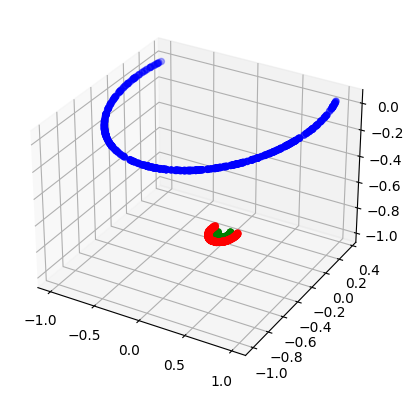}\\
    \includegraphics[width=0.23\linewidth]{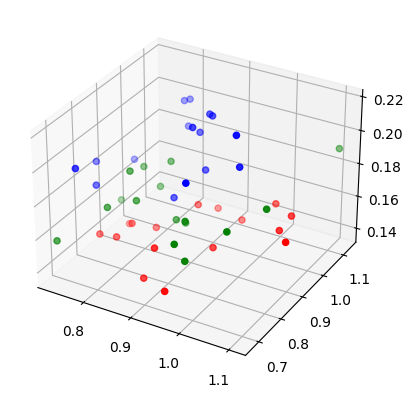}
    \includegraphics[width=0.23\linewidth]{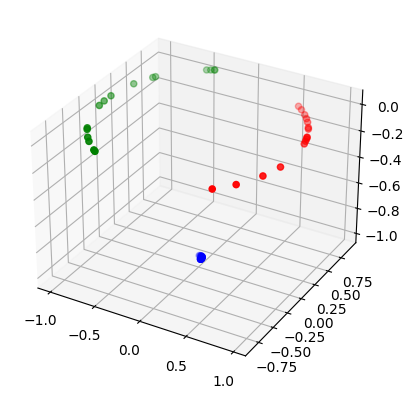}
    \includegraphics[width=0.23\linewidth]{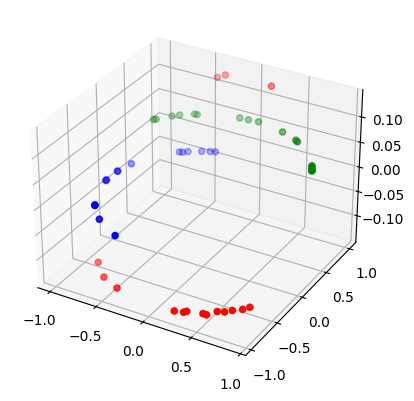}
    \includegraphics[width=0.23\linewidth]{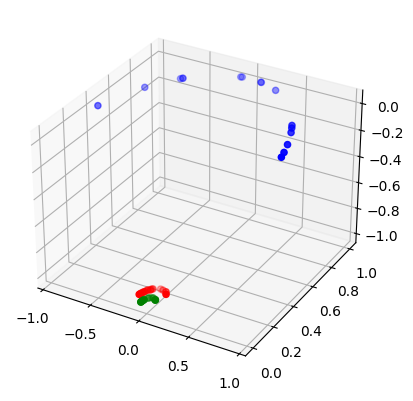}
    \caption{Reduced point cloud feature space in $\mathbb{R}^{12}$. From left to right: mean and three PCA components. Top: in distribution, bottom: \oos{}. Blue: bottles, red: mugs, green: teapots.}
    \label{fig:mug pca}
\end{figure}

\begin{figure}[!h]
    \centering
    \includegraphics[width=0.28\linewidth]{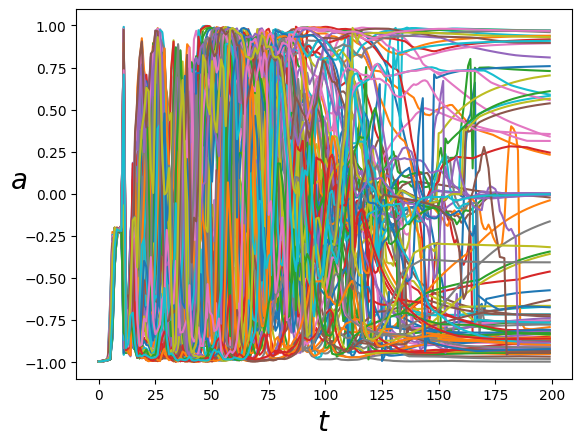}
    \includegraphics[width=0.28\linewidth]{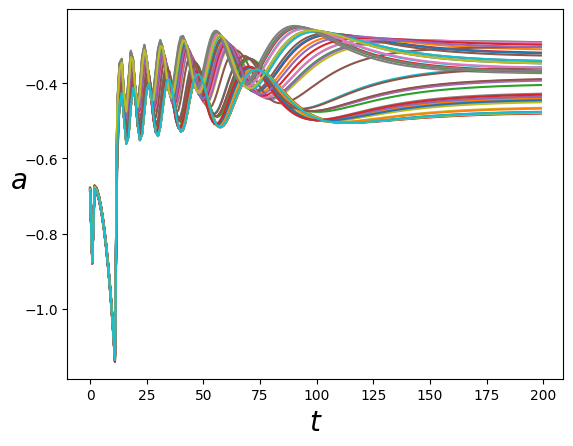}
    \includegraphics[width=0.28\linewidth]{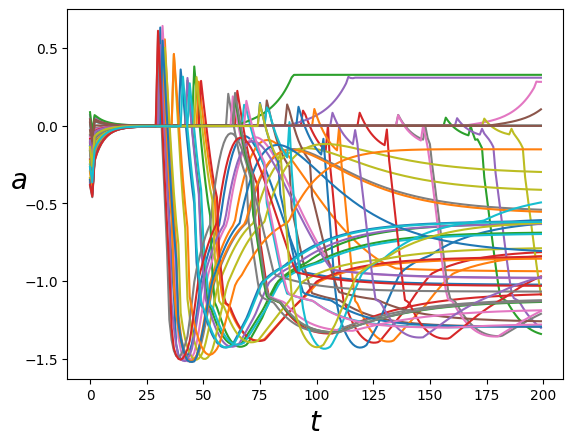}\\
    \includegraphics[width=0.28\linewidth]{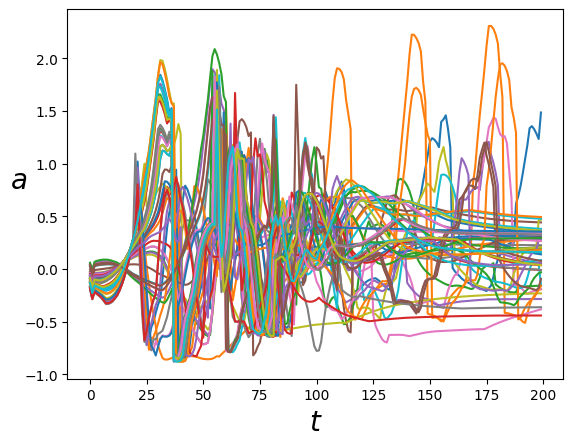}
    \includegraphics[width=0.28\linewidth]{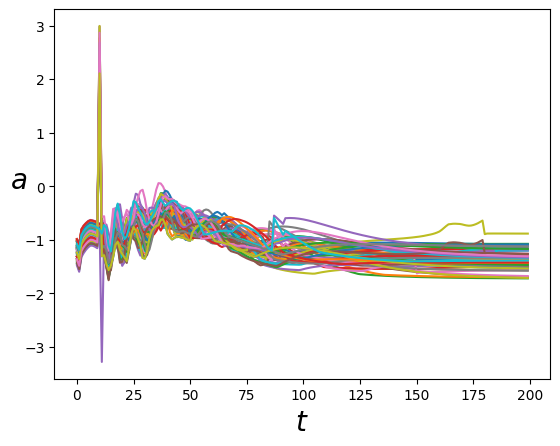}  
    \includegraphics[width=0.28\linewidth]{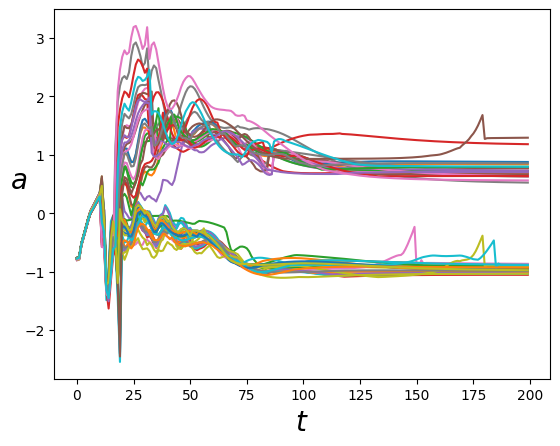}      
    \caption{Slider torque actions as a function of time steps. Each color represents a different trajectory and mass. From left to right, top to bottom: in-support demonstrations and Linear, Neural Net, DeepSets, Transduction and Bilinear Transduction policies for \oos{} masses.}
    \label{fig:slider act}
\end{figure}

\subsection{Training details}
\label{app:training_details}

\subsubsection{Details of Bilinear Transduction Implementation}
We describe the \textbf{bilinear transduction} algorithm implementation in detail. 
For grasp point prediction, during training we uniformly select an object $o\in\{\mathrm{bottle, mug, teapot}\}$, and two instances of those objects, $o_i, o_j$ that were rotated, translated and scaled with different parameters. We learn to transduce $o_i$ to $o_j$, i.e. learn $h_{\theta}(o_j-o_i, o_i)$ to predict $y_j$, the grasping point of $o_j$.
At test time, given test point $x_j$ with an unknown object category, we select a training point $x_i$ with an ``in-distribution'' difference $x_j-x_i$, as described in Algorithm~\ref{alg:unweighted}.
We select $\rho$ to be the closest difference $x_j-x_i$ to an in-distribution difference over all training points $x_i$.
The differences generated by the optimization process are sampled uniformly to generate a smaller subset for comparison.

For the goal-conditioned imitation learning setting, during training we uniformly sample trajectories $\tau_i=\{(s^i_t,a^i_t)\}_{t=1}^T$ for horizon $T$ specified in Appendix~\ref{app:domain_details} and $\tau_j=\{(s^j_t,a^j_t)\}_{t=1}^T$. We further uniformly sample a time step $t$ from $[0,T]$. We learn to transduce state $s^i_t\in\tau_i$ to $s^j_t\in\tau_j$. I.e., learn $h_{\theta}(s^j_t-s^i_t, s^i_t)$ to predict action $a^j_t$.
At test time, we select an anchor trajectory based on the goal (Meta-World, Adroit) or mass (Slider) $g_j$. This is done by selecting the training trajectory $\tau_i$ with goal or mass $g_i$ that generates an ``in-distribution'' difference $g_j-g_i$ with the test point, as described in Algorithm~\ref{alg:unweighted}. We select $\rho$ to be the $10$-th percentile of the in-distribution differences.
The differences generated by the optimization process are approximated by generating differences from the training data. 
We then transduce each \textit{state} in $\tau_i$ to predict actions for the test goal or mass. Given test state $s_t^j$, we predict $a_t^j$ with $h_{\theta}(s^j_t-s^i_t, s^i_t)$, execute action $a_t^j$ in the environment and observe state $s_{t+1}^j$. We complete this process for $T$ steps to obtain $\tau_j$. 



\subsubsection{Details of Training Distributions}
We train on $1000$ samples for all versions of the grasp point prediction domain: single objects, three objects and weighted transduction.  
In the sequential decision domains we train on $N$ demonstrations, sequences of state-action pairs,  $\{(x_i,y_i)\}^T_{i=1}$ where horizon $T$ for each domain is specified in ~\ref{app:domain_details}. 
For Meta-World reach and push,
$N=1000$.
For Slider and Adroit $N=100$. 
All domains were evaluated on $50$ in-distribution out-of-sample points and $50$ \oos{} points.
For each domain, all methods were trained on the same expert dataset and evaluated on the same in-distribution out-of-sample and \oos{} sets.

We collect expert demonstrations as follows.
For grasp point prediction, we use the Neural Descriptor Fields \citep{simeonov2022neural} simulator to generate a feasible grasping point by a Franka Panda arm for three ShapeNet objects (bottle 2d1aa4e124c0dd5bf937e5c6aa3f9597, mug 61c10dccfa8e508e2d66cbf6a91063 and Utah teapot wss.ce04f39420c7c3e82fb82d326efadfe3).
In Meta-World, we use the expert policies provided in the environment.
For Adroit and Slider we train an expert policy using standard Reinforcement Learning (RL) methods - TRPO~\citep{schulman15trpo} for Adroit and SAC~\citep{haarnoja18sac}

The weighting grasp point prediction network was provided with with $300$ labels of positive pairs to be transduced and $300$ negative pairs labeled with binary labels. This is a slightly more relaxed condition than labeling $1000$ objects as in the standard version.

\subsubsection{Details of Model Architectures}
\label{app:model_arch_details}

For analytic domains, we use MLPs (both for NN and bilinear embeddings) with 3 layers of $1000$ hidden units each with ReLu activations. We train all models with periodic activations of fourier features as described in ~\cite{tancik20fourier}. 

\textbf{Linear model}: data points $x$ are processed through a fully connected layer to output predictions $y$.
\textbf{Neural Network}: data points $x$ are processed through an mlp with $k$ layers, $n$ units each and ReLU activations for hidden layers to output predictions $y$.
\textbf{DeepSets}: we refer to the elements of the state space excluding target position (Meta-World, Adroit) or mass (Slider) as observations. Observations and target position or mass are processed separately by two mlps with $k$ layers, $n$ units each and ReLU activations for hidden layers. Their embeddings are summed and processed by an mlp with one hidden layer with $n$ units and ReLU activations to produce predictions $y$.
\textbf{Transduction}: training point $x_j$ and $x_i-x_j$ (for training or test point $x_i$) are concatenated and processed through an mlp with $k$ layers, $n$ units each and ReLU activations for hidden layers, to predict $y_i$.
\textbf{Our architecture Bilinear Transduction}: training point $x_j$ and $\Delta x_{ij}=x_i-x_j$ (for training or test point $x_i$) are embedded separately to a feature space in $\mathbb{R}^{d\cdot m}$, where $d$ is the dimension of $y_i$, by two mlps with $k$ layers, $n$ units each and ReLU activations for hidden layers. Each predicted element of $y_i$ is the dot product of $m$-sized segments of these embeddings as described in \Cref{eq:pibarthet}.

Further, we evaluate for all models with \textbf{Fourier Embedding} as a pre-processing step. We process inputs $x$, target or mass or $\Delta x$ through a linear model which outputs $\times40$ the number of inputs. We then multiply values by $\pi$ and process through the $\sin$ function. This layer is optimized as the first layer of the model.

In Table~\ref{tab:scaled-problems} we search over number of layers $k\in\{2,3\}$, and unit size $n\in\{32,512,1024\}$. The bilinear transduction embedding size is $m=32$.
In Tables~\ref{tab:weighted-mugs} and \ref{tab:3seeds}
we set $k=2$, $n=32$ and $m=32$.

The weighting function used for grasp point prediction is the bilinear transduction architecture with $k=2$, $n=128$ and $m=32$.

\subsubsection{Details of Model Optimization}
We train the analytic functions for $500$ epochs, batch size $32$, and Adam optimizer with learning rate $1e-4$. We optimize the mean squared error (MSE) loss for regression.

We trained all more complex models for $5$k epochs, batch size $32$,
with Adam \citep{kingma2014adam}  optimizer and learning rate 1e-4.
We optimize the L2-norm loss function comparing between ground truth and predicted grasping points or actions for the sequential decision making domains.
Each configuration of hyperparameters was ran and tested on one seed. We demonstrate the stability of our method across three seeds for a fixed set of hyperparameters in Table~\ref{tab:3seeds}.

We train the weighted grasp point prediction for $5$k epochs, batch size $16$, and Adam optimizer with learning rate $1e-4$. We optimize the MSE loss between the output and ground truth binary label indicating if a training point should be transduced to another training point.
The weighting function did not require further finetuning jointly with the bilinear predictor. 
